\documentclass[11pt,letterpaper]{article}
\usepackage[margin=1in]{geometry}
\usepackage[T1]{fontenc}
\usepackage{mathptmx}

\usepackage{amsmath,amssymb,amsthm,mathtools,bm,mathrsfs}
\usepackage{graphicx,booktabs,array,tabularx,longtable}
\usepackage{tikz}
\usepackage{microtype,natbib,enumitem,fancyhdr}
\usepackage[colorlinks=true,linkcolor=blue!45!black,citecolor=blue!45!black,urlcolor=blue!45!black]{hyperref}
\usepackage{xcolor}
\hypersetup{pdftitle={Human-Centered Learning Mechanics: Observable Entropy Control under Protected-Risk Constraints},pdfauthor={Kim Phuc Tran}}
\newtheorem{theorem}{Theorem}[section]
\newtheorem{proposition}[theorem]{Proposition}
\newtheorem{corollary}[theorem]{Corollary}

\theoremstyle{definition}\newtheorem{definition}[theorem]{Definition}
\theoremstyle{remark}
\newcommand{\R}{\mathbb{R}}\newcommand{\E}{\mathbb{E}}\newcommand{\Prob}{\mathbb{P}}
\newcommand{\tr}{\operatorname{tr}}\newcommand{\Cov}{\operatorname{Cov}}
\newcommand{\KL}{\operatorname{KL}}
\newcommand{\norm}[1]{\left\lVert#1\right\rVert}\newcommand{\ip}[2]{\left\langle#1,#2\right\rangle}

\title{\vspace{-1.3em}Human-Centered Learning Mechanics:\\Observable Entropy Control under Protected-Risk Constraints}
\author{Kim Phuc Tran\\[3pt]\small Univ. Lille, ENSAIT, ULR 2461--GEMTEX\\\small Laboratoire de G\'enie et Mat\'eriaux Textiles, F-59000 Lille, France\\\small International Chair in DS \& XAI, International Research Institute\\\small for Artificial Intelligence and Digital Transformation, Dong A University, Da Nang, Vietnam\\\small\texttt{kim-phuc.tran@ensait.fr}}
\date{}

\begin{document}\maketitle
\begin{abstract}
Human-Centered Learning Mechanics (HCLM) treats learning as an open controlled process in which task progress, representation information, structural restrictions and externally specified risk limits jointly govern parameter updates. We characterize this mechanism through a normalized covariance-information force, metric observability and exact feasible rate constraints. The projected task rate depends on an observable crossproduct that differs from the ordinary task--information product when a probe set does not contain the training task. Error, finite-step and delay margins specify when admissible local requests can be implemented. A primary 288-run neural study separates information control from empirical risk protection and retains cases where extra compression worsens held-out loss. Follow-up experiments train the observable restriction in a declared 13-dimensional subspace, implement a frozen-public-encoder readout through an actual Gaussian sample channel, evaluate a fixed-coefficient sensitivity grid, and test targeted nuisance control with two-hidden-layer minibatch learning on Fashion-MNIST. The nuisance study improves shifted accuracy over ERM, while generic protected balancing accounts for most of the gain. Together the theory and experiments identify attainable information actions, predictive costs and the separate assumptions required for population-risk certification.
\end{abstract}

\section{Introduction}
Learning systems are commonly trained to reduce an empirical objective, with regularization used to influence the representations or parameters that achieve that reduction. The objective provides a tractable optimization target, but its improvement does not by itself determine whether adaptation remains acceptable across operating conditions or affected groups. An industrial fault detector, for example, may improve average classification loss while losing sensitivity to a rare but consequential fault. A representation penalty can change this trade-off without identifying which information is removed or which outcome should be protected. Reliable adaptation therefore requires an account of the permitted learning changes as well as the quantity being minimized.

Several limitations arise when these decisions are encoded only through a scalar objective. A fixed penalty coefficient expresses a compromise, but it does not specify whether a required task rate and a requested representation change are simultaneously attainable. Adaptive gradient balancing responds to relative gradient magnitudes, yet magnitude alone does not identify an intervention's effect on predictions. Parameter units and prediction-preserving transformations can alter that magnitude or the representation covariance while leaving the predictor unchanged \citep{amari1998,dinh2017}. Even an intervention that lowers both task loss and a representation functional can increase a separately specified group risk. These are different failures: an unobservable change, an incompatible optimization request, and an unprotected consequence. Each requires its own diagnostic or constraint.

Information-based representation learning makes these distinctions particularly relevant. The information bottleneck and covariance-based objectives provide ways to regulate representation content \citep{tishby1999,alemi2017,yu2020coding}, while the interpretation of compression depends on the encoder, the noise model and the learning procedure \citep{saxe2018}. A small covariance-volume proxy need not imply a useful compressed representation, and information about a fresh input is different from the information a trained predictor retains about its training sample. Consequently, neither the value of an entropy proxy nor the norm of its gradient is sufficient to establish improved generalization or acceptable risk.

Human-Centered Learning Mechanics (HCLM) addresses these interactions by treating learning as an open controlled process. A task objective determines a reference parameter velocity. A differentiable representation functional determines an information request. Structural restrictions specify available interventions, and externally authorized risk limits determine which consequences an update must respect. The controller assesses their compatibility before releasing a candidate update, then checks the finite change. A person or institution with the relevant authority specifies the protected outcomes and the permitted response to an incompatible request. This human role enters through the learning specification and its revision; the numerical controller evaluates that specification rather than deciding what people ought to value.

The contribution developed here concerns the admissible information actions of this mechanism. We first relate the force induced by a normalized log-determinant functional to the representation spectrum and the encoder Jacobian. The resulting geometry distinguishes small forces caused by collapse, forces attenuated by parameterization, and entropy changes that are invisible to a declared prediction map. A metric projection identifies the observable component. When the probe map does not contain the training task, we derive the distinct task--information crossproduct required to evaluate the projected update. Fresh-probe results state when a direction selected from current observations can also be assessed on a specified population.

We then characterize how observable action interacts with requested task progress and protected risks. Exact scalar intervals describe when a nonnegative entropy coefficient can satisfy both rate requirements. A multi-risk feasibility formulation extends this question to a declared intervention space and budget, with a dual certificate identifying incompatible combinations of requirements. Error, curvature and delay bounds connect the local request to implementable finite updates and quantify the conservatism introduced by incomplete information. These results use established tools from matrix calculus, projection, convex analysis and concentration. Their contribution to HCLM is the explicit relation among information action, attainable rates, protected consequences and the evidence needed to authorize an update.

The experiments examine these relations separately. A paired neural benchmark distinguishes representation control from empirical risk filtering and retains cases in which additional compression worsens held-out loss. Follow-up studies train an observable restriction, implement a frozen-public-encoder readout through a genuine noisy sample channel, evaluate sensitivity to fixed penalty coefficients, and test targeted nuisance control with two-hidden-layer minibatch learning on Fashion-MNIST. The last study improves shifted accuracy over empirical risk minimization, while generic protected balancing explains most of the gain. These comparisons identify the mechanism responsible for each observed change rather than treating an entropy decrease as a common explanation for all outcomes.

HCLM is a learning mechanism that can be instantiated on a specified parameterized predictor; its architecture remains a separate choice. The guarantees concern declared losses, metrics, interventions, operating laws and sources of evidence. The experiments use protocol-defined risk limits, so the quality of human specification remains a question for prospective study. This scope preserves the human-centered purpose of HCLM while making its mathematical and empirical claims assessable.

The remainder of the paper is organized as follows. Section~\ref{sec:related} reviews the relevant literature and positions HCLM. Section~\ref{sec:theory} develops the mathematical setting, information geometry and admissible learning dynamics. Section~\ref{sec:experiments} evaluates the resulting mechanism and its assumptions. Section~\ref{sec:discussion} discusses limitations and research directions, and Section~\ref{sec:conclusion} concludes. Complete proofs, supporting statements and experimental protocols are supplied in the appendices.

\section{Related work and positioning of HCLM}\label{sec:related}
HCLM connects representation regularization, constrained learning and human authority over adaptation. These connections determine the scope of its contribution: the object of analysis is the set of information-driven learning actions compatible with a declared task, observation map and protected-risk specification.

\subsection{Information regularization and representation geometry}
The information bottleneck formalizes compression that preserves information relevant to a target \citep{tishby1999}; its variational formulation makes this objective tractable for neural representations \citep{alemi2017}. Interpreting compression during training requires specifying the encoder, the noise model and the information quantity being estimated. Compression patterns alone do not supply a general explanation of neural generalization \citep{saxe2018}. Representation-learning bounds address this issue by accounting for the representation function or its sample-dependent description \citep{kawaguchi2023,sefidgaran2023}. This motivates a distinction between changing a representation statistic, changing the predictor and controlling dependence on the training sample.

Covariance-based objectives already provide important representation-learning mechanisms. Maximal coding rate reduction compares the coding rate of the complete dataset with class-conditional coding rates \citep{yu2020coding}, and its representation geometry admits a detailed mathematical analysis \citep{wang2024}. VICReg combines view invariance, variance preservation and covariance regularization to control collapse and redundancy \citep{bardes2022}. HCLM studies the action induced by a chosen representation functional once task and risk requirements have been specified. Its normalized log-determinant example exposes two ambiguities: a small gradient can arise at different covariance scales, and a large change in covariance volume can occur along a prediction-preserving rescaling. The observable decomposition then identifies the component of the information direction that changes the declared predictions to first order. These questions concern the interpretation and admissibility of representation change; the choice of a useful compression functional remains part of the learning specification.

\subsection{Gradient control and protected learning updates}
The meaning of a gradient depends on the parameter geometry. Natural-gradient methods make this geometry explicit through a metric \citep{amari1998}, while prediction-preserving reparameterizations show why parameter-space measurements require care \citep{dinh2017}. HCLM uses a declared metric to measure interventions and construct the observable projection. It also specifies the representation noise scale, since changing latent units while holding that scale fixed changes the information channel being measured.

Adaptive weighting and gradient manipulation provide established ways to reconcile several objectives. GradNorm adjusts their relative gradient magnitudes \citep{chen2018}; PCGrad modifies conflicting gradient components \citep{yu2020pcgrad}; ConFIG constructs conflict-free updates for composite losses \citep{liu2025}. HCLM asks which information-decrease rates remain attainable after requiring a specified amount of task progress, restricting the intervention directions and imposing a movement budget. Its scalar interval and multi-risk feasibility characterizations give explicit answers, including cases in which the complete request is impossible. Standard projection and convex optimization supply the computational tools; the contribution lies in the information-learning specification, its observable restrictions and the resulting obstructions.

Protected updates also have a substantial control literature. Control-barrier methods enforce derivative conditions for state constraints \citep{ames2017}, predictive safety certification filters learning-based physical control actions \citep{wabersich2021}, and constrained policy optimization treats reward and expected-cost requirements jointly \citep{achiam2017}. HCLM applies derivative constraints to parameter adaptation while tracking representation information and separately specified risks. Those risks can concern class-defined groups or operating conditions. Their substantive meaning comes from the specification; enforcing a selected group-loss threshold establishes the corresponding stated requirement. Population calibration is a separate statistical problem, as illustrated by conformal risk control \citep{angelopoulos2024}.

\subsection{Human feedback, statistical guarantees and the HCLM contribution}
Human preferences can supply a learned reward \citep{christiano2017}, and demonstrations and rankings can guide model fine-tuning \citep{ouyang2022}. Learning to defer addresses the allocation of predictions between a model and an expert \citep{mozannar2020}. HCLM assigns an additional role to external authority: an authorized person determines protected limits, intervention permissions and the response to incompatible learning requests. Preference estimates or expert decisions can inform that specification when their provenance is declared. The mathematical controller determines whether an update meets the specified conditions; evaluating the quality of those human choices requires its own study.

Algorithmic information bounds concern dependence between the training sample and the learned predictor \citep{xu2017}. HCLM's sample-code result specializes that established argument to a fully specified noisy learning interface. Activation covariance and sample-code covariance describe different objects, so the theory states the assumptions connecting a released code to the final predictor. Similarly, fresh-probe certificates explicitly charge selection within a fixed intervention subspace, whereas repeatedly checking a reused control set establishes an empirical acceptance condition. Together these distinctions locate HCLM's contribution in a controlled learning mechanism with identifiable information actions, attainable rates and conditions for certifying protected consequences. It operates on a specified differentiable learner; its implementation can therefore be evaluated against both regularization baselines and generic constrained-learning methods under the same declared requirements.

\section{Theory of Human-Centered Learning Mechanics}\label{sec:theory}\label{sec:framework}
\subsection{Learner, data and mathematical conventions}\label{sec:notation}
Let $\Theta\subseteq\R^D$ be an open parameter domain, where $D\ge1$ is the number of parameters. For an input $x$ in a space $\mathcal X$, the predictor is
\[
 f_\theta(x)=r_\theta(e_\theta(x))\in\R^{d_y},\qquad
 e_\theta:\mathcal X\to\R^p,\quad r_\theta:\R^p\to\R^{d_y}.
\]
Here $\theta\in\Theta$, $e_\theta$ is the encoder, $r_\theta$ is the readout, $p\ge1$ is the representation dimension, and $d_y\ge1$ is the number of prediction coordinates. The output units are fixed by the application: the prediction can contain logits, probabilities or regression outputs. Differentiability is required with respect to the parameters wherever a derivative is used.

Write $S=(U_1,\ldots,U_n)$ for a training sample of size $n\ge1$, with $U_i=(X_i,Y_i)$ denoting an input--label pair with $Y_i$ in the label space $\mathcal Y$. For a declared loss $\ell:\R^{d_y}\times\mathcal Y\to\R$, a typical task objective is
\[
 L(\theta)=\frac1n\sum_{i=1}^n\ell(f_\theta(X_i),Y_i).
\]
Throughout, $L:\Theta\to\R$ denotes the differentiable objective whose progress is to be protected. Declared weights or structural penalties may be included in $L$; changing those terms changes the objective covered by its descent certificate. The geometric statements hold for a fixed sample. Statistical results separately specify independent sampling, conditioning and the operating law.

Let $H:\Theta\to\R$ be a differentiable representation functional and define the column gradients $g=\nabla_\theta L\in\R^D$ and $h=\nabla_\theta H\in\R^D$. The encoder $e_\theta$ and the information gradient $h$ are distinct objects. A symmetric positive definite matrix $M\in\R^{D\times D}$ specifies the parameter metric. For vectors $u,v\in\R^D$ and a gradient covector $q\in\R^D$, define
\[
 \langle u,v\rangle_M=u^\top Mv,\qquad
 \|u\|_M=(u^\top Mu)^{1/2},\qquad
 \|q\|_{M^{-1}}=(q^\top M^{-1}q)^{1/2}.
\]
The last norm measures a gradient covector in the dual metric. A state-dependent metric is evaluated at the current state; these local rates do not differentiate $M$ in time. Define
\begin{equation}\label{eq:abc}
 a=g^\top M^{-1}g,\qquad b=g^\top M^{-1}h,\qquad c=h^\top M^{-1}h.
\end{equation}
Thus $a,c\ge0$, whereas $b$ measures signed task--information interaction. The reference velocity is $v_0=-M^{-1}g$. An additional velocity $u$ gives $w=v_0+u$, with a finite update $\theta^+=\theta+\eta w$ for step length $\eta\ge0$. A dot denotes differentiation along physical learning time. An intervention budget $\|u\|_M\le B_{\rm int}$ and a total-velocity budget $\|w\|_M\le B_{\rm tot}$ define different admissible sets; both budgets are nonnegative. A declared linear subspace $\mathcal U\subseteq\R^D$ can further restrict $u$. The implementation stores the equivalent descent vector $d_{\rm code}=-w$.

Unadorned vector norms are Euclidean; $\|A\|_F$ and $\|A\|_{\rm op}$ denote the Frobenius norm and the induced Euclidean operator norm of a matrix $A$. Superscripts $\top$ and $\dagger$ denote transpose and the Moore--Penrose pseudoinverse. The identity matrix is $I_k$ in dimension $k$, and $\mathbf1_m$ is the length-$m$ vector of ones. The operators $\operatorname{tr}$, $\det$, $\operatorname{Cov}$ and $\operatorname{vec}$ denote trace, determinant, covariance and column-stacking vectorization. Logarithms are natural, and information is measured in nats. Expectations and probabilities use the law stated in the corresponding result. Appendix~\ref{app:notation} collects the specialized symbols introduced later.

\subsection{Controlled learning interface}\label{sec:mechanism}
HCLM treats learning as an open process whose admissible changes are specified before an update is released. The task loss supplies a reference direction; a representation functional supplies an information request; structural and outcome constraints restrict the permitted changes; and an authorized person determines the specification and the response to an incompatible request. These roles are retained throughout the analysis. An entropy value alone does not determine a permissible update.

At the current state, the controller seeks an intervention in the declared subspace and budget that meets the task, information and protected-risk requirements. The resulting velocity $w=v_0+u$ is passed to the finite-update check. The intervention budget measures parameter movement; protected risks have their own outcome scales and thresholds.

The specification contains a required fraction of task progress, an entropy-rate request, protected-risk thresholds, an intervention budget, and update-release authority. A structural prior can enter through the predictor class, the intervention subspace, or an explicitly declared term in $L$. These choices have different effects and are recorded separately. Human input changes the specification rather than being represented by an unspecified force that disappears from the proofs.

\begin{figure}[t]
\centering
\begin{tikzpicture}[
  >=stealth,
  every node/.style={font=\footnotesize},
  box/.style={draw,rounded corners=2pt,align=center,text width=3.5cm,
              minimum height=1.25cm,inner sep=4pt},
  flow/.style={->,thick},
  lab/.style={font=\scriptsize,fill=white,inner sep=1.5pt}
]
\node[box] (learner) at (-4.4,1.1)
  {\textbf{Current learner}\\Parameters and data roles\\Gradients and probes};
\node[box] (controller) at (0,1.1)
  {\textbf{Learning controller}\\Task and information rates\\Risk and error margins};
\node[box] (check) at (4.4,1.1)
  {\textbf{Candidate check}\\Finite-step loss and risk\\Release conditions};

\node[box] (authority) at (-4.4,-1.2)
  {\textbf{Authorized specification}\\Task, risk and information\\Budgets and authority};
\node[box] (retain) at (0,-1.2)
  {\textbf{Retain approved state}\\Log the decision status\\Await evidence or revision};
\node[box] (release) at (4.4,-1.2)
  {\textbf{Release learning update}\\Record spec and step\\Observe consequences};

\draw[flow] (learner.east) -- (controller.west);
\draw[flow] (controller.east) -- (check.west);
\draw[flow] (controller.south) --
  node[lab,left]{not admitted} (retain.north);
\draw[flow] (check.south) --
  node[lab,right]{passes} (release.north);
\draw[flow] (check.south west) -- ++(-0.28,0) |-
  (retain.east);
\node[lab] at (2.65,0.02) {fails};
\draw[flow] (retain.west) -- (authority.east);
\draw[flow] (authority.north east) -- ++(0.28,0) |- (controller.south west);
\draw[flow] (release.south) -- (4.4,-2.5) --
  node[lab,above]{next learning state} (-6.8,-2.5) |- (learner.west);
\end{tikzpicture}
\caption{The HCLM learning loop. A specification determines admissible learning updates. A failed check retains the approved state; certified infeasibility and numerical solver failure are logged separately. Population-risk claims additionally require the sampling conditions of their statistical certificates.}
\label{fig:hclmloop}
\end{figure}

Figure~\ref{fig:hclmloop} gives the order of operations. The controller first determines whether the requested rates fit the allowed directions and budget. It then chooses an admissible candidate and checks the actual finite update. When the complete request cannot be met, a declared priority policy can relax the information request while retaining the protected conditions, or retain the approved parameter state. A numerical solver's failure is reported separately from a mathematical infeasibility certificate. Changing a protected threshold requires authorization; a positive feasibility residual does not supply that authorization.

This mechanism operates on a specified parameterized learner. An encoder--readout network is the running example; its architecture remains a separate design choice. The results below concern metric gradient velocities and explicitly checked finite updates. Extensions to momentum or adaptive optimizers require accounting for their additional state. The primary experiments use a common full-batch learner to compare intervention and protection effects; the nuisance study separately evaluates minibatch learning.

\subsection{Representation geometry and information force}\label{sec:geometry}
For probe inputs $x_1^{\rm pr},\ldots,x_m^{\rm pr}$, with $m\ge2$, let $Z=Z_\theta\in\R^{m\times p}$ have row $i$ equal to $e_\theta(x_i^{\rm pr})^\top$. The superscript $\mathrm{pr}$ distinguishes these inputs from the training inputs. The centering matrix, centered representations and empirical covariance are
\[
 P_m=I_m-\frac{\mathbf1_m\mathbf1_m^\top}{m},\qquad
 Z_c=P_mZ,\qquad C_Z=\frac{Z_c^\top Z_c}{m-1}\in\R^{p\times p}.
\]
For any positive semidefinite $p\times p$ covariance matrix $C$ and fixed noise variance $\nu>0$, define
\begin{equation}\label{eq:H}
 \begin{aligned}
 \mathcal H_\nu(C)&=\frac1{2p}\log\det(I_p+C/\nu),\\
 H_\nu(Z)&=\mathcal H_\nu(C_Z),\qquad H(\theta)=H_\nu(Z_\theta).
 \end{aligned}
\end{equation}
The covariance functional $\mathcal H_\nu$ and the empirical batch functional $H_\nu$ have different arguments. Both are nonnegative and measured per representation coordinate. The variance $\nu$ fixes the reference latent noise scale and has the same units as a covariance eigenvalue. Since $C_Z+\nu I_p$ is positive definite, the determinant and inverse used below exist. Equivalently, $H_\nu(Z)=(2p)^{-1}\log\det(C_Z+\nu I_p)-\tfrac12\log\nu$.

The total proxy is $H_{\rm tot}=pH$. Its gradient and interaction quantities satisfy $h_{\rm tot}=ph$, $b_{\rm tot}=pb$ and $c_{\rm tot}=p^2c$. To preserve the same objective $L+\beta_{\rm tot}H_{\rm tot}$ in normalized units, one must set $\beta_{\rm norm}=p\beta_{\rm tot}$. A total decrease request $r_{\rm tot}$ corresponds to $r_H=r_{\rm tot}/p$. The theory and experiments use the normalized convention throughout.

For a fixed encoder, an independent input $X$, finite population covariance $C=\Cov(e_\theta(X))$, and independent $\xi\sim N(0,\nu I_p)$,
\begin{equation}\label{eq:capacity}
 I(X;e_\theta(X)+\xi)\le p\mathcal H_\nu(C)=\frac12\log\det(I_p+C/\nu).
\end{equation}
Here $I(A;B)$ denotes mutual information, and $N(0,\nu I_p)$ is the centered Gaussian law with covariance $\nu I_p$. To derive the bound, fix the encoder, put $\widetilde Z=e_\theta(X)+\xi$, and write $\mathrm h_{\rm diff}$ for differential entropy. Independence gives $\Cov\widetilde Z=C+\nu I_p$ and $\mathrm h_{\rm diff}(\widetilde Z\mid X)=\mathrm h_{\rm diff}(\xi)=\tfrac12\log\det(2\pi e\nu I_p)$. Gaussian maximum entropy bounds $\mathrm h_{\rm diff}(\widetilde Z)$ by $\tfrac12\log\det(2\pi e[C+\nu I_p])$. Subtracting proves \eqref{eq:capacity}. The empirical quantity in \eqref{eq:H} is a plug-in functional, not a certified population upper bound. If the encoder is fitted on the same sample, population claims require additional uniform control or an independent probe sample. Section~\ref{sec:generalization} explains why \eqref{eq:capacity} is not an algorithmic generalization bound.

The gradient of this functional determines which covariance scales contribute to a representation update. The next result expresses that contribution through the singular values of the centered batch.
\begin{theorem}[Gradient and covariance-spectrum dependence]\label{thm:spectral}
Let $s_1,\ldots,s_{r_Z}>0$ be the nonzero singular values of $Z_c$, where $r_Z=\operatorname{rank}(Z_c)\le\min\{m-1,p\}$. Let $\lambda_j=s_j^2/(m-1)$. Then
\begin{align}
 \nabla_Z H_\nu&=\frac1{p(m-1)}Z_c(C_Z+\nu I_p)^{-1},\label{eq:zgradient}\\
 \norm{\nabla_Z H_\nu}_F^2
 &=\frac1{p^2}\sum_{j=1}^{r_Z}\frac{s_j^2}{[s_j^2+(m-1)\nu]^2}
 =\frac1{p^2(m-1)}\sum_{j=1}^{r_Z}\frac{\lambda_j}{(\lambda_j+\nu)^2}
 \le\frac{r_Z}{4p^2(m-1)\nu}.\label{eq:spectral}
\end{align}
Empty sums are zero when $r_Z=0$. Each spectral term is maximized at $\lambda_j=\nu$ and tends to zero as $\lambda_j\downarrow0$ or $\lambda_j\to\infty$.
\end{theorem}
\noindent\emph{Proof.} See Appendix~\ref{proof:thm:spectral}.

Consequently, a small representation-space force is ambiguous. It can indicate a collapsed spectrum or a spectrum far above the reference noise scale. Parameter gradients add another possible source of attenuation:
\begin{equation}\label{eq:pullback}
 \nabla_\theta H_\nu=J_Z^\top\operatorname{vec}(\nabla_Z H_\nu),
 \qquad J_Z=D_\theta\operatorname{vec}(Z).
\end{equation}
In \eqref{eq:pullback}, $\nabla_\theta H_\nu$ abbreviates the gradient of the composition $H_\nu(Z_\theta)$. The Jacobian $J_Z\in\R^{mp\times D}$ maps a parameter velocity to a vectorized representation change. A nonzero representation force may lie in the kernel of $J_Z^\top$. Conversely, parameter rescaling can inflate its Euclidean norm. Neither force norm is a sufficient statistic for useful compression.

The distinction between representation change and predictive change also appears in an exact rescaling symmetry of an encoder with a linear readout.
\begin{proposition}[Covariance compression without prediction change]\label{prop:gauge}
Let $e(x)\in\R^p$ have finite covariance $C$ with eigenvalues $\lambda_1,\ldots,\lambda_p\ge0$, and let $W\in\R^{d_y\times p}$ be a linear readout matrix. For the deterministic predictor $f(x)=We(x)$ and a scale $\gamma>0$, set $e_\gamma=\gamma e$ and $W_\gamma=W/\gamma$. Then $f_\gamma=f$, while
\begin{equation}\label{eq:gauge}
 \mathcal H_\nu(\operatorname{Cov}(e_\gamma(X)))
 =\frac1{2p}\sum_{j=1}^p\log(1+\gamma^2\lambda_j/\nu)
 \longrightarrow0\quad\text{as }\gamma\downarrow0.
\end{equation}
If $C\ne0$, the derivative with respect to $\log\gamma$ is strictly positive for every finite $\gamma>0$. Thus the covariance functional changes along a direction with zero predictive derivative.
\end{proposition}
\noindent\emph{Proof.} See Appendix~\ref{proof:prop:gauge}.

The result applies directly to homogeneous encoders that admit this rescaling, and to any architecture with a freely scalable representation interface. It concerns deterministic predictions. If fixed latent noise is physically injected before the readout at inference, the transformation changes the output-noise distribution and is no longer a symmetry of that stochastic predictor.

A change of latent units should transform the noise as well. For a positive definite noise covariance $N$, define
\begin{equation}
 \mathcal H(C,N)=\frac1{2p}\log\frac{\det(C+N)}{\det N}.
\end{equation}
Here $N\in\R^{p\times p}$ is positive definite and represents latent noise covariance. For an invertible coordinate matrix $A\in\R^{p\times p}$, $\mathcal H(ACA^\top,ANA^\top)=\mathcal H(C,N)$, because the determinant factors cancel. Holding isotropic $N$ fixed while changing latent coordinates describes a different channel.

Covariance-error sensitivity and a bounded-readout volume floor are stated in Appendix~\ref{app:geometrysupport}; they specify when these geometric quantities remain interpretable.
\subsection{Observable action and statistical scope}\label{sec:observable}
Using the gradients and metric products in \eqref{eq:abc}, we now distinguish information activity from its effect on a declared prediction map.
The metric lift of the information gradient is $u_H=M^{-1}h$, with metric norm $\sqrt c$. A compressive contribution to the parameter velocity is $-\beta u_H$ for $\beta\ge0$. Let $F(\theta)\in\R^{m_F}$ be a specified vector of predictions or decisions on a probe set, with Jacobian $J_F=D_\theta F\in\R^{m_F\times D}$. The integer $m_F$ counts observed output coordinates and need not equal the number $m$ of probe inputs. The local predictive action is $J_Fu_H$. Its interpretation depends on the probe distribution and on whether $F$ contains logits, probabilities, or downstream decisions.

\begin{definition}[Three diagnostics for entropy intervention]
For a coefficient $\beta\ge0$, parameter activity is measured by $\beta\sqrt c$, predictive action by $\beta\norm{J_Fu_H}$, and task interaction by $\beta b$. Where $a>0$, the ratio $\beta\sqrt{c/a}$ measures intervention size relative to task descent. A threshold for any diagnostic must be fixed with its units, metric, and reference scale; existence of an arbitrarily small positive threshold is not evidence of practical effectiveness.
\end{definition}

\begin{theorem}[Observable and prediction-null decomposition]\label{thm:projection}
Define
\begin{equation}\label{eq:projection}
 \Pi_{\rm obs}=M^{-1}J_F^\top(J_FM^{-1}J_F^\top)^\dagger J_F,
 \qquad u_{\rm obs}=\Pi_{\rm obs}u_H,\quad u_{\rm null}=(I_D-\Pi_{\rm obs})u_H.
\end{equation}
Then $\Pi_{\rm obs}$ is the $M$-orthogonal projection onto $\operatorname{range}(M^{-1}J_F^\top)$,
\begin{equation}
 J_Fu_{\rm null}=0,\qquad
 \norm{u_H}_M^2=\norm{u_{\rm obs}}_M^2+\norm{u_{\rm null}}_M^2,
\end{equation}
and $J_Fu_H=0$ if and only if $u_{\rm obs}=0$. Under a smooth invertible reparameterization $\theta=\psi(\phi)$, these quantities are invariant when the metric is transformed as $M_\phi=D\psi^\top M_\theta D\psi$ and the same predictive map is used.
\end{theorem}
\noindent\emph{Proof.} See Appendix~\ref{proof:thm:projection}.

This is an infinitesimal metric statement. A finite probe map diagnoses the predictions it contains. A null direction may remove nuisance variation or exploit a rescaling; its utility is assessed separately.

For a predeclared $k_U$-dimensional intervention subspace, with $1\le k_U\le D$, Proposition~\ref{prop:uniformprobe} gives a population-action certificate from fresh probes. Let $K$ be a population bound on the restricted Jacobian norm, let $m$ now denote the fresh-probe count, and let $\delta\in(0,1)$ be an allocated failure probability. The uniform Gram-error bound is
\[
 \epsilon_m=K^2\min\{1,\sqrt{2[k_U\log9+\log(2/\delta)]/m}\}.
\]
It permits selecting the direction using that fresh batch because the inequality holds for every direction in the fixed subspace. A justified Hessian bound and a bounded Lipschitz harm loss extend it to finite-update risk (Corollary~\ref{cor:proberisk}). Appendix~\ref{sec:uniformprobe} states the complete sampling and derivative assumptions. These certificates are distinct from repeatedly checking one adaptive control set.

\subsection{Task-compatible information requests}\label{sec:control}
With the task, representation functional and observation map specified, we can determine which information requests are compatible with task progress. Consider the controlled flow
\begin{equation}\label{eq:flow}
 \dot\theta=-M^{-1}(g+\beta h),\qquad 0\le\beta\le\overline\beta,\qquad \overline\beta\ge0.
\end{equation}
Here $\beta$ is the nonnegative information coefficient, and $\overline\beta$ is its declared finite cap. At each state, the chain rule gives
\begin{equation}\label{eq:rates}
 \dot L=-a-\beta b,\qquad \dot H=-b-\beta c.
\end{equation}
The quantity $-b$ is the signed contribution of task descent to representation entropy. It need not be positive and is not an information-injection rate in the Shannon sense. The term $\beta c$ is a nonnegative dissipation term for this particular functional.

A task-descent fraction $\kappa\in(0,1]$ requires $\dot L\le-\kappa a$. Let $r_H\in\R$ be a requested entropy-decrease rate, so that $\dot H\le-r_H$. The request can be negative when limited entropy growth is allowed.

\begin{theorem}[Feasibility of a nonnegative entropy coefficient]\label{thm:feasible}
Suppose $c>0$. Define
\begin{equation}\label{eq:interval}
 \beta_{\rm lo}=\max\{0,(r_H-b)/c\},\qquad
 \beta_{\rm hi}=\begin{cases}
 \overline\beta,&b\ge0,\\
 \min\{\overline\beta,(1-\kappa)a/(-b)\},&b<0.
 \end{cases}
\end{equation}
The coefficients satisfying both rate requirements are exactly $[\beta_{\rm lo},\beta_{\rm hi}]$ when $\beta_{\rm lo}\le\beta_{\rm hi}$, and no such coefficient exists otherwise. If $c=0$, then $h=0$, $b=0$, and feasibility is equivalent to $r_H\le0$.
\end{theorem}
\begin{proof}
By the chain rule, $\dot L=g^\top[-M^{-1}(g+\beta h)]=-a-\beta b$ and $\dot H=-b-\beta c$. Therefore
\[
 \dot L\le-\kappa a\ \Longleftrightarrow\ \beta b\ge-(1-\kappa)a,
 \qquad
 \dot H\le-r_H\ \Longleftrightarrow\ \beta c\ge r_H-b.
\]
Since $c>0$, the second inequality and $\beta\ge0$ give the lower endpoint in \eqref{eq:interval}. If $b\ge0$, the first inequality is automatic because its left side is nonnegative and its right side nonpositive. If $b<0$, division reverses the inequality and yields $\beta\le(1-\kappa)a/(-b)$. Intersecting with $\beta\le\overline\beta$ gives the upper endpoint. These implications are equivalences, so every point in the resulting interval works and every point outside it violates a requirement.

If $c=0$, positive definiteness gives $M^{-1/2}h=0$, hence $h=0$ and $b=0$. Entropy has zero instantaneous rate for every coefficient, while task descent is $-a\le-\kappa a$. Thus the entropy request is feasible exactly when $r_H\le0$. Also, if $a=0$ and $c>0$, then $g=0$ and $b=0$: the interval remains valid, but the task certificate says only that its first-order rate is zero.
\end{proof}

\begin{corollary}[Sharp obstruction under gradient conflict]\label{cor:angle}
Suppose $a,c>0$, $b<0$, $r_H=0$, and there is no upper coefficient cap. Entropy nonincrease and preservation of a task-descent fraction $\kappa$ are simultaneously possible if and only if
\begin{equation}\label{eq:angle}
 b^2\le(1-\kappa)ac.
\end{equation}
Equivalently, for $\cos\vartheta=b/\sqrt{ac}<0$, one needs $\cos^2\vartheta\le1-\kappa$.
\end{corollary}
\noindent\emph{Proof.} See Appendix~\ref{proof:cor:angle}.

The lower endpoint is $-b/c$, the critical coefficient for stationary entropy. It is feasible only if it also respects task descent and the actuator cap. Perfectly opposing gradients make simultaneous strict task descent and entropy nonincrease impossible. Increasing the entropy-force magnitude alone does not remove this angular obstruction.

For a differentiable target $H_\star(t)$, let $e=H-H_\star$ and choose
\begin{equation}\label{eq:tracking}
 r_H=k e-\dot H_\star,\qquad k>0.
\end{equation}
On any interval where the controller satisfies \eqref{eq:interval}, $\dot e\le-ke$, hence $e(t)\le e(t_0)e^{-k(t-t_0)}$. This is a one-sided upper-budget guarantee. It neither prevents excessive compression nor proves convergence of $|e|$ to zero. A lower entropy bound would require an additional derivative constraint, which could make the feasible interval empty.

For any proposed coefficient $\beta^{\rm prop}$, the nearest admissible scalar is its projection onto the nonempty interval in \eqref{eq:interval}. If the interval is empty, an explicit priority policy is necessary. In the numerical experiment we preserve task descent and use $\beta_{\rm hi}$, which gives the largest available entropy-decrease rate under that priority. We record the unmet entropy request. This fallback is not described as successful entropy tracking.

The observable restriction changes the feasibility calculation when the probe map does not contain the task outputs.
For the projected entropy velocity $w=-M^{-1}g-\beta\Pi_{\rm obs}M^{-1}h$, set $b_{\rm obs}=g^\top\Pi_{\rm obs}M^{-1}h$ and $c_{\rm obs}=\|\Pi_{\rm obs}M^{-1}h\|_M^2$. Without assuming $L=\ell(F)$,
\[
 \dot L=-a-\beta b_{\rm obs},\qquad \dot H=-b-\beta c_{\rm obs}.
\]
The entropy lower endpoint uses $(r_H-b)/c_{\rm obs}$, whereas the task upper endpoint uses $b_{\rm obs}$. Proposition~\ref{prop:generalobsrates} gives the exact interval, including $c_{\rm obs}=0$, and quantifies $|b-b_{\rm obs}|$. The sufficient condition $L=\ell(F)$ puts $M^{-1}g$ in the observable range and gives $b_{\rm obs}=b$, recovering Theorem~\ref{thm:obscontrol}. A small separate probe batch does not inherit that identity. Section~\ref{sec:experiments} tests both the factorized controller and this diagnostic failure.

\subsection{Finite updates and certified gradient error}\label{sec:finitesteps}
We use one sign convention throughout: $w$ is a parameter velocity,
$\dot\theta=w$, and a finite update is $\theta^+=\theta+\eta w$ with $\eta>0$.
The task-only velocity is $v_0=-M^{-1}g$.
A bound $\|w-v_0\|_M\le B_{\rm int}$ limits the additional intervention;
a bound $\|w\|_M\le B_{\rm tot}$ limits the total velocity. These are different
admissible sets. The statements below specify the norm of the total velocity
needed for finite-step descent, irrespective of which budget produces it.

\begin{theorem}[Finite-step task descent]\label{thm:discrete}
Let $L$ be lower bounded by $L_{\inf}$, with $\ell_L>0$ and an $\ell_L$-Lipschitz gradient along
each accepted update segment. At iteration $t$, put
$g_t=\nabla L(\theta_t)$ and $a_t=g_t^\top M_t^{-1}g_t$, with $M_t\succ0$.
Suppose the chosen velocity satisfies
\begin{equation}\label{eq:safeguards}
 g_t^\top w_t\le-\kappa a_t,\qquad
 \|w_t\|^2\le C_w a_t,
\end{equation}
for fixed $\kappa,C_w>0$. For a fixed step
$0<\eta\le\kappa/(\ell_L C_w)$, the update
$\theta_{t+1}=\theta_t+\eta w_t$ obeys
\begin{equation}\label{eq:descent}
 L(\theta_{t+1})\le L(\theta_t)-\frac{\eta\kappa}{2}a_t,
 \qquad
 \min_{0\le t<T}a_t\le
 \frac{2[L(\theta_0)-L_{\inf}]}{\eta\kappa T}.
\end{equation}
If $M_t\preceq m_1I$, multiplying the last bound by $m_1$ bounds
$\min_{t<T}\|g_t\|^2$. For an infinite run under the same assumptions,
$a_t\to0$.
\end{theorem}
\noindent\emph{Proof.} See Appendix~\ref{proof:thm:discrete}.

For the scalar velocity $w_t=-M_t^{-1}(g_t+\beta_t h_t)$, assume
$M_t\succeq m_0I$ and $\beta_t\sqrt{c_t}\le\chi_v\sqrt{a_t}$, where
$c_t=h_t^\top M_t^{-1}h_t$ and $\chi_v\ge0$ is a relative intervention cap.
The triangle inequality gives
$\|w_t\|_{M_t}\le(1+\chi_v)\sqrt{a_t}$, so one can take
$C_w=(1+\chi_v)^2/m_0$. A total-velocity budget
$\|w_t\|_{M_t}\le\gamma_w\sqrt{a_t}$ instead gives $C_w=\gamma_w^2/m_0$.
For example, the implemented Euclidean bound $B_{\rm tot}=2\|g_t\|$
gives $C_w=4$. Armijo backtracking checks realized task descent without knowing
$\ell_L$; the displayed fixed-iteration rate additionally needs the positive
step restriction stated in the theorem.

Suppose estimated gradients obey
$\|g-\widehat g\|_{M^{-1}}\le\delta_g$ and
$\|h-\widehat h\|_{M^{-1}}\le\delta_h$.
Sufficient error-aware rate checks for a proposed velocity $w$ are
\begin{align}\label{eq:robust}
 \widehat g^\top w+\delta_g\|w\|_M
 &\le-\kappa(\|\widehat g\|_{M^{-1}}+\delta_g)^2,\\
 \widehat h^\top w+\delta_h\|w\|_M&\le-r_H.
\end{align}
Metric Cauchy--Schwarz gives
$g^\top w\le\widehat g^\top w+\delta_g\|w\|_M$, and the triangle inequality
gives $a\le(\|\widehat g\|_{M^{-1}}+\delta_g)^2$.
The first check therefore implies $g^\top w\le-\kappa a$; the same
argument proves $h^\top w\le-r_H$ from the second check.
The error radii must be justified independently of these inequalities.

For a stochastic update $\theta_{t+1}=\theta_t+\eta(w_t+\zeta_t)$, let $\mathcal F_t$ be the sigma-field containing the information available before the innovation $\zeta_t$. Require $\theta_t$ and $w_t$ to be $\mathcal F_t$-measurable and assume
$\mathbb E[\zeta_t\mid\mathcal F_t]=0$ and
$\mathbb E[\|\zeta_t\|^2\mid\mathcal F_t]\le s^2$, conditional on the
information preceding the innovation, where $s^2\ge0$ is a uniform conditional second-moment bound. Assume the descent lemma holds on every realized noisy update segment and the displayed expectations are finite. Under the same step restriction,
\begin{equation}\label{eq:stochastic}
 \frac1T\sum_{t<T}\mathbb E a_t\le
 \frac{2[L(\theta_0)-L_{\inf}]}{\eta\kappa T}
 +\frac{\ell_L\eta s^2}{\kappa}.
\end{equation}
Indeed, conditional expectation in the descent lemma removes the linear noise
term and gives
$\mathbb E[L_{t+1}\mid\mathcal F_t]
\le L_t-\eta\kappa a_t/2+\ell_L\eta^2s^2/2$.
Summing and using the lower loss bound proves the result.
If the controller selects its velocity using the same noisy gradients, its
resulting error need not have conditional mean zero; the assumption must then
be replaced by a valid bias or robust-error bound.

\subsection{Protected outcomes and joint feasibility}\label{sec:protected}
The purpose of a human-centered learning theory is to make the consequences of adaptation accountable to the people affected by them. An entropy target cannot specify those consequences. This section adds protected risks as separate mathematical objects and derives conditions for accepting or rejecting an update. A guarantee is always relative to the specified risks, their thresholds, the operating distribution, and the accuracy of the quantities used to certify them.

A simple group-conflict construction shows why a protected risk must be specified separately from the task and representation objectives.
\begin{proposition}[Task and entropy descent can increase a protected-group risk]\label{prop:humanrisk}
Let $X$ take values $-1$ and $1$ with equal probability. Independently, let $A_{\rm grp}$ be a Bernoulli group indicator with $\Prob(A_{\rm grp}=1)=\varpi\in(0,1)$, and set $Y=A_{\rm grp}X$. Consider the scalar representation and predictor $e_\theta(X)=f_\theta(X)=\theta X$. With half-squared loss, define
\begin{align}
 L(\theta)&=\tfrac12\E(f_\theta(X)-Y)^2
 =\tfrac12\{(1-\varpi)\theta^2+\varpi(\theta-1)^2\},\label{eq:groupL}\\
 H(\theta)&=\tfrac12\log(1+\theta^2/\nu),\qquad
 R_1(\theta)=\tfrac12\E[(f_\theta(X)-Y)^2\mid A_{\rm grp}=1]
 =\tfrac12(\theta-1)^2.\label{eq:grouprisk}
\end{align}
For every $\theta\in(\varpi,1)$ and $\beta\ge0$, the flow $\dot\theta=-L'(\theta)-\beta H'(\theta)$ strictly decreases $L$ and $H$ and strictly increases $R_1$. At a state in this interval, no velocity can both preserve a positive fraction of baseline task descent and prevent an increase in $R_1$.
\end{proposition}
\noindent\emph{Proof.} See Appendix~\ref{proof:prop:humanrisk}.

This construction uses an actual population covariance and an actual conditional risk. It therefore does not depend on covariance-estimation error or poor generalization. The issue is that the majority-weighted task and a protected group's outcome are different objectives. An entropy controller must expose this conflict if the group constraint is part of the specification.

We therefore give the protected outcomes their own thresholds and derivative requirements.
Let $R_j(\theta,t)$, $j=1,\ldots,J_R$, be continuously differentiable protected risks, and let $\tau_j(t)$ be their differentiable upper thresholds. Examples include a smooth upper bound on a false-negative risk in a specified group, a predicted constraint-violation cost in a specified operating condition, or a measured burden of unnecessary human interventions. These quantities must be defined and justified for the application; a differentiable surrogate protects a real outcome only if the relation between them is established. Human competence and psychological well-being are not identified with representation entropy.

Write
\begin{equation}\label{eq:riskspec}
 e_j=R_j(\theta,t)-\tau_j(t),\qquad k_j=\nabla_\theta R_j(\theta,t),\qquad
 \mathcal S(t)=\{\theta:e_j\le0\text{ for every }j\}.
\end{equation}
Here $e_j$ is the excess over the permitted risk, $k_j$ is a column gradient, and $\mathcal S(t)$ is the set satisfying the current specification. The rates $\gamma_j>0$ in the condition
\begin{equation}\label{eq:barrier}
 k_j^\top\dot\theta+\partial_tR_j-\dot \tau_j\le-\gamma_j e_j
\end{equation}
determine how quickly a positive excess would be reduced and how much increase is allowed below the threshold. At $e_j=0$, \eqref{eq:barrier} allows no outward motion relative to the moving boundary. It is a standard barrier-type condition \citep{ames2017}, here imposed on learning parameters and specified risk functions.

The responsible decision maker specifies the risks, thresholds, admissible intervention budget, permitted relaxations, and authority to suspend model updates. The controller reports whether these choices can be met. It does not infer a socially correct threshold from the entropy objective, and an infeasibility certificate does not authorize relaxing a protected constraint.

The task, information and risk requirements can now be written as one feasibility problem.
Fix a state and freeze the metric $M\succ0$. Let $v_0=-M^{-1}g$ be the task-only velocity, and write the controlled velocity as $w=v_0+u$. Here $u$ is an additional intervention; a finite update is $\theta^+=\theta+\eta w$. The same velocity convention is used for the scalar and multi-risk controllers.

Choose a linear subspace $\mathcal U\subseteq\R^D$ of admissible interventions, with $M$-orthogonal projector $\Pi_{\mathcal U}$. The unrestricted choice is $\mathcal U=\R^D$. An observable choice is the range of $M^{-1}J_{\mathcal F}^\top$, where $\mathcal F$ contains the predictions and protected quantities that the analysis is intended to observe. The intervention budget is $\norm u_M\le B_{\rm int}$, for a specified $B_{\rm int}\ge0$.

All local requirements can now be written as $d_i^\top u\le t_i$, $i=1,\ldots,N$. The task, entropy, and protected-risk rows, respectively, are
\begin{align}
 (d_i,t_i)&=(g,(1-\kappa)a),\label{eq:rows1}\\
 (d_i,t_i)&=(h,b-r_H),\label{eq:rows2}\\
 (d_i,t_i)&=(k_j,-\gamma_j e_j-\partial_tR_j+\dot \tau_j-k_j^\top v_0).
 \label{eq:rows3}
\end{align}
For example, the task row follows from $g^\top(v_0+u)=-a+g^\top u\le-\kappa a$. The entropy row follows from $h^\top v_0=-b$. A lower entropy budget, if justified, would add a row with gradient $-h$. The vector $d_i$ is a row normal written as a column; it is distinct from the scalar threshold $\tau_j(t)$.

Before comparing violations across requirements, fix positive scales $s_i$ with the units of the corresponding rate. Set $\bar d_i=d_i/s_i$ and $\bar t_i=t_i/s_i$. These scales leave feasibility unchanged but determine how a reported worst violation is measured. They must not be optimized afterward to disguise a particular conflict. Define
\begin{equation}\label{eq:minviolation}
 \varepsilon_\star(\mathcal U,B_{\rm int})=
 \min_{u\in\mathcal U,\,\norm u_M\le B_{\rm int}}
 \max_{1\le i\le N}(\bar d_i^\top u-\bar t_i).
\end{equation}

\begin{theorem}[Observable feasibility and a certificate of unavoidable violation]\label{thm:multirisk}
Assume $N\ge1$, $M\succ0$, and $B_{\rm int}\ge0$. Let
$\Delta_N=\{\lambda\in\R^N:\lambda_i\ge0,\ \sum_i\lambda_i=1\}$ be the probability simplex. Then
\begin{equation}\label{eq:dualcertificate}
 \varepsilon_\star(\mathcal U,B_{\rm int})=
 \max_{\lambda\in\Delta_N}\Bigl\{
 -\sum_i\lambda_i\bar t_i
 -B_{\rm int}\left\lVert \Pi_{\mathcal U}M^{-1}\sum_i\lambda_i\bar d_i\right\rVert_M
 \Bigr\}.
\end{equation}
Both optima are attained. All requirements are simultaneously feasible if and only if $\varepsilon_\star\le0$. Any $\lambda\in\Delta_N$ whose expression in braces is positive certifies infeasibility and lower-bounds the worst normalized violation of every admissible intervention. If $\mathcal U_1\subseteq\mathcal U_2$, then
\begin{equation}\label{eq:monotonesubspace}
 \varepsilon_\star(\mathcal U_1,B_{\rm int})\ge
 \varepsilon_\star(\mathcal U_2,B_{\rm int}).
\end{equation}
For $B_2\ge B_1$, the optimal violation cannot increase when $B_1$ is replaced by $B_2$.
\end{theorem}
\noindent\emph{Proof.} See Appendix~\ref{proof:thm:multirisk}.

The convex minimax and support-function calculation are established mathematics \citep{boyd2004}. The learning-specific content is the joint specification in \eqref{eq:rows1}--\eqref{eq:rows3}, the restriction to a declared observable subspace, and the interpretation of a positive residual as an unattainable learning request. The theorem extends the scalar interval to several risks without assuming that the intersection is nonempty. The multipliers identify a combination of conflicting requirements; they are not moral importance weights or a causal explanation of harm.

There are two distinct effects of expanding the specification. Adding protected-risk rows can only increase $\varepsilon_\star$. Expanding the observable map can enlarge $\mathcal U$ and can only decrease it when the rows are fixed. Doing both has no universal monotonic effect. A direction invisible to task predictions can still matter to a protected outcome; declaring it irrelevant before defining that outcome would remove precisely the intervention needed to protect it.

\subsection{Uncertainty, finite steps and delayed specifications}
The residual can also be characterized when the constraint normals and right-hand sides are uncertain. Work with the normalized rows above. Suppose
\begin{equation}\label{eq:rowuncertainty}
 \norm{\bar d_i-\widehat d_i}_{M^{-1}}\le\delta_i^d,
 \qquad |\bar t_i-\widehat t_i|\le\delta_i^t,
 \qquad\delta_i^d,\delta_i^t\ge0.
\end{equation}
The declared uncertainty set is the Cartesian product of these row balls and intervals, with the metric and intervention subspace fixed. This is an outer uncertainty model when actual estimation errors are correlated. A certificate valid for this product is valid for any smaller joint uncertainty set it contains, but the product can be conservative.

\begin{theorem}[Robust feasible capacity under row uncertainty]\label{thm:robustmultirisk}
The minimum worst normalized violation over the uncertainty set in \eqref{eq:rowuncertainty} equals
\begin{align}
 \varepsilon_{\rm rob}
 &=\min_{u\in\mathcal U,\,\norm u_M\le B_{\rm int}}
 \max_i\left\{\widehat d_i^\top u-\widehat t_i
 +\delta_i^d\norm u_M+\delta_i^t\right\}\label{eq:robustprimal}\\
 &=\max_{\lambda\in\Delta_N}\left\{
 \sum_i\lambda_i(\delta_i^t-\widehat t_i)
 -B_{\rm int}\left[
 \left\lVert \Pi_{\mathcal U}M^{-1}\sum_i\lambda_i\widehat d_i\right\rVert_M
 -\sum_i\lambda_i\delta_i^d\right]_+\right\}.\label{eq:robustdual}
\end{align}
Here $[x]_+=\max\{x,0\}$. A single intervention satisfies every admissible set of rows if and only if $\varepsilon_{\rm rob}\le0$. If, for some $\lambda\in\Delta_N$,
\begin{equation}\label{eq:uncertaintyobstruction}
 \sum_i\lambda_i(\delta_i^t-\widehat t_i)>0,
 \qquad
 \left\lVert \Pi_{\mathcal U}M^{-1}\sum_i\lambda_i\widehat d_i\right\rVert_M
 \le\sum_i\lambda_i\delta_i^d,
\end{equation}
no finite intervention budget can make the robust request feasible.
\end{theorem}
\noindent\emph{Proof.} See Appendix~\ref{proof:thm:robustmultirisk}.

With zero uncertainty this reduces to Theorem~\ref{thm:multirisk}. The expression separates the size of an observable combined gradient from its error radius. Increasing an entropy coefficient or an intervention budget cannot repair a request whose useful direction is smaller than the uncertainty that must be protected against. This is an exact conclusion for the declared uncertainty model; it does not provide the error radii. In particular, $a,b,c$ and risk thresholds appearing inside a right-hand side also have estimation errors, which must be included in $\delta_i^t$.

The same specification yields an entropy-capacity formula, moving-threshold invariance, finite-step curvature margins and fresh-checkpoint audits; their complete statements appear in Appendix~\ref{app:protectionsupport}. Acceptance protects a declared risk under its stated evidence model. The supervisor chooses thresholds and permitted changes, and can retain the last approved parameters when a request is incompatible.

Uncertain or delayed information also changes how much true feasibility margin is needed for an update to be admitted.
For the row-error bounds above, Theorem~\ref{thm:controlleracceptance} proves
\[
 \varepsilon_\star\le\varepsilon_{\rm rob}\le\varepsilon_\star+2\chi,
 \qquad\chi=\max_i\{B_{\rm int}\delta_i^d+\delta_i^t\}.
\]
Thus true feasibility margin $2\chi$ suffices to prevent conservative rejection by an exact robust solve. Logged parameter movement and bounded specification drift provide explicit components of these radii. For a Lipschitz entropy budget $H_\star$ known with delay, Proposition~\ref{prop:delayedentropytracking} supplies a one-sided tracking request and its extra control cost. These results quantify information requirements rather than assuming a generic thermostat converges. Complete conditions and proofs are in Appendix~\ref{sec:delayedcertification}.

\subsection{Information channels and generalization}\label{sec:generalization}
A sample-dependent encoder can memorize training labels while carrying zero information about an independent fresh input (Proposition~\ref{prop:memorization}). A valid alternative constrains the whole training sample channel $S\to(Q,V)\to W$, with $V$ independent of $S$ and $Q=\phi_V(S)+\zeta$ for independent Gaussian noise. The resulting expected-gap bound in Theorem~\ref{thm:gen} uses covariance across possible samples, not activation covariance inside one batch. Appendix~\ref{app:generalization} gives this standard information-bound specialization and its full proof. Section~\ref{sec:experiments} tests a readout that actually obeys the channel.

\section{Experiments}\label{sec:experiments}\label{sec:protected-benchmark}
The empirical studies connect admissible information actions with implemented learning updates. The primary benchmark separates information control from empirical risk protection. A projected-training study then measures observable and prediction-null action, and a frozen-encoder readout implements the sample-code channel. A coefficient grid and a known-nuisance study examine the sensitivity and predictive consequences of the chosen information functional. Each study has its own protocol; all coefficients and unfavorable outcomes are retained.

\subsection{Primary design and protected-learning outcomes}
The primary benchmark contains 288 runs: Digits and WDBC, uniform and 95/5 group-reweighted task objectives, eight paired random splits/initializations, and nine methods. Both datasets use a one-hidden-layer tanh learner, with 16 and eight representation coordinates. After a common 120-update task-only warmstart, methods receive 180 continuation updates. The 60/20/20 train/control/test roles are kept separate. Control data are reused for risk gradients and nonlinear acceptance checks and therefore belong to adaptive fitting. WDBC is a classification benchmark, not a clinical safety evaluation.

Groups are label parity in Digits and malignant/benign labels in WDBC. For group $j\in\{0,1\}$, $R_j$ is its mean cross-entropy (CE) on the fixed control set, and $\theta_0$ is the common warmstart state. The fixed control caps are $\tau_j=1.02R_j(\theta_0)+0.005$ at the common warmstart. The normalized functional in \eqref{eq:H} uses $\nu=0.1$. In the declared stress objective, group~1 receives 95\% of the task weight and group~0 receives 5\%. This imposes a fine-tuning conflict on real labels rather than describing an observed field shift.

In these runs the parameter metric is $M=I_D$, where $D$ is the number of trainable parameters, so all vector norms below are Euclidean. Write $g=\nabla_\theta L$, $h=\nabla_\theta H$ and $k_j=\nabla_\theta R_j$, with $H(\theta)=H_{0.1}(Z(\theta))$. Full HCLM projects the balanced proposal
\[
 w_{\rm prop}=-g-\beta h,\qquad
 \beta=\min\{20,0.5\|g\|/\max(\|h\|,10^{-30})\}
\]
onto the intersection
\begin{align}\label{eq:primaryvelocity}
 \|w\|&\le B_{\rm tot}=2\|g\|,&g^\top w&\le-0.5\|g\|^2,\\
 k_j^\top w&\le-\{R_j-\tau_j+0.05\|k_j\|\|g\|\},&h^\top w&\le-r_H.
\end{align}
Here $\beta$ is the nonnegative information coefficient and $r_H$ is the requested entropy-decrease rate: $r_H=0.25\|g\|\|h\|$ when $H(\theta)>0.8H(\theta_0)$ and zero otherwise. The denominator floor $10^{-30}$ defines the zero-gradient convention; it does not change the stated covariance functional. The inward-rate reserve is fixed from a separate pilot. The total-velocity ball is distinct from the additional-intervention ball used in the local capacity results. Actual finite-step task, information and control-cap checks follow the proposed projection. The implementation stores $d_{\rm code}=-w$ once; Appendix~\ref{app:benchmarkprotocol} gives the exact solver, tolerances and line searches.

The comparators are ERM, fixed and gradient-balanced logdet penalties, task-orthogonal entropy, a two-objective PCGrad adaptation, HCLM without risk rows, a risk-only filter, and protected gradient balancing. The last comparator shares full HCLM's task/risk filter and finite-step checks but omits its entropy-rate requirement. It separates the extra information request from generic constrained learning. The original fixed coefficient was initialized by norm matching; a follow-up grid below evaluates that choice rather than treating it as optimally tuned.

\begin{table}[t]
\centering
\caption{Declared objective-reweighting stress test. Means over eight paired split/initialization runs; $\pm$ gives a 95\% Student-$t$ variability interval half-width. Task CE uses the declared 95/5 objective weights on the held-out set. Group~0 CE is unweighted within the de-emphasized group. Cap violations are the percentage of continuation updates exceeding either fixed control-set cap, not held-out violations. Time excludes the shared warmstart and includes common diagnostics.}
\label{tab:protected-benchmark}
\small\setlength{\tabcolsep}{3pt}
\begin{tabular}{lrrrrr}
\toprule
Method & Test task CE & Test group~0 CE & $H_T/H_0$ & Cap viol. (\%) & Time (s)\\
\midrule
\multicolumn{6}{l}{\textit{Digits}}\\
ERM & $0.128\pm 0.022$ & $0.625\pm 0.077$ & $1.072\pm 0.007$ & 100.0 & 0.53\\
Fixed logdet & $0.252\pm 0.015$ & $2.004\pm 0.088$ & $0.474\pm 0.015$ & 100.0 & 0.53\\
Balanced logdet & $0.136\pm 0.020$ & $0.741\pm 0.082$ & $0.921\pm 0.005$ & 100.0 & 0.55\\
Task-orthogonal & $0.129\pm 0.020$ & $0.691\pm 0.078$ & $0.944\pm 0.007$ & 100.0 & 0.53\\
PCGrad (2 obj.) & $0.140\pm 0.019$ & $0.838\pm 0.086$ & $0.823\pm 0.006$ & 100.0 & 0.54\\
HCLM w/o risk & $0.140\pm 0.020$ & $0.795\pm 0.088$ & $0.880\pm 0.004$ & 100.0 & 0.69\\
Risk-only filter & $0.123\pm 0.024$ & $0.364\pm 0.052$ & $1.103\pm 0.008$ & 0.0 & 0.71\\
Protected balance & $0.123\pm 0.023$ & $0.380\pm 0.049$ & $0.939\pm 0.012$ & 0.0 & 0.80\\
HCLM (full) & $0.121\pm 0.022$ & $0.395\pm 0.049$ & $0.848\pm 0.005$ & 0.0 & 0.88\\
\midrule
\multicolumn{6}{l}{\textit{WDBC}}\\
ERM & $0.021\pm 0.009$ & $0.221\pm 0.101$ & $0.970\pm 0.004$ & 96.7 & 0.19\\
Fixed logdet & $0.024\pm 0.012$ & $0.221\pm 0.099$ & $0.586\pm 0.029$ & 94.9 & 0.19\\
Balanced logdet & $0.021\pm 0.009$ & $0.221\pm 0.101$ & $0.825\pm 0.015$ & 96.3 & 0.19\\
Task-orthogonal & $0.021\pm 0.009$ & $0.221\pm 0.101$ & $0.822\pm 0.015$ & 96.1 & 0.19\\
PCGrad (2 obj.) & $0.021\pm 0.009$ & $0.221\pm 0.101$ & $0.822\pm 0.014$ & 96.3 & 0.20\\
HCLM w/o risk & $0.021\pm 0.009$ & $0.221\pm 0.101$ & $0.825\pm 0.015$ & 96.3 & 0.34\\
Risk-only filter & $0.025\pm 0.012$ & $0.200\pm 0.109$ & $0.981\pm 0.017$ & 0.0 & 0.36\\
Protected balance & $0.025\pm 0.013$ & $0.200\pm 0.108$ & $0.799\pm 0.027$ & 0.0 & 0.50\\
HCLM (full) & $0.025\pm 0.013$ & $0.200\pm 0.108$ & $0.799\pm 0.027$ & 0.0 & 0.53\\
\bottomrule
\end{tabular}
\end{table}

Let $H_0=H(\theta_0)$ and let $H_T$ be its value after the 180 continuation updates. All protected methods respect the empirical caps. In Digits stress, full HCLM reduces $H_T/H_0$ from risk-only's 1.103 to 0.848, but group~0 test CE rises from 0.364 to 0.395. Relative to protected balancing, the extra entropy-rate constraint changes the ratio by $-0.0909\pm0.0111$ and group~0 test CE by $+0.0149\pm0.0031$; overall test CE rises by $0.00585\pm0.00252$. WDBC stress gives numerically identical full-HCLM and protected-balancing outcomes. Thus the risk filter explains cap compliance and the information request explains additional compression; predictive superiority is not established.

\begin{table}[t]
\centering
\caption{Paired differences, full HCLM minus the comparator. Entries are mean differences $\pm$ 95\% repeated-split variability interval half-widths. Negative loss differences favor full HCLM; a negative entropy-ratio difference only indicates additional compression. These intervals describe random splits and initializations of the same datasets, not independent populations.}
\label{tab:protected-ablation}
\small\setlength{\tabcolsep}{3pt}
\begin{tabular}{llrrr}
\toprule
Dataset / regime & Comparator & $\Delta$ test task CE & $\Delta$ group~0 CE & $\Delta(H_T/H_0)$\\
\midrule
Digits / uniform & Risk-only filter & $0.0212\pm 0.0046$ & $0.0205\pm 0.0054$ & $-0.2213\pm 0.0111$\\
Digits / uniform & Protected balance & $0.0112\pm 0.0021$ & $0.0108\pm 0.0024$ & $-0.0902\pm 0.0059$\\
Digits / 95/5 & Risk-only filter & $-0.0017\pm 0.0086$ & $0.0317\pm 0.0093$ & $-0.2547\pm 0.0113$\\
Digits / 95/5 & Protected balance & $-0.0021\pm 0.0034$ & $0.0149\pm 0.0031$ & $-0.0909\pm 0.0111$\\
WDBC / uniform & Risk-only filter & $0.0005\pm 0.0018$ & $0.0008\pm 0.0035$ & $-0.1841\pm 0.0197$\\
WDBC / uniform & Protected balance & $0.0001\pm 0.0002$ & $0.0002\pm 0.0007$ & $-0.0080\pm 0.0123$\\
WDBC / 95/5 & Risk-only filter & $0.0005\pm 0.0016$ & $-0.0002\pm 0.0026$ & $-0.1823\pm 0.0383$\\
WDBC / 95/5 & Protected balance & $-0.0000\pm 0.0000$ & $0.0000\pm 0.0000$ & $-0.0000\pm 0.0000$\\
\bottomrule
\end{tabular}
\end{table}

\begin{figure}[t]
\centering\includegraphics[width=\linewidth]{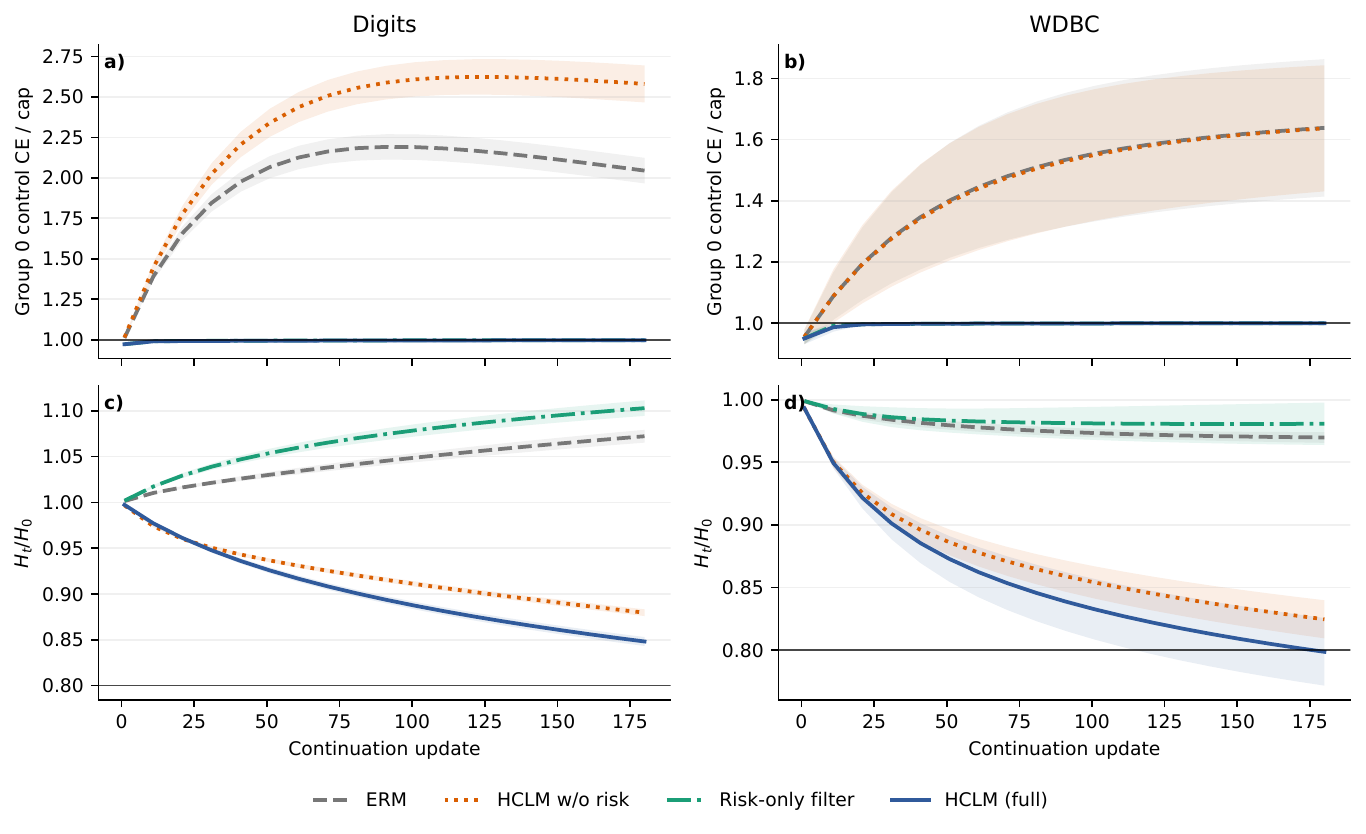}
\caption{Primary objective-reweighting study. Top: group~0 control CE divided by its fixed cap. Bottom: per-coordinate covariance-functional ratio. Curves average eight paired runs; bands describe repeated-split variability on the same datasets. The risk-only ablation isolates empirical protection, while the full controller retains a nonzero information intervention. Held-out costs appear separately in Tables~\ref{tab:protected-benchmark} and~\ref{tab:protected-ablation}.}\label{fig:protected-benchmark}
\end{figure}

In uniform Digits continuation, the risk filter is inactive and full HCLM equals its no-risk ablation, with test CE 0.195 versus ERM's 0.173. The final 288 runs accept every request and update, so rejection under infeasible specifications is evaluated by the exact and observable-bridge cases, not inferred from this benchmark. The target ratio 0.8 is not reached in the mean Digits stress trajectory. Intervals describe paired split/initialization variability, not independent populations. Equal update horizons also do not imply equal compute: actual gradient counts, forward checks and times are retained.

\subsection{Training with the observable restriction}\label{sec:observablebridge}
We implement the observable rule on binary Digits~0/1 and WDBC using eight paired splits and a four-unit tanh network. For an input $x\in\R^{d_x}$, representation coordinate $j\in\{1,\ldots,4\}$ is $z_j(x)=\exp(s_j)\tanh(A_j^\top x+b_j)$, where $A_j\in\R^{d_x}$ is a fixed input-weight vector, $b_j$ is a hidden bias and $s_j$ is a trainable latent log-scale. The corresponding readout weight is $\exp(-s_j)\omega_j$, with trainable coefficient $\omega_j$; an additional coefficient $\omega_0$ is the output bias. The resulting logit is independent of each $s_j$, so these scale coordinates are exactly prediction-null. The matrix $A=[A_1,\ldots,A_4]$ and a canonical 13-dimensional intervention subspace are declared before optimization; the ambient parameter dimensions are 269 and 133 for input dimensions $d_x=64$ and $30$, respectively. The four hidden biases, four readout coefficients, output bias and four log-scales are trainable. The representation entropy therefore has both observable and null derivatives. All gradients and Jacobians below are taken with respect to these 13 trainable coordinates using the Euclidean metric, so the theorem is applied to the restricted learner. The information functional is $H=H_{0.1}(Z)$ for the four-dimensional representation batch $Z$.

The map $F$ contains every training logit, ensuring $L=\ell(F)$. We compute its Jacobian projector by SVD and train task-only, unrestricted and observable information directions for 150 updates after a common 100-update warmstart. Both controllers use the coefficient cap $\min\{10,0.5\|g\|/\max(\|h\|,10^{-150})\}$, task fraction $1/2$ and entropy-rate request $r_H=0.05\|g\|\|h\|$. They report infeasibility and use the declared task-preserving fallback. Appendix~\ref{app:observablebridge} gives the complete protocol and derivative checks.

\begin{table}[t]
\centering\small
\caption{Observable-intervention bridge in the declared gauge-redundant neural family. Means over eight paired split/initialization runs; loss and entropy entries give 95\% repeated-split variability half-widths. Infeasibility concerns the requested entropy rate, not failure to take a task-preserving fallback update.}
\label{tab:observable-bridge}
\begin{tabular}{llrrr}
\toprule
Dataset & Rule & Test CE & $H_T/H_0$ & Infeasible request (\%)\\
\midrule
Digits 0/1 & Task only & $0.3557\pm0.1119$ & $0.9958\pm0.0044$ & --\\
Digits 0/1 & Unrestricted information & $0.3557\pm0.1119$ & $0.9177\pm0.0099$ & 0.00\\
Digits 0/1 & Observable information & $0.3556\pm0.1119$ & $0.9859\pm0.0103$ & 100.00\\
WDBC & Task only & $0.2441\pm0.0926$ & $0.9945\pm0.0054$ & --\\
WDBC & Unrestricted information & $0.2441\pm0.0926$ & $0.9698\pm0.0074$ & 0.00\\
WDBC & Observable information & $0.2439\pm0.0923$ & $0.9906\pm0.0068$ & 96.75\\
\bottomrule
\end{tabular}
\end{table}

Across 48 runs, unrestricted requests are always feasible, while observable requests are infeasible in all Digits updates and $96.75\%$ of WDBC updates. Only about $1$--$2\%$ of squared information-force energy is observable. The stronger unrestricted compression with nearly unchanged predictive losses demonstrates how a feasible entropy demand can be met largely through prediction-null motion. This deliberately redundant family tests the local mechanism, not the fresh-population certificate. A separate two-probe check in the appendix shows why unrelated minibatch probes require the general task crossproduct $b_{\rm obs}=g^\top\Pi_{\rm obs}h$ instead of $b=g^\top h$ under this Euclidean metric.

\subsection{A frozen-public-encoder sample channel}\label{sec:samplechannelbridge}
To test the sample-code condition with informative real labels, we use a restricted learner whose private-data interface is explicit. Before drawing any private training sample, a fixed public partition supplies feature standardization and a six-unit random tanh encoder. Adding a constant intercept coordinate gives the frozen feature map $\phi_{\rm pub}(x)\in\mathbb R^7$. Its public feature Gram matrix $K_{\rm pub}\in\R^{7\times7}$ is the average outer product of these features over the public inputs. This matrix and the ridge parameter $\lambda_{\rm ridge}=0.05$ are then frozen. The private sample is $S=((X_1,Y_1),\ldots,(X_n,Y_n))$, with $Y_i\in\{0,1\}$. The learner receives only the noisy moment $Q\in\R^7$ and maps it to readout coefficients $\omega_Q\in\R^7$ through
\begin{equation}\label{eq:frozencodebridge}
 Q=\frac1n\sum_{i=1}^n\phi_{\rm pub}(X_i)Y_i+\zeta,
 \qquad\zeta\sim N(0,\nu I_7),\qquad
 \omega_Q=(K_{\rm pub}+\lambda_{\rm ridge}I_7)^{-1}Q.
\end{equation}
Here $\nu>0$ is the per-coordinate noise variance, $I_7$ is the identity matrix, and $\zeta$ is independent of the private sample. The prediction is $\operatorname{clip}(\phi_{\rm pub}(x)^\top\omega_Q,0,1)$, where $\operatorname{clip}(v,0,1)=\min\{1,\max\{0,v\}\}$. The code is a label--feature cross-moment; this construction does not fit a sample-dependent feature covariance. The moment is a complete private-information interface for this learner, rather than a claim of statistical sufficiency for an unknown data-generating distribution. All preprocessing, encoder weights, $K_{\rm pub}$, $\lambda_{\rm ridge}$ and noise levels precede the private sample. Thus $S\to Q\to\omega_Q$ holds. Training-loss evaluation is performed by the evaluator and is never returned to the learner or a checkpoint selector.

The experimental population is uniform on the benchmark records outside the fixed public partition, and each private sample is drawn i.i.d. with replacement. Digits~0/1 supplies 108 public and 252 population records; WDBC supplies 170 and 399. This makes the sample-code covariance an exactly known evaluator quantity:
\[
 C_{\rm code}=\frac1n\operatorname{Cov}_{P}
      [\phi_{\rm pub}(X)Y].
\]
Squared prediction error against the binary label lies in $[0,1]$, giving the uniform loss sub-Gaussian variance proxy $\sigma_\ell^2=1/4$. Let $R_P(\omega_Q)$ be population squared loss and $\widehat R_S(\omega_Q)$ its average on the private training sample. Theorem~\ref{thm:gen} therefore gives
\begin{equation}\label{eq:frozencodebound}
 \left|\E_{S,\zeta}[R_P(\omega_Q)-\widehat R_S(\omega_Q)]\right|
 \le B_{\rm code}:=\sqrt{\frac1{4n}
 \log\det\left(I_7+\frac{C_{\rm code}}{\nu}\right)}.
\end{equation}
Appendix~\ref{app:channelbridge} details the channel and covariance calculation.

We executed 300 independent sample/noise draws for both datasets, with sample sizes 64 and 256 and noise variances $10^{-5}$, $10^{-4}$ and $10^{-3}$: 3,600 fitted readouts. The protocol was fixed before execution. Population loss is evaluated exactly over the finite population for each fitted readout; the expected generalization gap is estimated across the draws. Its Monte Carlo interval and the valid sample-code bound are reported separately in Table~\ref{tab:frozencodebridge}. The bound is $8.55$--$47.21$ times the absolute estimated expected gap. It is therefore valid in this restricted experiment but substantially looser than in the scalar-label example. Larger noise reduces the information bound while increasing mean population squared loss. At $n=256$, Digits mean accuracy falls from $97.38\%$ to $92.63\%$ across the displayed noise range. This is a controlled information--prediction trade-off, not an end-to-end neural generalization certificate or an external-site validation.

\begin{table}[t]
\centering\small\setlength{\tabcolsep}{4pt}
\caption{Frozen-public-encoder sample channel, endpoint noise settings at $n=256$. The gap is the Monte Carlo expected signed gap $\pm$ its $95\%$ interval half-width over 300 independent draws. Population risk is exact per fitted readout; the Gaussian sample-code bound uses exactly known code covariance. All twelve settings are retained in the supplement.}
\label{tab:frozencodebridge}
\begin{tabular}{lrrrrr}
\toprule
Dataset & $\nu$ & Pop. risk & Expected gap (MC) & $B_{\rm code}$ & Acc. (\%)\\
\midrule
Digits 0/1 & $10^{-5}$ & 0.0584 & $0.0027\pm0.0005$ & 0.1256 & 97.38\\
Digits 0/1 & $10^{-3}$ & 0.0803 & $0.0030\pm0.0008$ & 0.0358 & 92.63\\
WDBC & $10^{-5}$ & 0.0920 & $0.0035\pm0.0012$ & 0.1527 & 87.97\\
WDBC & $10^{-3}$ & 0.1057 & $0.0060\pm0.0013$ & 0.0516 & 85.96\\
\bottomrule
\end{tabular}
\end{table}

\subsection{Fixed-coefficient sensitivity and nuisance robustness}
The original fixed coefficient is $\beta_0=\min\{20,0.5\|g_0\|/\max(\|h_0\|,10^{-30})\}$, where $g_0$ and $h_0$ are the task and normalized information gradients at the common warmstart. A follow-up coefficient grid evaluates the fixed choices $\beta=\alpha\beta_0$ with multipliers $\alpha\in\{ 0.01,0.03,0.1,0.3,1,3\}$ on the same 32 paired experiments, adding 192 runs; the zero endpoint reuses ERM. Every coefficient is reported, with control-set trade-offs in Figure~\ref{fig:fixed-grid} and complete held-out outcomes in Table~\ref{tab:fixed-grid}. In Digits 95/5, multiplier 0.01 gives group 0 test CE 0.629, compared with 2.004 at the original multiplier 1 and 0.625 for ERM. Its entropy ratio is 1.066, compared with 0.474 at multiplier 1. All tested fixed coefficients exceed a control cap at every stress update in that dataset. The grid separates coefficient sensitivity from the additional requirement of empirical risk protection.

\begin{figure}[t]
\centering\includegraphics[width=\linewidth]{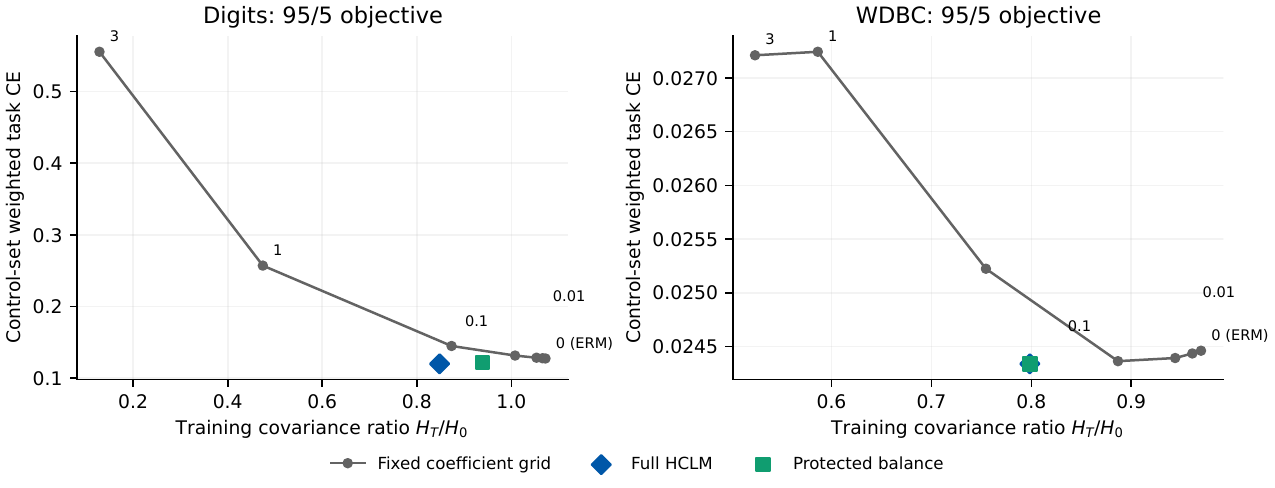}
\caption{Control-set trade-offs under the declared objective reweighting. The fixed-coefficient curve contains all prespecified multipliers, including the retained ERM endpoint0. Labels identify selected multipliers, not selected winners. Lines connect evaluated points; no globally optimal Pareto frontier is claimed. Full HCLM and protected balancing use their original fixed specifications; their WDBC mean points coincide. Uncertainty and held-out results for every grid point appear in Table~\ref{tab:fixed-grid}.}
\label{fig:fixed-grid}
\end{figure}

To examine an information functional that targets a specified source of instability, we add a two-hidden-layer minibatch study on a declared Fashion-MNIST subset \citep{xiao2017fashion}. Each of eight paired runs uses 8,000 fitting images, 256 reused risk-control images, 1,744 validation images used only for terminal reporting, and 2,000 images from the official test set. Four bounded features supply a deliberately constructed parity shortcut. The training shortcut agrees with label parity with probability $0.9$; the shifted test reverses that correlation. The nuisance-targeted information functional is the normalized covariance log determinant of paired encoder differences $\Delta z=e_\theta(x,n_1)-e_\theta(x,n_2)$, where $x$ is the stable pooled image and $n_1,n_2\in\R^4$ are independent nuisance draws at that image. Each vector $\Delta z\in\R^8$ measures the response to nuisance alone; its centered batch covariance defines $H_\Delta$ with the same per-coordinate normalization and reference variance $\nu=0.1$ as the main functional. It is distinct from global representation compression. All methods share subsets, initializations, warmstarts, minibatches and nuisance draws; no configuration is selected using test outcomes.

\begin{table}[t]
\centering\small\setlength{\tabcolsep}{3pt}
\caption{Fashion-MNIST paired-nuisance experiment: eight repeated subsets/initializations, two-hidden-layer minibatch learning. Entries are means $\pm$ 95\% repeated-split variability half-widths. Accuracies are percentages. The nuisance-response functional $H_\Delta$ differs from global representation volume. All methods have zero observed control-cap violations; these are empirical checks on the reused 256-example control set.}
\label{tab:fashion-nuisance}
\begin{tabular}{lrrrrr}
\toprule
Method & IID acc. & Reversed acc. & Zero acc. & Reversed CE & $H_{\Delta,T}/H_{\Delta,0}$\\
\midrule
ERM & $78.82\pm 0.52$ & $70.73\pm 0.82$ & $76.71\pm 0.71$ & $0.853\pm 0.031$ & $0.846\pm 0.095$\\
Risk-only filter & $78.83\pm 0.51$ & $70.74\pm 0.82$ & $76.71\pm 0.71$ & $0.853\pm 0.031$ & $0.846\pm 0.095$\\
Protected balance & $78.38\pm 0.61$ & $75.72\pm 0.71$ & $77.49\pm 0.73$ & $0.683\pm 0.023$ & $0.221\pm 0.049$\\
HCLM (full) & $78.42\pm 0.63$ & $75.86\pm 0.74$ & $77.55\pm 0.70$ & $0.679\pm 0.023$ & $0.211\pm 0.051$\\
\bottomrule
\end{tabular}
\end{table}

Full HCLM increases reversed-nuisance accuracy relative to ERM by $5.13\pm0.58$ percentage points, with a paired shifted-CE difference of $-0.1735\pm0.0222$. Its IID-accuracy difference is $-0.39\pm0.55$ points. Protected balancing already obtains most of this robustness gain: the additional entropy-rate constraint changes reversed accuracy by only $0.14\pm0.13$ points and shifted CE by $-0.00335\pm0.00281$. These are repeated-subset/initialization variability intervals on one benchmark. Every method remains below both empirical control caps, so this setting provides no observed cap-protection advantage. The result supports nuisance-targeted information control under this known intervention model, rather than a general benefit of shrinking global covariance or a unique predictive advantage of the full controller. Appendix protocol and retained raw records give all three test scenarios and unfavorable outcomes.

The numerical implementation is also checked against the analytic results. The independent 100-case geometry check gives a maximum directional finite-difference error of $7.75\times10^{-10}$ for the normalized gradient. The accompanying records also retain 63,000 scalar-interval checks, 1,500 metric-projection cases and 500 independent robust primal--dual comparisons. Exact finite-design probe experiments expose conservative certificates, and binary Gaussian examples quantify entropy-proxy gaps and sample-code bound tightness. Appendix~\ref{app:informationchecks} reports these scoped checks and their assumptions, including unfavorable cases. They complement rather than replace the proofs and neural studies.

\section{Discussion and limitations}\label{sec:discussion}\label{sec:perspectives}
HCLM contributes a controlled-learning specification and a characterization of its admissible information actions. The task, representation, risk and authority components have separate roles. The experiments show why these roles matter: prediction-null scaling can satisfy entropy requests; a risk filter can protect an empirical boundary independently of entropy; and a stronger compression demand can worsen held-out loss. The additional studies test the hypotheses suggested by these observations rather than converting one favorable curve into a universal learning claim.

\subsection{What the evidence establishes}
The scalar interval, joint feasibility and error-aware certificates describe attainable rates under specified gradients, directions and budgets. The observable bridge trains the projected rule where task factorization holds, while the small-probe diagnostic checks the necessary alternative crossproduct. The frozen-public-encoder study implements the complete noisy sample interface and measures bound slack on a declared finite population. These bridges address different assumptions; one study does not validate the others.

The full end-to-end benchmark reuses its control set and has no population-risk certificate. The fresh-probe results require that the model and intervention subspace precede the fresh batch, with justified global derivative bounds. Those constants can make a certificate uninformative in rare directions. Likewise, delay tracking requires a feasible request, a justified drift bound and an enforced response. A change in these prerequisites is a change in the guarantee. Human intervention is represented by externally specified limits and authority; the experiments do not estimate how well people choose or use them.

\subsection{Research directions with falsifiable targets}
The current observation map and parameter metric are declared. Learning a useful low-dimensional intervention space from reused observations requires a uniform-selection or stability argument beyond the fresh-batch result. A useful next theorem would replace the ambient-subspace dimension by an effective dimension while preserving rare, high-consequence directions. The associated test should compare empirical rank, certified action and harmful omitted directions under matched probe budgets.

Global covariance volume cannot distinguish nuisance variation from task-critical rare features. A targeted functional requires an explicit source model or justified interventions. Future work should identify conditions under which reducing that functional preserves a sufficient predictive subspace, then test those conditions across held-out environments. Matched-risk comparisons with generic constrained learning are essential: compression alone is not the endpoint.

Mini-batch control needs calibrated, simultaneous bounds for the gradient and covariance errors selected along the learning trajectory. The main theoretical target is a useful guarantee that charges data reuse and selection while bounding both false certification and conservative rejection. Time-uniform methods \citep{howard2021} and stability-based audits offer starting points. Experiments should vary batch size, latent dimension, uncertainty and feedback delay, with computation and rejection recorded together.

A prospective study should compare the ability to recognize incompatible requests, choose valid revisions and suspend adaptation under misleading evidence. Outcomes should include harmful threshold relaxation, delayed intervention, overreliance on a certificate and retained independent competence. Distributed and physical systems further require aligned client coordinates, communication/error budgets and a separate plant-level protection model. Scaling laws would need held-out predictions derived from a specified model rather than assumed power-law balances.

\section{Conclusion}\label{sec:conclusion}
HCLM specifies how task progress, representation information, structural restrictions and external authority enter a controlled learning update. Its mathematical object is the set of admissible learning actions under a declared specification. A force can be large but prediction-null, an entropy request can be incompatible with task progress, and aggregate improvement can harm a separately protected group. The spectral, observable and feasibility results distinguish these cases. Fresh-probe and delay analyses further state what information is required to certify selected actions and what conservatism that information imposes.

The neural experiments retain nonzero information motion while checking task and empirical risk constraints. They also expose a cost: the extra entropy-rate requirement can worsen held-out group loss relative to generic protected balancing, and can have no additional effect in another dataset. The risk filter explains empirical cap compliance; entropy control explains additional representation compression. Neither result alone demonstrates improved population risk or effective human supervision. The exact information examples separately show how loose a covariance entropy proxy can be and when a valid sample-code bound can track an expected generalization gap.

The next steps are nonvacuous certification with adaptive data reuse, scalable observable subspaces, protection under specified distribution shifts, and prospective studies of how people set and revise learning constraints. These questions preserve HCLM's human-centered purpose while assigning each claimed benefit a measurable outcome and an appropriate proof or evaluation.

\label{maintext:end}
\clearpage\label{supplement:start}
\appendix
\setcounter{page}{1}\renewcommand{\thepage}{S\arabic{page}}
\section{Detailed proofs}\label{app:detailedproofs}

\subsection{Proof of Theorem~\ref{thm:spectral}}\label{proof:thm:spectral}
\begin{proof}
Put $V=C_Z+\nu I_p$. Since $C_Z$ is positive semidefinite and $\nu>0$, $V$ is positive definite and its inverse exists. The identity $H_\nu=\tfrac1{2p}\log\det V-\tfrac12\log\nu$ gives
\[
 dH_\nu=\tfrac1{2p}\tr(V^{-1}dC_Z),\qquad
 dC_Z=\frac{(P_m dZ)^\top Z_c+Z_c^\top P_m dZ}{m-1}.
\]
Here $dZ$ denotes an arbitrary matrix perturbation. Because $V^{-1}$ and $P_m$ are symmetric and $P_mZ_c=Z_c$, cyclic invariance of trace makes the two terms equal. Thus
\[
 dH_\nu=\frac1{p(m-1)}\tr(V^{-1}Z_c^\top dZ)
 =\tr\!\left[\left(\frac{Z_cV^{-1}}{p(m-1)}\right)^\top dZ\right].
\]
The Frobenius-gradient definition $dH_\nu=\tr[(\nabla_ZH_\nu)^\top dZ]$ proves \eqref{eq:zgradient}.

Take a thin singular value decomposition $Z_c=U_s\operatorname{diag}(s_1,\ldots,s_{r_Z})V_s^\top$, where $U_s\in\R^{m\times r_Z}$ and $V_s\in\R^{p\times r_Z}$ have orthonormal columns. On the $j$th right singular direction, $C_Z$ has eigenvalue $s_j^2/(m-1)$. Therefore
\[
 \nabla_ZH_\nu=U_s\operatorname{diag}\!\left(
 \frac{s_j}{p[s_j^2+(m-1)\nu]}\right)_{j=1}^{r_Z} V_s^\top.
\]
The Frobenius norm is the square root of the sum of squared singular values. Squaring gives both equalities in \eqref{eq:spectral}. Finally, for $f(x)=x/(x+\nu)^2$,
\[
 f'(x)=\frac{\nu-x}{(x+\nu)^3},\qquad f(\nu)=\frac1{4\nu}.
\]
It increases up to $\nu$, decreases afterward, and tends to zero at zero and infinity. Applying the scalar bound to each of the $r_Z$ nonzero modes proves the inequality and all limiting claims.
\end{proof}

\subsection{Proof of Proposition~\ref{prop:cov}}\label{proof:prop:cov}
\begin{proof}
Let $D=\widetilde C-C$ and $C_t=C+tD$, $0\le t\le1$. A convex combination of positive semidefinite matrices remains positive semidefinite, so $\norm{(C_t+\nu I)^{-1}}_{\rm op}\le1/\nu$. The fundamental theorem of calculus gives
\[
 \mathcal H_\nu(\widetilde C)-\mathcal H_\nu(C)
 =\tfrac1{2p}\int_0^1\tr[(C_t+\nu I)^{-1}D]dt.
\]
For $p\times p$ matrices, $|\tr(AD)|\le p\norm A_{\rm op}\norm D_{\rm op}$ in this symmetric setting. Taking absolute values and using $\norm D_{\rm op}\le\delta$ proves the value bound.

The inverse difference has the exact resolvent form
\[
 (\widetilde C+\nu I)^{-1}-(C+\nu I)^{-1}
 =(\widetilde C+\nu I)^{-1}(C-\widetilde C)(C+\nu I)^{-1}.
\]
Submultiplicativity bounds its operator norm by $\delta/\nu^2$. Next use $\norm{AB}_F\le\norm A_F\norm B_{\rm op}$ with $A=Z_c/[p(m-1)]$. This proves \eqref{eq:forceerror}. Vectorization preserves Frobenius norm, and $\norm{J_Z^\top x}\le\norm{J_Z}_{\rm op}\norm x$, giving the parameter-gradient bound. No derivative of the covariance approximation is assumed in this calculation.
\end{proof}

\subsection{Proof of Proposition~\ref{prop:gauge}}\label{proof:prop:gauge}
\begin{proof}
The product $W_\gamma e_\gamma$ equals $We$, and the representation covariance becomes $\gamma^2C$. Diagonalizing $C$ gives \eqref{eq:gauge}. Its derivative with respect to $\log\gamma$ is
\[
 \frac1p\sum_{j=1}^p\frac{\gamma^2\lambda_j}{\nu+\gamma^2\lambda_j},
\]
which is positive whenever at least one $\lambda_j>0$. Each term in the covariance functional tends to zero as $\gamma\downarrow0$, while the predictor is unchanged.
\end{proof}

\subsection{Proof of Proposition~\ref{prop:readout}}\label{proof:prop:readout}
\begin{proof}
If $V$ and $V'$ are independent copies with finite second moments, expansion of the square gives
$\E\norm{V-V'}^2=2\E\norm{V-\E V}^2=2\tr\Cov V$.
Apply this identity first to $V=r(Z)$ and then to $Z$. Lipschitz continuity gives
\[
 v\le\tr\Cov(r(Z))
 =\tfrac12\E\norm{r(Z)-r(Z')}^2
 \le\tfrac{L_r^2}2\E\norm{Z-Z'}^2=L_r^2\tr\Cov Z.
\]
Put $C=\Cov(Z)$, and let $\lambda_1,\ldots,\lambda_p$ be its nonnegative eigenvalues. Hence $\sum_j\lambda_j=\tr\Cov Z\ge v/L_r^2$. Expanding the product of the nonnegative factors $1+\lambda_j/\nu$ yields a constant term, the sum of the linear terms, and only nonnegative higher-order terms. Thus
\[
 \det(I+C/\nu)=\prod_j(1+\lambda_j/\nu)
 \ge1+\sum_j\lambda_j/\nu\ge1+v/(L_r^2\nu).
\]
The logarithm is increasing on positive numbers; dividing it by $2p$ proves the result.
\end{proof}

\subsection{Proof of Theorem~\ref{thm:projection}}\label{proof:thm:projection}
\begin{proof}
Let $A=J_FM^{-1/2}$, with the symmetric positive definite square root of $M$. An SVD of $A$ shows that
$Q=A^\top(AA^\top)^\dagger A$ acts as identity on the right singular vectors with nonzero singular values and as zero on their orthogonal complement. Consequently $Q^2=Q$, $Q^\top=Q$, and $\ker Q=\ker A$, including when $A$ is rank deficient. Direct substitution gives $\Pi_{\rm obs}=M^{-1/2}QM^{1/2}$. It follows that
\[
 \Pi_{\rm obs}^2=\Pi_{\rm obs},\qquad \Pi_{\rm obs}^\top M=M\Pi_{\rm obs},
 \qquad\operatorname{range}\Pi_{\rm obs}=\operatorname{range}(M^{-1}J_F^\top).
\]
Since $A(I-Q)=0$, one has $J_F(I-\Pi_{\rm obs})=0$. Therefore $J_Fu_{\rm null}=0$ and $J_Fu_{\rm obs}=J_Fu_H$. The self-adjoint identity $\Pi_{\rm obs}^\top M=M\Pi_{\rm obs}$ gives
$\ip{\Pi_{\rm obs}u_H}{(I-\Pi_{\rm obs})u_H}_M=0$. Expanding the squared norm of the sum proves the Pythagorean decomposition. If $J_Fu_H=0$, then $AM^{1/2}u_H=0$, so $QM^{1/2}u_H=0$ and $\Pi_{\rm obs}u_H=0$. The converse follows from $J_F\Pi_{\rm obs}=J_F$.

For a smooth invertible reparameterization, put $T=D\psi$. The chain rule yields
$h_\phi=T^\top h_\theta$ and $J_{F,\phi}=J_{F,\theta}T$, while the prescribed metric is $M_\phi=T^\top M_\theta T$. Inverting the latter gives $u_{H,\phi}=T^{-1}u_{H,\theta}$. Moreover,
\[
 J_{F,\phi}M_\phi^{-1}J_{F,\phi}^\top
 =J_{F,\theta}M_\theta^{-1}J_{F,\theta}^\top,
 \qquad \Pi_{\rm obs,\phi}=T^{-1}\Pi_{\rm obs,\theta} T.
\]
Thus both projected velocities transform by $T^{-1}$, their metric norms are unchanged, and multiplying by the corresponding Jacobian gives the same output velocities. This is an infinitesimal coordinate-invariance statement; a finite Euler step is not itself invariant under an arbitrary nonlinear change of coordinates.
\end{proof}

\subsection{Proof of Proposition~\ref{prop:uniformprobe}}\label{proof:prop:uniformprobe}
\begin{proof}
Condition on $\mathcal A$ throughout. For any fixed unit $z\in\mathbb R^{k_U}$,
the independent variables $Y_i=\|A(X_i)z\|^2$ lie in $[0,K^2]$, so Hoeffding gives
\[
 \Prob\{|z^\top(\widehat C_F-C_F)z|>t\}
 \le2\exp(-2mt^2/K^4).
\]
The case $K=0$ is deterministic. For $K>0$, take a $1/4$-net $\mathcal N$
of the unit sphere with cardinality at most $9^{k_U}$. A maximal separated
set has this property: disjoint radius-$1/8$ balls fit inside a radius-$9/8$
ball, giving the cardinality bound by volume comparison, and maximality gives
the covering property. A union bound with
\[
 t=K^2\sqrt{[k_U\log9+\log(2/\delta)]/(2m)}
\]
therefore controls all net quadratic forms except on an event of probability
at most $\delta$.

Put $E_F=\widehat C_F-C_F$. Since it is symmetric, some unit $x$ attains
$|x^\top E_Fx|=\|E_F\|_{\rm op}$. Choose $z\in\mathcal N$ within $1/4$ of $x$.
Then
\[
 |x^\top E_Fx-z^\top E_Fz|
 \le |(x-z)^\top E_Fx|+|z^\top E_F(x-z)|
 \le\tfrac12\|E_F\|_{\rm op}.
\]
It follows that $\|E_F\|_{\rm op}\le2\max_{z\in\mathcal N}|z^\top E_Fz|\le2t$.
Also $0\preceq C_F,\widehat C_F\preceq K^2I_{k_U}$ gives the deterministic
bound $\|E_F\|_{\rm op}\le K^2$. Combining these proves the stated
$\epsilon_m$. For every $z_U$,
\[
 z_U^\top C_F z_U=\mathbb E_P\|A(X)z_U\|^2,\qquad
 |z_U^\top(\widehat C_F-C_F)z_U|\le\epsilon_m\|z_U\|^2.
\]
These inequalities hold for all directions on the same event, including a
direction selected using the probes. Set $z_U^\top\widehat C_Fz_U=0$ for the
null-direction assertion. For unit $z_U$, the two inequalities
$z_U^\top C_Fz_U\ge z_U^\top\widehat C_Fz_U-\epsilon_m$ and
$z_U^\top\widehat C_Fz_U\ge z_U^\top C_Fz_U-\epsilon_m$
give the eigenvalue conclusions by taking infima.
\end{proof}

\subsection{Proof of Corollary~\ref{cor:proberisk}}\label{proof:cor:proberisk}
\begin{proof}
Taylor's formula with integral remainder along the segment gives
\[
 f_{\theta+\eta Tz_U}(x)-f_\theta(x)=\eta A(x)z_U+e_x,
 \qquad\norm{e_x}\le\tfrac12L_A\eta^2\norm{z_U}^2.
\]
Minkowski's inequality for the $L^2(P)$ norm and Proposition~\ref{prop:uniformprobe} prove \eqref{eq:probefinite}. Lipschitz continuity, then Cauchy--Schwarz, give
\[
 R_{P_{XY}}(\theta+\eta Tz_U)-R_{P_{XY}}(\theta)
 \le L_{\rm harm}\E_P\norm{f_{\theta+\eta Tz_U}(X)-f_\theta(X)}
 \le L_{\rm harm}\left(\E_P\norm{f_{\theta+\eta Tz_U}(X)-f_\theta(X)}^2\right)^{1/2}.
\]
Conditional on $\mathcal A$, the current $\theta$ and joint law $P_{XY}$ were fixed before the fresh labeled pairs. The variables $\ell_{\rm harm}(Y_i,f_\theta(X_i))$ are therefore conditionally independent, lie in $[0,1]$, and share conditional mean $R_{P_{XY}}(\theta)$. The one-sided Hoeffding bound gives $R_{P_{XY}}(\theta)\le u_R$ with conditional probability at least $1-\delta_R$. A union bound combines this event with \eqref{eq:probeuniform}; no independence between the two events is needed. On their intersection both statements hold uniformly over the allowed candidate directions, proving \eqref{eq:probeprotected} even after selection using the probes.
\end{proof}

\subsection{Proof of Corollary~\ref{cor:angle}}\label{proof:cor:angle}
\begin{proof}
Theorem~\ref{thm:feasible} gives the lower endpoint $-b/c$ and upper endpoint $(1-\kappa)a/(-b)$. Both denominators are positive. Multiplying the inequality between endpoints by $c(-b)$ shows that the interval is nonempty exactly when $b^2\le(1-\kappa)ac$. Dividing by $ac>0$ gives the angle form. Equality gives one feasible coefficient; strict reversal gives none.
\end{proof}

\subsection{Proof of Theorem~\ref{thm:obscontrol}}\label{proof:thm:obscontrol}
\begin{proof}
The factorization and the chain rule give $g=J_F^\top\nabla\ell(F(\theta))$. Thus $M^{-1}g$ belongs to $\operatorname{range}\Pi_{\rm obs}$, so $\Pi_{\rm obs}M^{-1}g=M^{-1}g$. Since $\Pi_{\rm obs}$ is self-adjoint in the metric,
\[
 g^\top \Pi_{\rm obs}M^{-1}h=\ip{M^{-1}g}{\Pi_{\rm obs}M^{-1}h}_M
 =\ip{\Pi_{\rm obs}M^{-1}g}{M^{-1}h}_M=b.
\]
The orthogonal decomposition of $M^{-1}h$ gives
$h^\top \Pi_{\rm obs}M^{-1}h=\ip{M^{-1}h}{\Pi_{\rm obs}M^{-1}h}_M
=\norm{\Pi_{\rm obs}M^{-1}h}_M^2=c_{\rm obs}$.
Multiplying the velocity in \eqref{eq:obsflow} by $g^\top$ and $h^\top$ proves the two rates.

For a fixed $\beta\ge0$, $b+\beta c\ge b+\beta c_{\rm obs}$ because $c_{\rm obs}\le c$. Hence meeting the restricted entropy inequality implies meeting the unrestricted one, while the task inequality and coefficient cap are unchanged. This proves set inclusion even for signed $r_H$. When $c_{\rm obs}>0$, repeat the scalar interval calculation with this value; the angle condition follows as in Corollary~\ref{cor:angle}. If $c_{\rm obs}=0$, the projected vector is zero and the first displayed identity gives $b=0$. The restricted entropy rate is then zero, so a positive request cannot be met.
\end{proof}

\subsection{Proof of Theorem~\ref{thm:budget}}\label{proof:thm:budget}
\begin{proof}
Put $\widetilde g=M^{-1/2}g$, $\widetilde h=M^{-1/2}h$, and $\widetilde u=M^{1/2}u$. The constraints become $\widetilde g^\top \widetilde u\le0$ and $\norm{\widetilde u}\le B_{\rm int}$, while the objective is $-\widetilde h^\top \widetilde u$. Also $a=\norm{\widetilde g}^2>0$ and $b=\widetilde g^\top\widetilde h$.

If $b\ge0$ and $h\ne0$, Cauchy--Schwarz bounds the objective by $B_{\rm int}\norm{\widetilde h}$, attained by $\widetilde u=-B_{\rm int}\widetilde h/\norm{\widetilde h}$. Its inner product with $\widetilde g$ is $-B_{\rm int}b/\norm{\widetilde h}\le0$, so it is feasible. If $h=0$, the objective is identically zero.

If $b<0$, decompose $\widetilde h=(b/a)\widetilde g+\widetilde h_{\rm eff}$, where $\widetilde h_{\rm eff}=M^{-1/2}[h-(b/a)g]$ is orthogonal to $\widetilde g$. Then
\[
 -\widetilde h^\top \widetilde u
 =-(b/a)\widetilde g^\top \widetilde u-\widetilde h_{\rm eff}^\top \widetilde u
 \le-\widetilde h_{\rm eff}^\top \widetilde u\le B_{\rm int}\norm{\widetilde h_{\rm eff}}.
\]
The first inequality uses $-(b/a)>0$ and $\widetilde g^\top \widetilde u\le0$; its sign is essential. For $\widetilde h_{\rm eff}\ne0$, taking $\widetilde u=-B_{\rm int}\widetilde h_{\rm eff}/\norm{\widetilde h_{\rm eff}}$ makes the task constraint an equality and attains the bound. Expanding its norm gives
$\norm{\widetilde h_{\rm eff}}^2=c-2b^2/a+b^2/a=c-b^2/a$.
If $\widetilde h_{\rm eff}=0$, the objective is at most zero and $\widetilde u=0$ attains zero. Transforming back to $u$ proves the complete formula, including zero budget.
\end{proof}

\subsection{Proof of Theorem~\ref{thm:discrete}}\label{proof:thm:discrete}
\begin{proof}
The descent lemma and the two conditions in \eqref{eq:safeguards} give
\[
 L(\theta_t+\eta w_t)
 \le L(\theta_t)+\eta g_t^\top w_t
       +\frac{\ell_L\eta^2}{2}\|w_t\|^2
 \le L(\theta_t)-\eta\left(\kappa-\frac{\ell_L\eta C_w}{2}\right)a_t.
\]
The stated step restriction makes the final parenthesis at least $\kappa/2$.
Sum this inequality for $t=0,\ldots,T-1$ and telescope:
\[
 \frac{\eta\kappa}{2}\sum_{t<T}a_t
 \le L(\theta_0)-L(\theta_T)
 \le L(\theta_0)-L_{\inf}.
\]
The minimum is bounded by the average, proving the finite-time claim.
If $M_t\preceq m_1I$, then $M_t^{-1}\succeq I/m_1$, and hence
$\|g_t\|^2\le m_1a_t$. Finally the displayed sum is uniformly bounded as
$T\to\infty$, so the nonnegative sequence $a_t$ is summable and tends to zero.
This conclusion is task stationarity under the specified metric conditions;
it does not assert convergence of the parameter vector or attainment of an
entropy target.
\end{proof}

\subsection{Proof of Proposition~\ref{prop:humanrisk}}\label{proof:prop:humanrisk}
\begin{proof}
Since $X^2=1$, $(f_\theta(X)-Y)^2=(\theta-A_{\rm grp})^2$. Conditioning on $A_{\rm grp}$ gives \eqref{eq:groupL}. Since $\E X=0$ and $\E X^2=1$, the population representation variance is $\theta^2$, giving \eqref{eq:grouprisk}. Differentiation yields
\[
 g=\theta-\varpi>0,\qquad h=\frac{\theta}{\nu+\theta^2}>0,
 \qquad R_1'(\theta)=\theta-1<0.
\]
Thus $w:=\dot\theta=-g-\beta h<0$. The chain rule gives $\dot L=gw<0$, $\dot H=hw<0$, and $\dot R_1=(\theta-1)w>0$.

For an arbitrary velocity $w$, the task requirement $gw\le-\kappa g^2$, with $\kappa>0$, is equivalent to $w\le-\kappa g<0$. The group requirement $(\theta-1)w\le0$ is equivalent to $w\ge0$. Their intersection is empty. This impossibility holds for every intervention direction, not only for a scalar entropy penalty.
\end{proof}

\subsection{Proof of Theorem~\ref{thm:multirisk}}\label{proof:thm:multirisk}
\begin{proof}
For a vector of real numbers $z_i$, $\max_i z_i=\max_{\lambda\in\Delta_N}\sum_i\lambda_i z_i$: every weighted average is at most the largest component, and a unit mass at a largest component attains it. Therefore \eqref{eq:minviolation} is
\[
 \min_{u\in\mathcal U,\,\norm u_M\le B_{\rm int}}
 \max_{\lambda\in\Delta_N}
 \left(\sum_i\lambda_i\bar d_i\right)^\top u-\sum_i\lambda_i\bar t_i.
\]
The intervention ball and the simplex are nonempty compact convex sets in finite dimensions. The expression is continuous and affine in each argument. The finite-dimensional convex minimax theorem therefore permits exchange of the minimum and maximum; these are the precise conditions used for that exchange.

For fixed $\lambda$, put $z=\sum_i\lambda_i\bar d_i$ and $p=\Pi_{\mathcal U}M^{-1}z$. For every $u\in\mathcal U$, orthogonality gives
\[
 z^\top u=\ip{M^{-1}z}{u}_M=\ip p u_M\ge-B_{\rm int}\norm p_M.
\]
If $p\ne0$, the choice $u=-B_{\rm int}p/\norm p_M$ attains the last bound. If $p=0$, the expression is zero for all $u\in\mathcal U$, including $u=0$. Substitution proves \eqref{eq:dualcertificate}. Continuity on compact sets proves attainment, including $B_{\rm int}=0$.

An intervention satisfies every row exactly when its largest normalized residual is nonpositive. Attainment then gives the stated equivalence with $\varepsilon_\star\le0$, including the boundary value zero. For any fixed $\lambda$, its weighted residual is bounded above by the largest residual, and bounded below by the expression in \eqref{eq:dualcertificate}; positivity is therefore a certificate valid for every admissible $u$. Finally, restricting a minimization domain cannot decrease its minimum, which proves \eqref{eq:monotonesubspace}; enlarging the budget enlarges that domain and proves the last claim.
\end{proof}

\subsection{Proof of Theorem~\ref{thm:robustmultirisk}}\label{proof:thm:robustmultirisk}
\begin{proof}
For fixed $u$, metric Cauchy--Schwarz bounds the worst directional error by $\delta_i^d\norm u_M$. If $u\ne0$, the error $\bar d_i-\widehat d_i=\delta_i^d Mu/\norm u_M$ attains it and has dual norm $\delta_i^d$. For $u=0$, the directional error is zero. The worst right-hand side is $\bar t_i=\widehat t_i-\delta_i^t$. This proves the exact robust residual in \eqref{eq:robustprimal}.

Replace the maximum over rows by a maximum over $\Delta_N$. For a fixed $\lambda$, the function of $u$ is convex, since its norm coefficient is nonnegative; for fixed $u$, it is affine in $\lambda$. Continuity and the same compact convex domains as in Theorem~\ref{thm:multirisk} permit the minimax exchange. Put
\[
 p=\Pi_{\mathcal U}M^{-1}\sum_i\lambda_i\widehat d_i,
 \qquad \tau_\lambda=\sum_i\lambda_i\delta_i^d.
\]
The part depending on $u$ is $\ip p u_M+\tau_\lambda\norm u_M$. When $p\ne0$, at a fixed radius $s=\norm u_M$ its minimum is $s(\tau_\lambda-\norm p_M)$, attained opposite to $p$. Minimizing over $s\in[0,B_{\rm int}]$ gives $-B_{\rm int}[\norm p_M-\tau_\lambda]_+$: choose zero radius if the coefficient is nonnegative and radius $B_{\rm int}$ otherwise. If $p=0$, $u=0$ gives the same formula, including the zero-dimensional subspace. Adding the constant term proves \eqref{eq:robustdual}.

Compactness guarantees attainment, so a nonpositive optimal residual is equivalent to existence of one intervention satisfying all uncertain rows. Under \eqref{eq:uncertaintyobstruction}, the positive-part term vanishes for that multiplier, while the constant term is positive. Its dual value is therefore a positive lower bound independent of $B_{\rm int}$, proving the obstruction.
\end{proof}

\subsection{Proof of Theorem~\ref{thm:controlleracceptance}}\label{proof:thm:controlleracceptance}
\begin{proof}
For a fixed admissible $u$, put
$f_i(u)=\bar d_i^\top u-\bar t_i$ and
$\widehat f_i^{\rm rob}(u)=\widehat d_i^\top u-\widehat t_i
+\delta_i^d\norm u_M+\delta_i^t$.
Metric Cauchy--Schwarz and the error event give
\[
 |(\bar d_i-\widehat d_i)^\top u|
 \le\delta_i^d\norm u_M,
 \qquad |\bar t_i-\widehat t_i|\le\delta_i^t.
\]
Consequently,
\[
 f_i(u)\le\widehat f_i^{\rm rob}(u)
 \le f_i(u)+2\delta_i^d\norm u_M+2\delta_i^t
 \le f_i(u)+2\chi.
\]
Taking the maximum over $i$ preserves both inequalities. Taking the minimum
over the same compact admissible ball then proves
\eqref{eq:controllerresidualsandwich}. In particular, a nonpositive robust
residual bounds every true residual above by zero, and
$\varepsilon_\star\le-2\chi$ implies $\varepsilon_{\rm rob}\le0$.
For the last assertion, add the worst unreported error
$\epsilon_{\rm opt}$ to the reported residual; the result remains nonpositive.
\end{proof}

\subsection{Proof of Proposition~\ref{prop:delayedentropytracking}}\label{proof:prop:delayedentropytracking}
\begin{proof}
The drift bound gives
$H_\star(t)\ge H_\star(\sigma_t)-v_\star\Delta$ and
$\dot H_\star(t)\ge-v_\star$ almost everywhere. The entropy estimate gives
$\widehat H_t+\epsilon_H\ge H(\theta_t)$. Hence
$r_{\rm del}\ge k_H(H-H_\star)+v_\star$, so
$\dot e_H=\dot H-\dot H_\star\le-k_He_H-v_\star-\dot H_\star\le-k_He_H$.
Multiplication by $e^{k_Ht}$ and integration prove
\eqref{eq:delayedentropytracking}. For the final bound, subtract the ideal request:
\[
 r_{\rm del}-r_{\rm ideal}
 =k_H\{\widehat H_t+\epsilon_H-H
       +H_\star(t)-H_\star(\sigma_t)+v_\star\Delta\}+v_\star+\dot H_\star(t).
\]
The first difference lies in $[0,2\epsilon_H]$, the budget term in
$[0,2v_\star\Delta]$, and the last term in $[0,2v_\star]$. Summing their bounds proves
the claim.
\end{proof}

\subsection{Proof of Theorem~\ref{thm:conecapacity}}\label{proof:thm:conecapacity}
\begin{proof}
The first-order optimality condition for the metric projection is
$\ip{z_H-p_H}{u-p_H}_M\le0$ for every $u\in\mathcal K$.
Both $0$ and $2p_H$ belong to the cone. Substituting them gives opposite inequalities for $\ip{z_H-p_H}{p_H}_M$, so that inner product is zero. The optimality condition now implies $z_H-p_H\in\mathcal K^{\circ_M}$. For any feasible $u$,
\[
 -h^\top u=\ip{z_H}{u}_M
 =\ip{p_H}{u}_M+\ip{z_H-p_H}{u}_M
 \le\norm{p_H}_M\norm u_M\le B_{\rm int}\norm{p_H}_M.
\]
If $p_H\ne0$, the stated $u$ lies in the cone, has norm $B_{\rm int}$, and is orthogonal to $z_H-p_H$, so every inequality is an equality. If $p_H=0$, the upper bound is zero and $u=0$ attains it. The projection condition gives the polar-cone characterization in both directions.
\end{proof}

\subsection{Proof of Theorem~\ref{thm:viability}}\label{proof:thm:viability}
\begin{proof}
The chain rule for an absolutely continuous trajectory gives
$\dot e_j=k_j^\top\dot\theta+\partial_tR_j-\dot \tau_j$ almost everywhere. Multiplying $\dot e_j+\gamma_j e_j\le0$ by the positive integrating factor $e^{\gamma_j t}$ gives
$\frac{d}{dt}(e^{\gamma_j t}e_j(t))\le0$ almost everywhere. Integration from zero to $t$ proves \eqref{eq:riskcomparison}. With disturbance, the same calculation gives
\[
 e_j(t)\le e_j(0)e^{-\gamma_j t}
 +\int_0^t e^{-\gamma_j(t-s)}\omega_j(s)\,ds,
\]
which is bounded by \eqref{eq:riskdisturbance}. If $e_j(0)\le0$ and there is no disturbance, the bound remains nonpositive.
\end{proof}

\subsection{Proof of Proposition~\ref{prop:finitesafety}}\label{proof:prop:finitesafety}
\begin{proof}
Metric Cauchy--Schwarz gives
\[
 (\nabla R_j-\widehat k_j)^\top w
 \le\norm{\nabla R_j-\widehat k_j}_{M^{-1}}\norm w_M
 \le\delta_j\norm w_M.
\]
Combining this with the value error and \eqref{eq:metricsmooth} gives $R_j(\theta+s w)\le U_j(s)$. The coefficient of $s^2$ in $U_j$ is nonnegative, so convexity implies
$U_j(s)\le(1-s/\eta)U_j(0)+(s/\eta)U_j(\eta)\le \tau_j$ for $0\le s\le\eta$ and $\eta>0$. The zero-step case is immediate. Without the initial bound, only the endpoint conclusion follows from this argument.
\end{proof}

\subsection{Proof of Theorem~\ref{thm:audit}}\label{proof:thm:audit}
\begin{proof}
Conditional on $\mathcal H_{t-1}$, the candidate is fixed and each batch loss is independent, bounded in $[0,1]$, and has mean $\mathscr R_{j,P_{jt}}(\theta_t)$. Hoeffding's one-sided inequality gives
\[
 \Prob\!\left(\mathscr R_{j,P_{jt}}(\theta_t)
 >\widehat{\mathscr R}_{jt}+\epsilon_{jt}^{\rm stat}
 \mid\mathcal H_{t-1}\right)\le\alpha_{jt}.
\]
Taking expectations removes the conditioning. A countable union bound is valid without independence across audit times or groups, and
$\sum_{j=1}^{J_R}\sum_{t=0}^\infty\alpha_{jt}=\alpha$ because
$\sum_{t=0}^\infty(t+1)^{-2}=\pi^2/6$.

For any measurable $f\in[0,1]$, the layer-cake identity gives
\[
 |\E_P f-\E_Q f|
 =\left|\int_0^1[P(f>s)-Q(f>s)]\,ds\right|
 \le\operatorname{TV}(P,Q).
\]
Apply this to $f(Z)=\ell^{\rm harm}_j(\theta_t;Z)$ and add $\Delta_{jt}$ to every audit bound. The simultaneous event then contains every accepted subset of the audited candidates, even if acceptance is decided from these bounds.
\end{proof}

\subsection{Proof of Corollary~\ref{cor:composition}}\label{proof:cor:composition}
\begin{proof}
Apply the finite-step proposition to each accepted segment. Its conclusion is deterministic whenever its error and curvature assumptions hold. The audit theorem supplies one event of probability at least $1-\alpha$ containing all candidate-risk bounds; any subset selected by the threshold and authorization gates inherits those bounds. The authorization conclusion follows from the release rule itself. Finally, sum the descent inequality over accepted updates, indexed by $k$. Rejected updates leave the current parameters unchanged, so the loss differences telescope exactly as in Theorem~\ref{thm:discrete}. Divide the resulting bound on $\sum_k a_k$ by $N_{\rm acc}$. If there are no accepted updates, this division is unavailable and no task-progress conclusion is asserted.
\end{proof}

\subsection{Proof of Proposition~\ref{prop:latency}}\label{proof:prop:latency}
\begin{proof}
For $t\in[t_0,t_0+\Delta]$ before the response is effective,
$r_j(t)=r_j(t_0)+\int_{t_0}^t\dot r_j(s)ds\le-m_j+D_j(t-t_0)\le-m_j+D_j\Delta\le0$.
\end{proof}

\subsection{Proof of Proposition~\ref{prop:memorization}}\label{proof:prop:memorization}
\begin{proof}
Let $X$ be uniform on $[0,1]$, let $Y$ be an independent fair Bernoulli label, and let $S$ contain $n$ i.i.d. pairs. Inputs are distinct almost surely. Set $e_S(X_i)=A(2Y_i-1)$ and $e_S(x)=0$ otherwise. Predict $1$ when $e_S(x)+\xi\ge0$, with $\xi\sim N(0,\nu)$ independent. Conditional on $S$, a fresh $X$ hits no training point almost surely, so its noisy representation is independent Gaussian noise. Population error is $1/2$ because the fresh label is independent and balanced; expected error at each training point is $\Phi(-A/\sqrt\nu)$. This proves every claim.
\end{proof}

\subsection{Proof of Theorem~\ref{thm:gen}}\label{proof:thm:gen}
\begin{proof}
Fix a side-information value $v$ for which the conditional laws and covariances exist. Independence of $V$ and $S$ means that $S$ still has its original i.i.d. law given $V=v$. For a fixed predictor, averaging $n$ independent centered losses divides their sub-Gaussian variance proxy by $n$. Applying the relative-entropy variational inequality to the joint law of $(S,W)$ and the product of its conditional marginals therefore gives
\[
 \left|\E[R(W)-\widehat R_S(W)\mid V=v]\right|
 \le\sqrt{2\sigma_\ell^2 I(S;W\mid V=v)/n}.
\]
Appendix~\ref{app:math} derives this information inequality explicitly. Take expectations over $V$, write $\Delta_{\rm gen}=R(W)-\widehat R_S(W)$ and use $|\E\Delta_{\rm gen}|\le\E|\E[\Delta_{\rm gen}\mid V]|$, and apply concavity of the square root. This proves the first inequality of \eqref{eq:gen}.

The Markov condition states that, conditional on $(Q,V)$, $W$ contains no further information about $S$. Conditional data processing thus gives $I(S;W\mid V)\le I(S;Q\mid V)$. Write $\mathrm h_{\rm diff}$ for differential entropy, to distinguish it from the encoder and gradient. Since $\phi_v(S)$ is deterministic given $S$, translation invariance of entropy and independence of the noise give
\[
 \mathrm h_{\rm diff}(Q\mid S,V=v)=\mathrm h_{\rm diff}(\zeta)
 =\tfrac12\log\det(2\pi e\nu I_d).
\]
Also $\Cov(Q\mid V=v)=C_v+\nu I_d$. Gaussian maximum entropy with this covariance gives
\[
 \mathrm h_{\rm diff}(Q\mid V=v)
 \le\tfrac12\log\det(2\pi e[C_v+\nu I_d]).
\]
Subtracting the two entropies yields
$I(S;Q\mid V=v)\le\tfrac12\log\det(I_d+C_v/\nu)$.
Averaging this inequality, substituting it into the first bound, and cancelling the factor two proves the result. If the expected log-determinant is infinite, the bound is interpreted as vacuous rather than as a finite certificate.
\end{proof}

\subsection{Proof of Corollary~\ref{cor:sequential}}\label{proof:cor:sequential}
\begin{proof}
Write $Q_{\le T}=(Q_1,\ldots,Q_T)$. The assumed absence of additional sample-dependent information gives
\[
 I(S;W\mid V)\le I(S;Q_{\le T}\mid V)
 =\sum_{t=1}^T I(S;Q_t\mid Q_{<t},V).
\]
Conditional on the past releases and $V$, the new signal is a deterministic function of $S$, and the new Gaussian noise is independent of that conditional signal. Its conditional covariance is $\nu_tI_{d_t}$. The same entropy subtraction as in Theorem~\ref{thm:gen}, now conditioning on $(Q_{<t},V)$, gives
\[
 I(S;Q_t\mid Q_{<t},V)
 \le\tfrac12\E\log\det(I_{d_t}+C_t/\nu_t).
\]
Summing and using the original conditional generalization inequality proves \eqref{eq:sequential}. The sample need not remain i.i.d. after conditioning on the releases: the Gaussian-channel step requires no such assumption. The i.i.d. assumption is used only for the generalization inequality conditional on independent $V$.
\end{proof}

\section{Supporting theory and statistical assumptions}
\subsection{Supporting representation geometry}\label{app:geometrysupport}
\begin{proposition}[Sensitivity to covariance approximation]\label{prop:cov}
For positive semidefinite $C,\widetilde C$ with $\norm{C-\widetilde C}_{\rm op}\le\delta$, the functional $\mathcal H_\nu(C)=(2p)^{-1}\log\det(I+C/\nu)$ satisfies
\begin{equation}
 |\mathcal H_\nu(C)-\mathcal H_\nu(\widetilde C)|\le\frac{\delta}{2\nu}.
\end{equation}
If $C=C_Z$ and the inverse in \eqref{eq:zgradient} is replaced by $(\widetilde C+\nu I)^{-1}$, the resulting approximate force $\widetilde G_Z$ obeys
\begin{equation}\label{eq:forceerror}
 \norm{\widetilde G_Z-\nabla_ZH_\nu}_F
 \le\frac{\norm{Z_c}_F\delta}{p(m-1)\nu^2}.
\end{equation}
The associated parameter-force error is at most the right-hand side times $\norm{J_Z}_{\rm op}$.
\end{proposition}
\noindent\emph{Proof.} See Appendix~\ref{proof:prop:cov}.

The approximate force in this proposition is an approximation to \eqref{eq:zgradient}; it need not be the exact gradient of an arbitrary covariance estimator. A smaller $\nu$ increases spectral resolution but worsens sensitivity to covariance error. This is a quantitative reason to calibrate the noise level rather than treating it only as a numerical jitter.

\begin{proposition}[Readout regularity prevents arbitrary volume collapse]\label{prop:readout}
Let $Z\in\R^p$ have finite covariance, and let the readout $r$ be $L_r$-Lipschitz, with $L_r>0$. If $\tr\Cov(r(Z))\ge v>0$, then
\begin{equation}\label{eq:readoutfloor}
 \mathcal H_\nu(\Cov Z)\ge\frac1{2p}\log\left(1+\frac{v}{L_r^2\nu}\right).
\end{equation}
\end{proposition}
\noindent\emph{Proof.} See Appendix~\ref{proof:prop:readout}.

This is a necessary geometric condition, not a prediction guarantee. It identifies a missing assumption in arguments that treat small latent volume as an intrinsic property of a predictor: the readout must not amplify the shrinking representation without limit.

\subsubsection{How much can a task-preserving intervention accomplish?}
Scalar feedback restricts the intervention direction to $-M^{-1}h$. Allowing a general direction changes the attainable rates, but does not remove all conflicts.

\begin{theorem}[Sharp limit for bounded, task-preserving intervention]\label{thm:budget}
Let $a>0$, $B_{\rm int}\ge0$, and consider $\dot\theta=-M^{-1}g+u$ with $g^\top u\le0$ and $\norm{u}_M\le B_{\rm int}$. Define
\begin{equation}\label{eq:effective-gradient}
 h_{\rm eff}=\begin{cases}h,&b\ge0,\\ h-(b/a)g,&b<0,\end{cases}
 \qquad c_{\rm eff}=h_{\rm eff}^\top M^{-1}h_{\rm eff}
 =c-\frac{\min(b,0)^2}{a}.
\end{equation}
The maximum additional entropy-decrease rate is
\begin{equation}\label{eq:budget}
 \max_{g^\top u\le0,\,\norm{u}_M\le B_{\rm int}}(-h^\top u)=B_{\rm int}\sqrt{c_{\rm eff}}.
\end{equation}
When $c_{\rm eff}>0$, the maximizer is $u=-B_{\rm int}M^{-1}h_{\rm eff}/\sqrt{c_{\rm eff}}$.
\end{theorem}
\noindent\emph{Proof.} See Appendix~\ref{proof:thm:budget}.

If interventions are further restricted to the observable range under Theorem~\ref{thm:obscontrol}, the same argument gives the limit $B_{\rm int}\sqrt{c_{\rm obs}-\min(b,0)^2/a}$. Thus prediction-null directions can increase the apparent entropy-control budget. The formula for $h_{\rm eff}$ is a gradient-conflict projection, closely related to PCGrad \citep{yu2020pcgrad}. The result supplies its task-preserving entropy-rate limit. If $b<0$ and $b^2=ac$, then $c_{\rm eff}=0$: even an arbitrarily oriented finite intervention cannot produce additional entropy decrease while preserving the full baseline task-descent rate. For $a=0$, the first-order task constraint is vacuous, but curvature controls whether an actual finite step increases $L$.

\subsection{Factorized observable control}
\subsubsection{How predictive observability changes feasibility}\label{sec:obscontrol}
A feasible entropy request can be met partly through directions that do not change the chosen predictions. The decomposition in Section~\ref{sec:observable} allows this possibility to be quantified, rather than treated as an unexplained reduction in entropy.

\begin{theorem}[Feasibility restricted to observable intervention]\label{thm:obscontrol}
Suppose the protected loss factors through the probe map, $L(\theta)=\ell(F(\theta))$. Let $\Pi_{\rm obs}$ be the projection in \eqref{eq:projection}, set $u_{\rm obs}=\Pi_{\rm obs}M^{-1}h$, and define $c_{\rm obs}=\norm{u_{\rm obs}}_M^2$. For the restricted flow
\begin{equation}\label{eq:obsflow}
 \dot\theta=-M^{-1}g-\beta u_{\rm obs},
\end{equation}
the exact rates are
\begin{equation}\label{eq:obsrates}
 \dot L=-a-\beta b,\qquad \dot H=-b-\beta c_{\rm obs}.
\end{equation}
Consequently, Theorem~\ref{thm:feasible} holds with $c$ replaced by $c_{\rm obs}$ and the same upper coefficient cap. The restricted feasible set is a subset of the unrestricted one. For $b<0$, $r_H=0$, and no cap, the sharp condition becomes
\begin{equation}\label{eq:obsangle}
 b^2\le(1-\kappa)a c_{\rm obs}.
\end{equation}
If $c_{\rm obs}=0$, then $b=0$, and a positive entropy-decrease request is impossible under this restricted flow.
\end{theorem}
\noindent\emph{Proof.} See Appendix~\ref{proof:thm:obscontrol}.

For example, take $M=I_2$, $J_F=(1,0)$, $g=(1,0)^\top$, and $h=(-1,2)^\top$ at a state. Then $(a,b,c,c_{\rm obs})=(1,-1,5,1)$. With $\kappa=1/2$, $r_H=0$, and $\overline\beta=1$, unrestricted coefficients in $[1/5,1/2]$ satisfy both requirements, while no restricted coefficient does: its lower bound is $1$ and its upper bound is $1/2$. The difference comes entirely from the prediction-null component. It can be relevant for removing nuisance variation, but cannot be evidence of a comparable change in predictive information.

The factorization is a sufficient condition for these shared crossproducts. The observable bridge uses the full training-logit map. Separate small probes instead require Proposition~\ref{prop:generalobsrates}; its exact rates are tested by the two-probe diagnostic.

\subsection{General observation maps and exact rates}
\label{sec:generalobsrates}
The observable projection can be formed from a smaller probe map than the
predictions defining the training loss. In that case the task--information
crossproduct must be recomputed after projection. Let
$\Pi_{\rm obs}$ be the $M$-orthogonal projector onto
$\operatorname{range}(M^{-1}J_F^\top)$, and retain
$a=g^\top M^{-1}g$, $b=g^\top M^{-1}h$, and $c=h^\top M^{-1}h$.
Define
\[
 u_{\rm obs}=\Pi_{\rm obs}M^{-1}h,\qquad
 b_{\rm obs}=g^\top u_{\rm obs},\qquad
 a_{\rm obs}=\|\Pi_{\rm obs}M^{-1}g\|_M^2,\qquad
 c_{\rm obs}=\|u_{\rm obs}\|_M^2.
\]
Here $u_{\rm obs}$ is the projected metric gradient lift; the actual entropy
velocity in the controlled flow is $-\beta u_{\rm obs}$.

\begin{proposition}[Exact rates for a general probe map]
\label{prop:generalobsrates}
For the velocity $w=-M^{-1}g-\beta u_{\rm obs}$, without any factorization
assumption on $L$, the exact local rates are
\begin{equation}\label{eq:generalobsrates}
 \dot L=-a-\beta b_{\rm obs},\qquad
 \dot H=-b-\beta c_{\rm obs}.
\end{equation}
If $c_{\rm obs}>0$, the coefficients in $[0,\overline\beta]$ satisfying
$\dot L\le-\kappa a$ and $\dot H\le-r_H$ are precisely the interval with endpoints
\begin{align}\label{eq:generalobsinterval}
 \beta_{\rm lo}^{\rm obs}
 &=\max\{0,(r_H-b)/c_{\rm obs}\},\\
 \beta_{\rm hi}^{\rm obs}
 &=\begin{cases}
 \overline\beta,&b_{\rm obs}\ge0,\\
 \min\{\overline\beta,(1-\kappa)a/(-b_{\rm obs})\},&b_{\rm obs}<0,
 \end{cases}
\end{align}
provided the lower endpoint does not exceed the upper endpoint.
If $c_{\rm obs}=0$, the entropy intervention vanishes, $b_{\rm obs}=0$, and
feasibility is equivalent to $r_H\le b$.
Moreover,
\begin{equation}\label{eq:probemismatch}
 |b-b_{\rm obs}|\le
 \sqrt{(a-a_{\rm obs})(c-c_{\rm obs})}.
\end{equation}
If $L=\ell\circ F$, then $a_{\rm obs}=a$ and $b_{\rm obs}=b$, recovering the
factorized observable result. Without this property, the projected feasible
interval need not be a subset of its unrestricted counterpart.
\end{proposition}
\begin{proof}
The chain rule gives the first rate in \eqref{eq:generalobsrates} immediately.
Since $\Pi_{\rm obs}$ is self-adjoint in the $M$ metric,
\[
 h^\top u_{\rm obs}
 =\langle M^{-1}h,\Pi_{\rm obs}M^{-1}h\rangle_M
 =\|\Pi_{\rm obs}M^{-1}h\|_M^2=c_{\rm obs},
\]
which gives the second rate. The task inequality is equivalent to
$\beta b_{\rm obs}\ge-(1-\kappa)a$; the entropy inequality is equivalent to
$\beta c_{\rm obs}\ge r_H-b$. Solving these two scalar inequalities and
intersecting with $[0,\overline\beta]$ proves
\eqref{eq:generalobsinterval}. When $c_{\rm obs}=0$, the projected vector is
zero, so only the task velocity remains and the entropy rate is $-b$.

Let $v_g=M^{-1}g$ and $v_h=M^{-1}h$. Orthogonality of the projected and residual
components gives
\[
 b-b_{\rm obs}
 =\langle(I-\Pi_{\rm obs})v_g,(I-\Pi_{\rm obs})v_h\rangle_M.
\]
Their squared norms are $a-a_{\rm obs}$ and $c-c_{\rm obs}$. Metric
Cauchy--Schwarz proves \eqref{eq:probemismatch}. Finally $L=\ell\circ F$ implies
$g=J_F^\top\nabla\ell(F)$, so $M^{-1}g$ lies in the projector range and its
residual is zero.
\end{proof}

For an explicit failure of interval inclusion, take $M=I_2$,
$\Pi_{\rm obs}=\operatorname{diag}(1,0)$, $g=(1,1)^\top$, $h=(1,-2)^\top$,
$\kappa=1/2$, $r_H=0$ and $\overline\beta=2$. Then
$(a,b,c,b_{\rm obs},c_{\rm obs})=(2,-1,5,1,1)$.
The unrestricted feasible interval is $[1/5,1]$, whereas the projected
interval is $[1,2]$. At $\beta=2$ the projected velocity satisfies both
requirements and the unrestricted one violates task descent. Thus a small
separate probe set cannot inherit the factorized-loss calculation merely by
using the same symbol $b$.

\subsection{Fresh-probe population certificates}
\label{sec:uniformprobe}
The observable decomposition is local to a declared probe map. A direction chosen to be invisible on those probes need not be invisible on the operating population. A fresh probe batch permits a uniform statement when the intervention class is restricted before that batch is observed. The following result applies standard covering and concentration arguments \citep{hoeffding1963,vershynin2018}; its purpose is to specify when a probe-adaptive HCLM intervention has a population interpretation.

Condition on a history $\mathcal A$ containing the fitted state $\theta$, the positive definite metric $M$, and an orthonormal matrix $U\in\R^{D\times k_U}$, with $1\le k_U\le D$ and $U^\top U=I_{k_U}$. All these objects must be selected before the fresh probes below. Define the metric-orthonormal lift $T=M^{-1/2}U$, so $T^\top MT=I_{k_U}$ and $\norm{Tz_U}_M=\norm{z_U}$. Thus $\operatorname{range}(T)$ is the declared intervention subspace. It can, for example, be the span of task, information, and protected-risk directions computed from prior training data. Let $f_\theta(x)\in\R^{d_y}$ be the declared prediction map, in fixed output units, and define
\begin{equation}\label{eq:probejacobian}
 A(x)=D_\theta f_\theta(x)T,\qquad
 C_F=\E_{X\sim P}[A(X)^\top A(X)\mid\mathcal A],\qquad
 \widehat C_F=\frac1m\sum_{i=1}^m A(X_i)^\top A(X_i).
\end{equation}
Conditional on $\mathcal A$, the $m\ge1$ probes are i.i.d. from a specified law $P$. Assume the population bound $\norm{A(X)}_{\rm op}\le K$ holds almost surely for a finite, history-measurable $K$. An observed maximum on the probe sample alone does not establish this assumption.

\begin{proposition}[Uniform prediction-action certificate]\label{prop:uniformprobe}
For $0<\delta<1$, set
\begin{equation}\label{eq:probeepsilon}
 \epsilon_m=K^2\min\left\{1,
 \sqrt{\frac{2[k_U\log9+\log(2/\delta)]}{m}}\right\}.
\end{equation}
With conditional probability at least $1-\delta$,
\begin{equation}\label{eq:probeuniform}
 \norm{\widehat C_F-C_F}_{\rm op}\le\epsilon_m,
 \qquad
 \left|\E_P\norm{A(X)z_U}^2-z_U^\top\widehat C_F z_U\right|
 \le\epsilon_m\norm{z_U}^2\quad\text{for every }z_U\in\R^{k_U}.
\end{equation}
In particular, $z_U$ may be selected using the entire fresh probe batch. On this event, any empirical prediction-null direction satisfies
$\E_P\norm{A(X)z_U}^2\le\epsilon_m\norm{z_U}^2$.
If $\lambda_{\min}(\widehat C_F)>\epsilon_m$, the population Gram matrix is positive definite, with $\lambda_{\min}(C_F)\ge\lambda_{\min}(\widehat C_F)-\epsilon_m$. Conversely, if $C_F\succeq\mu I_{k_U}$ and $\epsilon_m<\mu$, the empirical Gram matrix has no nonzero null direction.
\end{proposition}
\noindent\emph{Proof.} See Appendix~\ref{proof:prop:uniformprobe}.

The result does not certify zero population action from zero empirical action. In a wide parameter space the bound can be vacuous: its untruncated scale is proportional to $K^2\sqrt{k_U/m}$. Restricting to a small, predeclared intervention subspace can make certification less costly, but may exclude a useful correction. Choosing the subspace or fitted model after seeing the same probes invalidates the proof unless a separate uniform bound covers that additional selection. Reusing old probes after model adaptation has the same problem.

The following construction explains why the fitted model must be selected before its probes.
Let $X$ have a continuous law. After observing probes $x_1,\ldots,x_m$, select the one-parameter family $f_\vartheta(x)=\vartheta\mathbf1\{x\notin\{x_1,\ldots,x_m\}\}$. It is smooth in $\vartheta$, with parameter derivative bounded by one. Every observed derivative is zero, while the derivative at an independent fresh input is one almost surely. Thus $\widehat C_F=0$ and $C_F=1$ for every $m$. This construction violates the pre-probe model requirement and defeats a deviation bound that shrinks with $m$, even in one parameter.

\begin{corollary}[A finite update and a probe-adaptive risk certificate]\label{cor:proberisk}
In addition to Proposition~\ref{prop:uniformprobe}, let $\varphi_x(z)=f_{\theta+Tz}(x)$. Fix a radius $b_U>0$ such that $\theta+Tz\in\Theta$ whenever $\|z\|\le b_U$. Suppose $\varphi_x$ is twice continuously differentiable on $\{z:\norm z\le b_U\}$ for almost every $x$, with the population-uniform bilinear Hessian bound
\begin{equation}\label{eq:probehessian}
 \norm{D_z^2\varphi_x(z)[v_1,v_2]}
 \le L_A\norm{v_1}\norm{v_2}
 \quad\text{whenever }\norm z\le b_U.
\end{equation}
For any probe-adaptively selected $z_U$ and $\eta\ge0$ with $\eta\norm{z_U}\le b_U$, on the event in \eqref{eq:probeuniform},
\begin{equation}\label{eq:probefinite}
 \left(\E_P\norm{f_{\theta+\eta Tz_U}(X)-f_\theta(X)}^2\right)^{1/2}
 \le \eta\sqrt{z_U^\top\widehat C_F z_U+\epsilon_m\norm{z_U}^2}
       +\frac{L_A\eta^2}{2}\norm{z_U}^2.
\end{equation}
For the risk claim, fix a joint law $P_{XY}$ with input marginal $P$ before drawing the probes, and assume that the same probe inputs and their labels form pairs $(X_i,Y_i)$, $i=1,\ldots,m$, sampled conditionally i.i.d. from $P_{XY}$ given $\mathcal A$. The law may be measurable with respect to $\mathcal A$, but cannot be selected using the current probes. Suppose a specified harm loss $\ell_{\rm harm}(y,s)\in[0,1]$ is $L_{\rm harm}$-Lipschitz in the prediction $s$. Put $R_{P_{XY}}(\vartheta)=\E_{P_{XY}}\ell_{\rm harm}(Y,f_\vartheta(X))$. On the same probes, let
$\widehat R_{P_{XY}}(\theta)=m^{-1}\sum_i\ell_{\rm harm}(Y_i,f_\theta(X_i))$.
For $0<\delta_R<1$, define
\[
 u_R=\min\left\{1,\widehat R_{P_{XY}}(\theta)+
       \sqrt{\frac{\log(1/\delta_R)}{2m}}\right\}.
\]
With conditional probability at least $1-\delta-\delta_R$, simultaneously for every such selected $z_U$,
\begin{align}\label{eq:probeprotected}
 R_{P_{XY}}(\theta+\eta Tz_U)
 \le u_R+L_{\rm harm}\left\{
 \eta\sqrt{z_U^\top\widehat C_F z_U+\epsilon_m\norm{z_U}^2}
 +\frac{L_A\eta^2}{2}\norm{z_U}^2\right\}.
\end{align}
Consequently, accepting a candidate only when the displayed upper bound is at most a declared threshold $d$ certifies that candidate's population risk, although its direction was selected using these probes.
\end{corollary}
\noindent\emph{Proof.} See Appendix~\ref{proof:cor:proberisk}.

For several protected groups, apply the result under each specified conditional group law using its own fresh probes, constants, and allocated failure probability. Summing these probabilities gives a simultaneous certificate. An aggregate prediction-action bound alone does not protect a rare group. The Lipschitz hypothesis excludes zero--one harm loss and thresholded decisions unless an additional, valid relation to a smooth upper bound is established. Moreover, $K,L_A,L_{\rm harm}$, the output units, the intervention radius, and distributional assumptions must be justified; substituting empirical maxima for population bounds would change the guarantee.

The right side of \eqref{eq:probeprotected} is convex in $z_U$ for fixed nonnegative $\eta$: the square-root term is the norm induced by the positive semidefinite matrix $\widehat C_F+\epsilon_m I_{k_U}$, and the final term is a nonnegative quadratic. It can therefore be used as an additional candidate-acceptance constraint in a reduced-space HCLM controller. It protects the declared Lipschitz population risk, rather than identifying entropy reduction with risk reduction. The estimate can be too loose to accept any nonzero update; that is an explicit limitation, not a failed test to be removed.

These certificates extend to an adaptive trajectory when each step uses a new batch under the required conditional sampling law.
At iteration $t$, choose $\theta_t,M_t,U_t,P_t,K_t,L_{A,t}$ and the probe count using the history before drawing a fresh batch. Allocate
$\delta_t=6\delta_{\rm all}/[\pi^2(t+1)^2]$ for each chosen overall failure budget, or split this amount across groups and the Gram/risk events. Conditional versions of the propositions followed by a countable union bound give simultaneous validity over all audited steps. The direction at step $t$ may use its current fresh probes, but the next step needs new probes unless a separate result justifies reuse. The certificate is for the declared law at each step; distribution shift requires an additional justified discrepancy bound.

\subsection{Delayed-specification guarantees}
\label{sec:delayedcertification}
Human input is an operational part of HCLM when it changes a declared entropy
budget, protected threshold, or intervention budget and thereby changes the set
of admissible updates. A delayed packet does not reveal an unconstrained current
human preference. Any claim about that preference therefore needs a bound on
its change. Alternatively, a specification can remain in force until an authorized
replacement is received and checked. These are distinct contracts.

We first quantify the cost of information delay for the row certificate. Freeze
$M$, $\mathcal U$, the row scales, and the intervention budget $B_{\rm int}$ during one
certification decision. Write $\bar d_i,\bar t_i$ for the true current normalized
rows and $\widehat d_i,\widehat t_i$ for the available estimates. On an event
$\mathcal E$, suppose the bounds in \eqref{eq:rowuncertainty} hold simultaneously.
Let $\varepsilon_\star$ and $\varepsilon_{\rm rob}$ be the true and robust optimal
residuals in \eqref{eq:minviolation} and \eqref{eq:robustprimal}, and define
\begin{equation}\label{eq:controlleruncertaintycost}
 \chi=\max_{1\le i\le N}\{B_{\rm int}\delta_i^d+\delta_i^t\}.
\end{equation}
Here $\chi$ is a uniform bound on the uncertainty of any normalized directional
requirement over the allowed intervention ball.

\begin{theorem}[Certified acceptance and a bound on unnecessary rejection]
\label{thm:controlleracceptance}
On $\mathcal E$, the optimal residuals obey
\begin{equation}\label{eq:controllerresidualsandwich}
 \varepsilon_\star\ \le\ \varepsilon_{\rm rob}
 \ \le\ \varepsilon_\star+2\chi.
\end{equation}
Every intervention accepted by the robust row test satisfies every true current
row. Moreover, if $\varepsilon_\star\le-2\chi$, the robust problem is feasible:
a request with that much true margin cannot be rejected as infeasible by an
exact implementation of the robust test. With an optimization tolerance
$\epsilon_{\rm opt}\ge0$, a returned intervention whose reported largest robust
row residual is at most $-\epsilon_{\rm opt}$ remains certified when the
unreported residual error is at most $\epsilon_{\rm opt}$.
\end{theorem}
\noindent\emph{Proof.} See Appendix~\ref{proof:thm:controlleracceptance}.

The factor two cannot be improved for the stated information model. In one
Euclidean dimension, let $\mathcal U=\mathbb R$, take a zero row normal, and set
$\bar t=2\delta$, $\widehat t=\delta$, $\delta^d=0$,
$\delta^t=\delta>0$. Then
$\varepsilon_\star=-2\delta$ and
$\varepsilon_{\rm rob}=0$, with $\chi=\delta$. Both the normal and the threshold
error are allowed by the model. Thus the conservatism is quantitatively
bounded but need not vanish without better information.

The radii can include a directly logged delay term. Let $s$ denote the complete
numerical specification, including any threshold rates used in the rows. Suppose
on a declared region and with the frozen metric,
\begin{align}
 \norm{\bar d_i(\theta,s)-\bar d_i(\theta',s')}_{M^{-1}}
 &\le L^d_{i\theta}\norm{\theta-\theta'}_M
       +L^d_{is}\norm{s-s'},\\
 |\bar t_i(\theta,s)-\bar t_i(\theta',s')|
 &\le L^t_{i\theta}\norm{\theta-\theta'}_M
       +L^t_{is}\norm{s-s'}.
\end{align}
The specification norm and all four constants are fixed before certification.
For rows evaluated at a stale state and specification, let $D_{\rm mov}$ bound the
accumulated parameter movement and let $D_{\rm spec}$ bound the specification change
since that evaluation. If the evaluation errors are bounded by
$\xi_i^d,\xi_i^t$, valid current radii are
\begin{equation}\label{eq:delayradii}
 \delta_i^d=\xi_i^d+L^d_{i\theta}D_{\rm mov}+L^d_{is}D_{\rm spec},
 \qquad
 \delta_i^t=\xi_i^t+L^t_{i\theta}D_{\rm mov}+L^t_{is}D_{\rm spec}.
\end{equation}
Indeed, insert the stale exact row between the current row and its estimate and
apply the triangle inequality. In discrete training, let $k$ be the current update index and $\tau\le k$ a nonnegative integer packet age. One may take
$D_{\rm mov}=\sum_{\ell=k-\tau}^{k-1}\norm{\theta_{\ell+1}-\theta_\ell}_M$;
if $V_s$ bounds specification movement per update, then $D_{\rm spec}\le V_s\tau$. These are certified radii only where the Lipschitz constants,
evaluation errors, and drift bound are justified. An arbitrary threshold jump,
a change of observable subspace, or a change of metric requires a new check.

For a human-authorized entropy budget, delay has an explicit tracking cost.
Let the desired upper budget be $H_\star(t)$, assume $H_\star$ is absolutely continuous with
$|\dot H_\star(t)|\le v_\star$ almost everywhere, and let the most recent packet contain
$H_\star(\sigma_t)$ with $0\le t-\sigma_t\le\Delta$. At the current state, suppose
$|\widehat H_t-H(\theta_t)|\le\epsilon_H$. Fix $k_H>0$ and request
\begin{equation}\label{eq:delayedentropyrequest}
 r_{\rm del}(t)=k_H\{
 \widehat H_t+\epsilon_H-H_\star(\sigma_t)+v_\star\Delta\}+v_\star.
\end{equation}
The correction accounts for the smallest possible current budget and its
largest possible downward rate. The resulting rate margin
$k_Hv_\star\Delta+v_\star$ has entropy-per-time units.

\begin{proposition}[One-sided entropy tracking with a delayed budget]
\label{prop:delayedentropytracking}
Suppose an absolutely continuous parameter trajectory satisfies
$\dot H(\theta_t)\le-r_{\rm del}(t)$ almost everywhere on an interval
$[t_0,T]$, with the stated packet, estimation, and budget-drift bounds.
Then, for $e_H(t)=H(\theta_t)-H_\star(t)$,
\begin{equation}\label{eq:delayedentropytracking}
 e_H(t)\le e^{-k_H(t-t_0)}e_H(t_0),\qquad t\in[t_0,T].
\end{equation}
The delayed request exceeds the ideal instantaneous request
$r_{\rm ideal}=k_H(H-H_\star)-\dot H_\star$ by an amount in
$[0,2k_H\epsilon_H+2k_Hv_\star\Delta+2v_\star]$ almost everywhere.
\end{proposition}
\noindent\emph{Proof.} See Appendix~\ref{proof:prop:delayedentropytracking}.

The entropy row in \eqref{eq:rows2} can use $r_H=r_{\rm del}$; the robust test
then decides whether the delayed demand is compatible with task progress,
protected risks, the observable directions, and the intervention budget.
The tracking conclusion applies only while that complete demand is feasible
and enforced. A fallback that relaxes the entropy request retains only its
other certificates. In particular, neither a larger entropy gradient nor a
larger scalar coefficient can guarantee tracking when the protected feasible
set is empty. The proof does not establish convergence of the parameter vector
or of an unconstrained integral thermostat, and it is not a human-participant
result. Its technical tools are robust optimization and a scalar comparison
inequality; its role here is to quantify the interaction between authorized
specifications, delayed information, and admissible representation updates.

\subsection{Additional protection and audit results}\label{app:protectionsupport}
\subsubsection{The entropy reduction that remains after protection}
A simpler closed form is available when additional intervention must not increase any protected objective relative to its baseline rate. Let $A$ have row normals $k_j^\top$ and include $g^\top$ if the full task-only descent rate is to be preserved. Define the closed convex cone
\begin{equation}\label{eq:protectedcone}
 \mathcal K=\{u\in\mathcal U:Au\le0\},\qquad
 z_H=-M^{-1}h,\qquad p_H=\Pi_{\mathcal K}^M(z_H),
\end{equation}
where $\Pi_{\mathcal K}^M(z_H)$ is the unique point in $\mathcal K$ minimizing $\norm{z_H-u}_M^2$. Unlike the affine requirements above, these homogeneous constraints preserve baseline rates; they certify safety only when the baseline already satisfies the required risk rates.

\begin{theorem}[Entropy capacity of the protected intervention cone]\label{thm:conecapacity}
For every $B_{\rm int}\ge0$,
\begin{equation}\label{eq:conecapacity}
 \max_{u\in\mathcal K,\,\norm u_M\le B_{\rm int}}(-h^\top u)=B_{\rm int}\norm{p_H}_M.
\end{equation}
When $p_H\ne0$, $u=B_{\rm int}p_H/\norm{p_H}_M$ attains this value. When $p_H=0$, no allowed intervention can produce positive additional entropy decrease, whatever its finite budget. In that case $z_H$ belongs to the metric polar cone
$\mathcal K^{\circ_M}=\{z:\ip z u_M\le0\text{ for all }u\in\mathcal K\}$.
\end{theorem}
\noindent\emph{Proof.} See Appendix~\ref{proof:thm:conecapacity}.

This is a metric cone-projection identity specialized to an entropy intervention. For one task half-space, it reduces to Theorem~\ref{thm:budget}; it is not a competing invention of gradient projection. It specifies the entropy-reduction capacity left after all declared protections are enforced. A large unrestricted force can have zero protected capacity. That is a precise distinction between force activity and admissible action.

\subsubsection{What the derivative constraints actually preserve}
\begin{theorem}[Conditional preservation of moving risk thresholds]\label{thm:viability}
Suppose $\theta(t)$ is an absolutely continuous solution on $[0,T]$, $R_j$ and $\tau_j$ are continuously differentiable, and \eqref{eq:barrier} holds almost everywhere. Then
\begin{equation}\label{eq:riskcomparison}
 e_j(t)\le e_j(0)e^{-\gamma_j t},\qquad 0\le t\le T.
\end{equation}
In particular, $\theta(0)\in\mathcal S(0)$ implies $\theta(t)\in\mathcal S(t)$ throughout that interval. If the right-hand side of \eqref{eq:barrier} has an additional disturbance $\omega_j(t)\le\bar\omega_j$, then
\begin{equation}\label{eq:riskdisturbance}
 e_j(t)\le e_j(0)e^{-\gamma_j t}
 +\frac{\bar\omega_j}{\gamma_j}(1-e^{-\gamma_j t}).
\end{equation}
\end{theorem}
\noindent\emph{Proof.} See Appendix~\ref{proof:thm:viability}.

The assumptions do not establish existence of a global feasible trajectory. Theorem~\ref{thm:multirisk} must remain feasible along the path, and a selected vector field must admit a solution. Constant-rank smooth observable projections and a locally regular controller are useful sufficient ingredients, but rank changes or discontinuous policies require separate analysis. An initially violated threshold is approached from above under \eqref{eq:riskcomparison}; finite-time recovery is not asserted. The disturbed result permits an excess of size $\bar\omega_j/\gamma_j$ and must not be described as exact safety.

\subsubsection{Finite-step certificates with estimation error}
For implementation, a local derivative inequality alone is insufficient. At a fixed iteration, consider a time-independent risk $R_j$ and threshold $\tau_j$. Let $\widehat R_j$ and $\widehat k_j$ be estimates at the current state, with valid bounds
\begin{equation}\label{eq:riskerrors}
 R_j(\theta)\le\widehat R_j+\epsilon_j,\qquad
 \norm{\nabla R_j(\theta)-\widehat k_j}_{M^{-1}}\le\delta_j.
\end{equation}
The metric is held fixed along the candidate segment. Assume the segment satisfies the smoothness bound
\begin{equation}\label{eq:metricsmooth}
 R_j(\theta+s w)\le R_j(\theta)+s\nabla R_j(\theta)^\top w
 +\tfrac12\ell_j s^2\norm w_M^2,\qquad 0\le s\le\eta,
\end{equation}
for $\ell_j\ge0$; a Hessian bound in the $M$ metric is one sufficient condition. Define
\begin{equation}\label{eq:finitecertificate}
 U_j(s)=\widehat R_j+\epsilon_j
 +s(\widehat k_j^\top w+\delta_j\norm w_M)
 +\tfrac12\ell_j s^2\norm w_M^2.
\end{equation}
\begin{proposition}[Accepting a finite update without crossing a specified threshold]\label{prop:finitesafety}
If \eqref{eq:riskerrors}--\eqref{eq:metricsmooth} hold, $U_j(0)\le \tau_j$, and $U_j(\eta)\le \tau_j$, then $R_j(\theta+s w)\le \tau_j$ for every $s\in[0,\eta]$. If only $U_j(\eta)\le \tau_j$ is checked, the endpoint is certified, but the statement for the whole segment does not follow from these hypotheses. The result holds simultaneously for every risk whose conditions are satisfied.
\end{proposition}
\noindent\emph{Proof.} See Appendix~\ref{proof:prop:finitesafety}.

Even an exactly tangential direction can fail a finite-step test at a curved boundary. For example, for $R(\theta)=\norm\theta^2$, $\theta=(1,0)$, and $w=(0,1)$, one has $\nabla R^\top w=0$, but $R(\theta+\eta w)=1+\eta^2>1$ for every $\eta>0$. Backtracking cannot repair this direction except by selecting zero or changing the direction. The finite certificate makes that obstruction explicit.

When the protected quantity changes with time, an independently justified upper bound on its temporal increase must be added to $U_j$, and the threshold must be evaluated at the corresponding time. An empirical gradient-error estimate is not automatically a uniform certificate. The next result provides a separate, elementary audit for population risk at candidate checkpoints.

\subsubsection{Auditing candidate models under adaptation and distribution shift}
Let $\ell^{\rm harm}_j(\theta;Z)\in[0,1]$ be a specified harm loss, and let
$\mathscr R_{j,P}(\theta)=\E_{Z\sim P}\ell^{\rm harm}_j(\theta;Z)$.
The index $j$ can denote a subgroup or an operating condition. A harm loss may be discontinuous; the audit below does not require its differentiability. For use in gradient control, a smooth risk must either equal it or be related to it by a separate valid bound.

Let $J_R\ge1$ be the number of audited risks and $\mathcal H_{-1}$ the initial information. At audit time $t=0,1,\ldots$, choose a candidate $\theta_t$ using the history $\mathcal H_{t-1}$, before observing a fresh batch $Z_{jt1},\ldots,Z_{jtm_{jt}}$. Conditional on that history, the batch is i.i.d. from a declared audit law $P_{jt}$ and independent of the candidate. Both $P_{jt}$ and the positive integer $m_{jt}$ may be specified by the history. Write $\widehat{\mathscr R}_{jt}$ for its mean loss. Let $Q_{jt}$ be the deployment law, with a justified total-variation bound
$\operatorname{TV}(P_{jt},Q_{jt})\le\Delta_{jt}$, using the convention
$\operatorname{TV}(P,Q)=\sup_A|P(A)-Q(A)|$.

\begin{theorem}[Simultaneous checkpoint certificates with fresh audits]\label{thm:audit}
For a chosen failure probability $\alpha\in(0,1)$, set
\begin{equation}\label{eq:auditmargin}
 \alpha_{jt}=\frac{6\alpha}{\pi^2J_R(t+1)^2},\qquad
 \epsilon_{jt}^{\rm stat}=\sqrt{\frac{\log(1/\alpha_{jt})}{2m_{jt}}}.
\end{equation}
With probability at least $1-\alpha$, simultaneously for every audited $j,t$,
\begin{equation}\label{eq:populationcertificate}
 \mathscr R_{j,Q_{jt}}(\theta_t)
 \le\widehat{\mathscr R}_{jt}+\epsilon_{jt}^{\rm stat}+\Delta_{jt}.
\end{equation}
Consequently, accepting only candidates whose right-hand sides do not exceed the declared thresholds certifies those deployment risks at every accepted audit time on the same event.
\end{theorem}
\noindent\emph{Proof.} See Appendix~\ref{proof:thm:audit}.

This is a conservative application of concentration and error-probability allocation, not a new conformal or confidence-sequence method. Related risk calibration and time-uniform inference offer other designs \citep{angelopoulos2024,howard2021}. Freshness is essential to this proof: fitting several candidates on the same audit batch and reporting the best one violates the stated conditioning. The theorem certifies bounded expected harm at the audited distributions and times. It does not guarantee zero harmful outcomes, per-individual protection, or behavior under an unbounded distribution shift. An unknown $\Delta_{jt}$ cannot be set to zero for convenience.

\subsubsection{A supervisory rule and the limits of an override}
A reviewable update rule first fixes the protected specification, then computes the rows and the residual in Theorem~\ref{thm:multirisk}. If the residual is positive, it reports the violated requests and a dual certificate. A controller may drop a task-progress or entropy request according to a previously authorized policy, but may not silently remove a protected-risk row. If feasible, it selects an intervention, checks finite-step bounds, and audits the candidate before deployment when population protection is required. The candidate is released only while the responsible operator's authorization is active. Every relaxation is a new specification and requires its own certificate.

This rule distinguishes pausing learning from stopping the deployed system. Setting the parameter update to zero prevents further parameter change; it does not by itself keep a robot, vehicle, or decision service safe. Such a system needs a separately verified fallback policy and a model of physical state, uncertainty, and intervention delay, as in predictive safety filtering \citep{wabersich2021}.

\begin{corollary}[Composition for authorized, audited learning]\label{cor:composition}
Suppose every accepted parameter update satisfies Theorem~\ref{thm:discrete} with the same constants and step size, and satisfies Proposition~\ref{prop:finitesafety} for the chosen time-independent smooth protected risks. Suppose every released candidate also passes the fresh-audit conditions and threshold test of Theorem~\ref{thm:audit}. Require an active authorization flag for each release. Then the smooth protected thresholds hold on every accepted update segment; with probability at least $1-\alpha$, the audited deployment-risk thresholds hold at every release; and no release occurs while authorization is inactive. If $N_{\rm acc}\ge1$ updates are accepted, their pre-update task strengths satisfy
\begin{equation}\label{eq:acceptedprogress}
 \min_{0\le k<N_{\rm acc}} a_k
 \le\frac{2[L(\theta_0)-L_{\inf}]}{\eta\kappa N_{\rm acc}}.
\end{equation}
This statement gives no positive lower bound on $N_{\rm acc}$ per unit time.
\end{corollary}
\noindent\emph{Proof.} See Appendix~\ref{proof:cor:composition}.

The corollary states how the guarantees compose, without identifying a smooth training surrogate with a population harm loss. It also makes a possible cost of protection explicit: an incompatible specification can stop progress. Restoring progress then requires a better intervention space, a valid larger budget, better estimation, or a consciously revised specification; it cannot be obtained from the certificate alone.

\begin{proposition}[Risk reserve for a delayed supervisory response]\label{prop:latency}
Let $r_j(t)$ denote an actual protected risk minus its threshold along the deployed system, including environmental change. Suppose at alarm time $t_0$ one has $r_j(t_0)\le-m_j$, and $r_j$ is absolutely continuous with $\dot r_j(t)\le D_j$ almost everywhere until a fallback becomes effective, no later than $t_0+\Delta$. Here $m_j,D_j,\Delta\ge0$. If $m_j\ge D_j\Delta$, then $r_j(t)\le0$ during this response interval.
\end{proposition}
\noindent\emph{Proof.} See Appendix~\ref{proof:prop:latency}.

The reserve is useful only when the increase bound and response delay are valid; protection after the response requires the fallback's own guarantee. The formal human contribution in this paper is authority over specifications and release decisions. Whether a user understands the certificates, can intervene effectively, or retains independent competence is an empirical HCI question. Those outcomes cannot be established by the existence of an override switch.

\subsection{Sample-code generalization details}\label{app:generalization}
For a learned random predictor $W$ and an independent input--label pair $U$, write $R(W)=\E_U\ell(W,U)$ and $\widehat R_S(W)=n^{-1}\sum_i\ell(W,U_i)$ for the population and empirical risks of a learned predictor $W$. Algorithmic information bounds involve dependence between $S$ and $W$. The fresh-input quantity in \eqref{eq:capacity} describes a different channel.

\begin{proposition}[Zero fresh-input information does not prevent memorization]\label{prop:memorization}
For every $n\ge1$, $A>0$, and $\nu>0$, there is a sample-dependent encoder and a predictor operating on its Gaussian-noisy scalar output such that
\begin{equation}
 I(X;e_S(X)+\xi\mid S)=0,\qquad\Cov(e_S(X)\mid S)=0
\end{equation}
for an independent fresh input $X$, yet the expected generalization gap under zero--one loss is
\begin{equation}\label{eq:counterexample}
 \E[R(W)-\widehat R_S(W)]=\frac12-\Phi(-A/\sqrt\nu)>0.
\end{equation}
Here the predictor's loss averages over its independent representation noise, and $\Phi$ is the standard normal distribution function.
\end{proposition}
\noindent\emph{Proof.} See Appendix~\ref{proof:prop:memorization}.

The construction is measurable. If smooth encoders are required without a uniform derivative bound, disjoint smooth bumps around the training inputs can preserve their values while making the total support measure arbitrarily small. If the support has Lebesgue measure at most $\epsilon$ and the bump heights have magnitude at most $A$, then $\E[e_S(X)^2\mid S]\le A^2\epsilon$, and so the variance is at most $A^2\epsilon$. Equation~\eqref{eq:capacity} bounds the fresh-input information by $\tfrac12\log(1+A^2\epsilon/\nu)$, which tends to zero. The values at the training inputs, their noise-averaged training errors, and the population error $1/2$ are unchanged; thus the same generalization gap persists. The obstruction is sample-dependent encoding, not a singular differential-entropy calculation. A statement that the predictor acts through a noisy representation at prediction time does not remove it.

\subsubsection{A valid sample-code theorem}
A sufficient alternative is to constrain the channel by which the entire training sample influences the learned predictor. Let $V$ be side information independent of $S$, such as a fixed public encoder or independently trained initialization.

\begin{theorem}[Generalization through a noisy code of the training sample]\label{thm:gen}
Let $S$ contain $n$ i.i.d. copies of $U$, and let $V$ be independent of $S$. Assume a uniform finite variance proxy $\sigma_\ell^2>0$: for every fixed predictor $\omega$ and every $\lambda\in\R$,
\[
 \E_U\exp\{\lambda[\ell(\omega,U)-\E_U\ell(\omega,U)]\}
 \le\exp(\lambda^2\sigma_\ell^2/2).
\]
Let $d\ge1$ be the sample-code dimension, $\phi_V(S)\in\R^d$ a measurable deterministic code given $(S,V)$, and $\nu>0$ its Gaussian noise variance. Suppose
\begin{equation}\label{eq:code}
 Q=\phi_V(S)+\zeta,\qquad \zeta\sim N(0,\nu I_d),
 \qquad S\longrightarrow(Q,V)\longrightarrow W,
\end{equation}
where $Q\in\R^d$ is the released code, the Gaussian noise is independent of $(S,V)$, the last relation is conditional independence, and $C_V=\Cov(\phi_V(S)\mid V)$ is finite. Then
\begin{equation}\label{eq:gen}
 \left|\E[R(W)-\widehat R_S(W)]\right|
 \le\sqrt{\frac{2\sigma_\ell^2}{n}I(S;W\mid V)}
 \le\sqrt{\frac{\sigma_\ell^2}{n}\E_V\log\det(I_d+C_V/\nu)}.
\end{equation}
\end{theorem}
\noindent\emph{Proof.} See Appendix~\ref{proof:thm:gen}.

A loss in $[0,1]$ admits $\sigma_\ell^2=1/4$ by Hoeffding's lemma. Theorem~\ref{thm:gen} is a conditional specialization of existing information-theoretic bounds, not a new universal generalization law. Its covariance varies over possible training samples conditional on $V$; it is not the covariance of activations across one minibatch. The absolute value is outside the expectation. The theorem does not bound the expectation of the absolute gap or a realized high-probability gap.

The Markov condition is substantive. If a readout, a controller, a checkpoint selector, or a hyperparameter rule retains additional information about $S$, that information must be included in $Q$ or charged separately. In particular, the deterministic neural training procedures in Section~\ref{sec:protected-benchmark} and Appendix~\ref{app:legacy} do not implement \eqref{eq:code}, and its empirical covariance curves are not used to evaluate the right-hand side of \eqref{eq:gen}.

\begin{corollary}[Sequential releases and an information budget]\label{cor:sequential}
Let $T\ge1$ be a fixed number of releases. For $t=1,\ldots,T$, let $d_t\ge1$ be the code dimension, $\nu_t>0$ the noise variance, and $Q_{<t}=(Q_1,\ldots,Q_{t-1})$ the earlier releases. Suppose the only sample-dependent information available to the final predictor consists of these releases
\begin{equation}
 Q_t=\phi_t(S,Q_{<t},V)+\zeta_t,\qquad
 \zeta_t\sim N(0,\nu_tI_{d_t}),
\end{equation}
where $\phi_t$ is a measurable deterministic $\R^{d_t}$-valued map given its arguments and the new noise is independent of $(S,V,Q_{<t})$. Set $C_t=\Cov(\phi_t(S,Q_{<t},V)\mid Q_{<t},V)$. Assume these conditional covariances are finite almost surely. Under the sampling and loss assumptions of Theorem~\ref{thm:gen},
\begin{equation}\label{eq:sequential}
 \left|\E[R(W)-\widehat R_S(W)]\right|
 \le\sqrt{\frac{\sigma_\ell^2}{n}\sum_{t=1}^T
 \E\log\det(I_{d_t}+C_t/\nu_t)}.
\end{equation}
\end{corollary}
\noindent\emph{Proof.} See Appendix~\ref{proof:cor:sequential}.

This corollary identifies a possible connection between controlled learning and statistical complexity: a controller could budget the information released by the training process. It does not justify replacing conditional sample-code covariance by representation covariance. If the sum in \eqref{eq:sequential} is $O(n^\alpha)$ for $\alpha<1$, the bound is $O(n^{(\alpha-1)/2})$; establishing that condition for a useful learning algorithm remains a separate task.

\subsection{Additional mathematical details}\label{app:math}
\subsubsection{A direct derivation of the information inequality}
Let $\mathsf P$ be the joint law of $(S,W)$ and $\mathsf Q=\mathsf P_S\otimes \mathsf P_W$, conditionally on fixed independent side information when present. For any fixed $\omega$, the centered average loss is $\sigma_\ell^2/n$-sub-Gaussian by independence of $U_i$. Thus, for $\Delta_{\rm gen}=R(W)-\widehat R_S(W)$,
\begin{equation}
 \log\E_{\mathsf Q} e^{\lambda \Delta_{\rm gen}}\le\frac{\lambda^2\sigma_\ell^2}{2n}.
\end{equation}
To detail the first step, for fixed $\omega$ independence of the examples gives
\[
 \E_S e^{\lambda[R(\omega)-\widehat R_S(\omega)]}
 =\prod_{i=1}^n\E_{U_i}e^{(\lambda/n)[R(\omega)-\ell(\omega,U_i)]}
 \le\prod_{i=1}^n e^{\lambda^2\sigma_\ell^2/(2n^2)}.
\]
Integrating over an independent $W$ under $\mathsf Q$ yields the displayed moment bound. For finite information, put $z=d\mathsf P/d\mathsf Q$, the Radon--Nikodym density. Jensen's inequality gives
\[
 \log\E_{\mathsf Q} e^{\lambda \Delta_{\rm gen}}\ge\log\E_{\mathsf P}[e^{\lambda \Delta_{\rm gen}}/z]
 \ge\E_{\mathsf P}\log(e^{\lambda \Delta_{\rm gen}}/z)
 =\lambda\E_{\mathsf P}\Delta_{\rm gen}-\KL(\mathsf P\Vert \mathsf Q).
\]
The first inequality accounts for possible $\mathsf Q$-mass outside the support of $\mathsf P$. Since $\KL(\mathsf P\Vert \mathsf Q)=I(S;W)$, it follows that, for every $\lambda>0$,
\[
 \E_{\mathsf P}\Delta_{\rm gen}\le I(S;W)/\lambda+\lambda\sigma_\ell^2/(2n).
\]
Apply the same argument to $-\Delta_{\rm gen}$. When the information is positive and finite, and $\sigma_\ell>0$, the optimizing value and resulting bound are
\[
 \lambda_\star=\frac{\sqrt{2nI(S;W)}}{\sigma_\ell},\qquad
 |\E_{\mathsf P}\Delta_{\rm gen}|\le\sqrt{\frac{2\sigma_\ell^2 I(S;W)}{n}}.
\]
 Zero information is handled by $\lambda\downarrow0$; infinite information gives a vacuous bound. A positive sub-Gaussian proxy can be used whenever a zero proxy is valid. The conditional argument and Jensen step over $V$ are given in Theorem~\ref{thm:gen}. Bounded losses satisfy a uniform sub-Gaussian assumption; unbounded cross-entropy needs its own tail condition, which is not inferred from a finite experimental loss.

The Gaussian maximum-entropy fact used in both information channels also follows from nonnegativity of relative entropy. If a vector $Y$ has mean $\mu$, positive definite covariance $\Sigma$, and a density, compare its law with $N(\mu,\Sigma)$. Since
$\E[(Y-\mu)^\top\Sigma^{-1}(Y-\mu)]=\tr(\Sigma^{-1}\Sigma)=p$,
\[
 0\le\KL(P_Y\Vert N(\mu,\Sigma))
 =-\mathrm h_{\rm diff}(Y)+\tfrac12\log\det(2\pi e\Sigma).
\]
This proves the entropy bound when the quantities are finite. The Gaussian-noise channels with finite covariance used here have well-defined finite entropies. Equality holds for Gaussian $Y$.

\subsubsection{Fixed-energy descent and time-varying coefficients}
For a fixed $\beta$, define $E_\beta=L+\beta H$ and its velocity $w_E=-\nabla E_\beta$. For $\ell_E>0$, a lower-bounded, $\ell_E$-smooth energy satisfies
\begin{equation}
 E_\beta(\theta+\eta w_E)\le E_\beta(\theta)-\eta(1-\ell_E\eta/2)\norm{\nabla E_\beta}^2.
\end{equation}
For constant $0<\eta<2/\ell_E$, summation yields $\norm{\nabla E_\beta(\theta_t)}\to0$ along an infinite run, not merely a zero liminf. This standard statement concerns a fixed objective. For differentiable $\beta(t)$ and Euclidean flow $\dot\theta=-\nabla L-\beta\nabla H$,
\begin{equation}
 \frac{d}{dt}[L(\theta(t))+\beta(t)H(\theta(t))]
 =-\norm{\nabla L+\beta\nabla H}^2+\dot\beta(t)H(\theta(t)).
\end{equation}
The fixed-step inequality follows by integrating the gradient along the segment and bounding the remainder by $\ell_E\eta^2\norm{\nabla E_\beta}^2/2$. Summing gives a finite sum of nonnegative squared gradient norms. For the time-dependent expression, the chain rule has two parts: $(\nabla L+\beta\nabla H)^\top\dot\theta$ and $\dot\beta H$. Substitution of the velocity gives the displayed identity. The last term can be positive. A bounded coefficient by itself is insufficient to make this time-varying energy a Lyapunov function. This is why Theorem~\ref{thm:discrete} protects the fixed task objective directly.

\section{Notation, literature comparison and exact numerical checks}
\subsection{Conventions for information and probability}\label{app:corenotation}
Vectors are columns, logarithms are natural, and expectations specify their sampling law or conditioning. The representation functional $H_\nu$ is measured in nats per representation coordinate; $pH_\nu$ is the total covariance information proxy. Sample-code information bounds are total information quantities and are not divided by representation dimension. Table~\ref{tab:specialnotation} records the distinct roles of the remaining specialized symbols. Matrix decompositions used inside a single proof are local constructions.

\subsection{Specialized symbols and assumptions}\label{app:notation}
The physical velocity is $w$, the additional intervention is $u$, and $\theta^+=\theta+\eta w$. Reduced probe coordinates are $z_U$, with $w=Tz_U$. Symbols defined only inside an individual derivation are explicitly local.
\begin{longtable}{>{\raggedright\arraybackslash}p{.29\linewidth}>{\raggedright\arraybackslash}p{.63\linewidth}}
\caption{Specialized notation and the distinct evidence requirements.}\label{tab:specialnotation}\\
\toprule Symbols & Meaning\\\midrule\endfirsthead
\toprule Symbols & Meaning\\\midrule\endhead\bottomrule\endfoot
$p,D,d_y,m,m_F,k_U,r_Z$ & Representation, parameter, prediction-output, probe-sample, observation-output and intervention dimensions; centered-representation rank.\\
$\mathcal H_\nu(C),H_\nu(Z),H(\theta),\nu$ & Covariance-domain functional, batch functional, parameter composition and declared noise variance; $pH$ is the total proxy.\\
$w,v_0,u,B_{\rm tot},B_{\rm int}$ & Total velocity, task reference, additional intervention, and their distinct budgets.\\
 $e_\theta,r_\theta,f_\theta$ & Encoder, readout and composed predictor.\\
$g,h$ & Column gradients of the task objective and normalized representation functional.\\
$L_r$ & Lipschitz constant of the readout.\\
$\Pi_{\rm obs},\Pi_{\mathcal U},P_m$ & Observable projector, admissible-subspace projector, and representation centering matrix.\\
$h_{\rm eff},a,b,c,b_{\rm obs},c_{\rm obs}$ & Conflict-corrected information differential and signed or squared gradient products.\\
$C_F,\widehat C_F,z_U$ & Population/empirical restricted-prediction Gram matrices and reduced intervention coordinates.\\
$A_{\rm grp},\varpi,R_j,k_j,\tau_j,e_j$ & Group indicator, group probability, protected risk, its gradient, risk cap and excess.\\
$d_i,t_i,s_i$ & Constraint-row normal, right-hand side and positive row-normalization scale.\\
$\delta_i^d,\delta_i^t,\varepsilon_\star,\varepsilon_{\rm rob},\chi$ & Certified row errors, true and robust residuals, and uncertainty cost.\\
$D_{\rm mov},H_\star,v_\star,\Delta$ & Logged parameter-path bound, authorized entropy budget, its drift bound, and feedback delay.\\
$\ell_{\rm harm},\ell_j^{\rm harm},P_{jt},Q_{jt}$ & Specified harm losses and audit/deployment probability laws.\\
$S,Q,V,W,C_V,\sigma_\ell$ & Complete sample, noisy sample code, independent public information, learned predictor, sample-code covariance, and loss sub-Gaussian parameter.\\
\end{longtable}
A derivative constraint needs a solution and feasible requirements along the trajectory. A finite-step bound needs valid curvature and estimation margins. A fresh audit needs its stated data-selection and sampling conditions. Delayed response needs a valid drift bound and executable fallback. Reuse of one empirical control set provides neither independence nor an uncharged population certificate.

\subsection{Detailed relation to prior frameworks}

The object studied here is the set of entropy interventions that remain available after specifying what is observed, which learning rates must be preserved, which risks must stay below thresholds, and how much intervention is permitted. A fixed regularized objective does not state this set. A generic constrained optimizer can compute it once the appropriate constraints are supplied; the contribution must therefore lie in the learning-specific characterization and its consequences, rather than in renaming constrained optimization.

\begin{longtable}{>{\raggedright\arraybackslash}p{.21\linewidth}>{\raggedright\arraybackslash}p{.35\linewidth}>{\raggedright\arraybackslash}p{.35\linewidth}}
\caption{Scope of the comparison. Each row identifies a related research question and the additional question addressed here; it does not claim that an entire literature lacks safety, geometry, or feedback.}\label{tab:comparison}\\
\toprule Related framework & Established object or mechanism & Question addressed in this manuscript\\\midrule
\endfirsthead
\toprule Related framework & Established object or mechanism & Question addressed in this manuscript\\\midrule
\endhead
\bottomrule\endfoot
IB / VIB & Compression--prediction trade-off in an information objective \citep{tishby1999,alemi2017}. & Is the chosen entropy force observable, and can its requested rate coexist with explicit protected rates?\\
MCR$^2$ / VICReg & Coding-rate or variance--covariance geometry of representations \citep{yu2020coding,bardes2022}. & Does a change in the scalar proxy reflect predictor change, scale change, or a prediction-null action?\\
Natural gradient & Metric-dependent optimization with a coordinate transformation rule \citep{amari1998}. & Which part of the entropy velocity remains after projection onto a declared observation map?\\
GradNorm / PCGrad / ConFIG & Gradient balancing and conflict manipulation \citep{chen2018,yu2020pcgrad,liu2025}. & What is the exact attainable entropy rate under a norm budget and several protected requirements?\\
Barrier functions / predictive safety filters & Invariance conditions and filtered physical control actions \citep{ames2017,wabersich2021}. & How do such constraints interact with observable entropy interventions in parameter space, and when must a learning request be rejected?\\
Constrained policy optimization & Policy updates subject to expected-cost constraints in reinforcement learning \citep{achiam2017}. & How do representation symmetry and force observability change an entropy-control feasibility calculation? No new policy-performance bound is claimed.\\
Risk calibration & Statistical control of a specified risk from calibration data \citep{angelopoulos2024}. & How is a local learning certificate separated from, and combined with, an independently valid population-risk audit?\\
Learning to defer / preference learning & Allocation of predictions to experts, or learning objectives from feedback \citep{mozannar2020,christiano2017}. & Who may change a threshold or release a candidate after conflicting requests are detected? The present theory does not estimate expert competence or preferences.\\
Information-theoretic generalization & Dependence between a sample and the learned predictor \citep{xu2017,kawaguchi2023,sefidgaran2023}. & Which noisy sample channel supports a bound, and why is fresh-input covariance insufficient?\\
\end{longtable}

The central claims to assess are the observable restriction of entropy-rate feasibility (Theorem~\ref{thm:obscontrol}), the specification-dependent feasibility and uncertainty obstructions (Theorems~\ref{thm:multirisk} and~\ref{thm:robustmultirisk}), and the entropy capacity remaining under protection (Theorem~\ref{thm:conecapacity}). Their interpretation is sharpened by two counterexamples: prediction-preserving covariance compression and aggregate improvement that harms a protected group. The cone identity, minimax step, barrier comparison, and concentration inequalities are standard tools, and their use is identified in the proofs. The distinguishing contribution is their application to a jointly specified entropy-learning problem, including exact obstructions and the observable/unobservable distinction. Mathematical validity can be proved; historical priority cannot be inferred from a proof or from terminology. A stronger claim of novelty would require further comparison with methods solving the same joint specification.

\subsection{Information examples and numerical checks}\label{app:informationchecks}
We separate mathematical implementation checks, exactly specified information examples, and neural-training evidence. The code, split indices, all configurations, per-run results, and numerical verification outputs accompany the manuscript. The primary benchmark below is newly executed. Appendix~\ref{app:legacy} retains the earlier fixed-split digits diagnostic and its original results; the two studies have different protocols and are reported separately.

\subsubsection{Tests of the mathematical mechanisms}
The retained analytic-gradient checks agree with central differences to below $3\times10^{-10}$. Independent primal--dual checks of the bounded intervention, multi-risk and robust-row formulas have maximum objective discrepancies below $8\times10^{-9}$; all accepted constraint residuals are below $10^{-7}$. The raw record retains one non-success SLSQP status whose independent dual gap and primal residual met this tolerance. Separate checks cover rank deficiency, zero support, uncertainty dominance and quadratic finite-step risk bounds. These calculations test the implementation of stated formulas; the proofs supply their mathematical justification.

For the uniform-probe result, 24,000 draws from two known finite-support laws verify the Gram and finite-update inequalities. The informative finding is conservatism: in the rare six-direction design, the population minimum eigenvalue is $0.001$, while at 8,192 probes the uniform error radius is $0.06549$. Empirical nullity occurs in only $0.05\%$ of those draws, yet the positive-population-rank certificate never passes. Observed rank and certified action are therefore different quantities. Thirty-two scalar delayed-budget cases separately check the tracking law; ten encounter actuator saturation, where the global tracking conclusion is deliberately withheld. Full scripts, assumptions and numerical records accompany these checks.

Figure~\ref{fig:mechanisms} displays two exact consequences of the theory. The spectral response is small at both ends of the scale axis, and the admissible entropy rate depends on the conflict angle. The curves are analytic evaluations with no fitted noise, confidence bands, or simulated performance claim. Proposition~\ref{prop:memorization} gives another explicit check: for $A/\sqrt\nu=3$, the expected gap is $0.4986501$ although the fresh-input covariance and mutual information are zero. No training algorithm was run to produce that counterexample value.

\begin{figure}[t]
\centering\includegraphics[width=\linewidth]{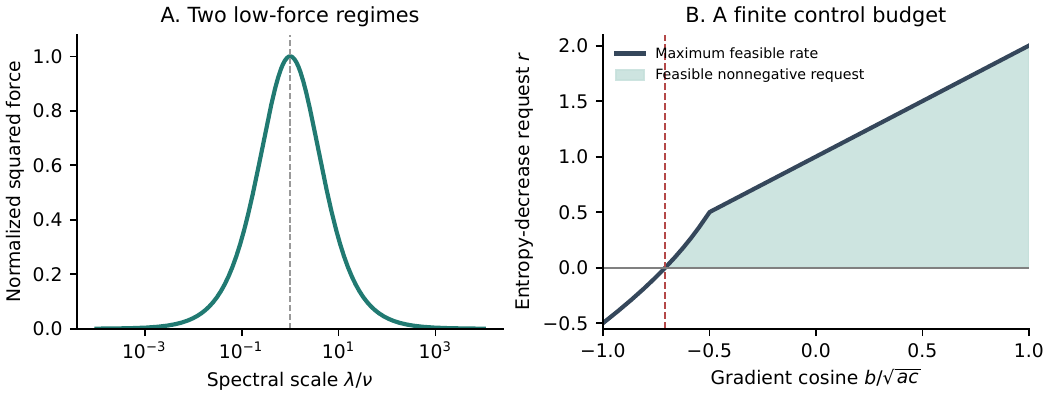}
\caption{Exact mechanisms. Left: one spectral contribution to the squared log-determinant force, normalized by its maximum, as a function of $\lambda/\nu$. Both near-collapse and large-scale representations can produce small force. Right: the largest feasible requested entropy-decrease rate for $a=c=1$, $\kappa=1/2$, and $0\le\beta\le1$. Only nonnegative requests are shaded. Strong negative gradient alignment makes even entropy nonincrease infeasible. These panels evaluate Theorem~\ref{thm:spectral} and Theorem~\ref{thm:feasible}; they are not empirical learning curves.}\label{fig:mechanisms}
\end{figure}

\subsubsection{What the entropy proxy measures}
The Gaussian covariance bound is exact for a Gaussian noisy representation and can be loose for a non-Gaussian one. The following standard binary-input calculation quantifies that gap in a fully specified channel. It is a calibration example for the interpretation of the proxy, not a neural entropy estimator.
For a standard binary-input Gaussian channel, let $X$ be equally likely $-1$ or $1$, let $N\sim N(0,1)$ be independent, and set $Y=aX+\sqrt\nu N$ with $a\ge0$ and $\nu>0$. Writing $s=a/\sqrt\nu$, the covariance proxy is $H_\nu=\tfrac12\log(1+s^2)$, whereas
\[
 I(X;Y)=\log2-\E_N\log(1+\exp[-2s(s+N)])\le\log2.
\]
The proxy gap is exactly the relative entropy from the law of $Y$ to its matching Gaussian and obeys $H_\nu-I(X;Y)\ge\max\{0,H_\nu-\log2\}$. It diverges as $s$ grows, although the true information approaches one bit. Table~\ref{tab:binaryentropy} evaluates this standard example. Appendix~\ref{app:binaryentropy} derives the identities and records quadrature accuracy. This specifies a limitation of covariance volume; it does not estimate neural-representation entropy.

\begin{table}[t]
\centering
\caption{Standard binary-input Gaussian-channel example. Natural logarithms give nats. The covariance proxy $H_\nu$ can substantially exceed the true noisy-input information at large normalized amplitude $s=a/\sqrt\nu$. The error estimate concerns the conditional-entropy integral before displayed rounding. These are deterministic quadrature values for a specified distribution, not entropy estimates on learned representations.}
\label{tab:binaryentropy}
\begin{tabular}{rrrrr}
\hline
$s$ & $H_\nu$ & $I(X;Y)$ & $H_\nu-I(X;Y)$ & Integral error estimate\\
\hline
0.1 & 0.004975165 & 0.004975165 & $8.0113\times10^{-10}$ & $4.32\times10^{-14}$\\
0.3 & 0.043088848 & 0.043084882 & $3.9665\times10^{-6}$ & $3.56\times10^{-14}$\\
1 & 0.346573590 & 0.336830820 & 0.009742770 & $1.51\times10^{-14}$\\
3 & 1.151292546 & 0.689297933 & 0.461994613 & $1.09\times10^{-16}$\\
10 & 2.307560258 & 0.693147181 & 1.614413078 & $1.72\times10^{-36}$\\
\hline
\end{tabular}
\end{table}

\subsubsection{Tightness in an actual sample-code learner}
A separate finite-label construction implements the complete channel of Theorem~\ref{thm:gen}. Its expected gap and both information bounds can be evaluated exactly by finite sums. Table~\ref{tab:samplecodegap} compares the actual gap with the actual-information bound and the Gaussian sample-channel bound. Appendix~\ref{app:samplecodegap} specifies the learner and derives every expression. This tests tightness within that restricted channel; the neural benchmark below does not satisfy the channel assumption.
\begin{table}[t]
\centering
\caption{Exact expected gaps and valid generalization bounds in the noisy sample-mean model. $B_I=\sqrt{I(S;W)/(2n)}$ uses exact finite-sum information; $B_G=\tfrac12\sqrt{\log(1+1/(n\nu))/n}$ is the Gaussian sample-code bound. Ratios compare each bound with the exact positive expected gap, not with an absolute or high-probability gap.}
\label{tab:samplecodegap}
\begin{tabular}{rrr rr rr}
\hline
$n$ & $n\nu$ & Expected gap & $B_I$ & $B_G$ & $B_I$/gap & $B_G$/gap\\
\hline
16 & 0.01 & 0.098190 & 0.131936 & 0.268535 & 1.344 & 2.735\\
16 & 1 & 0.070991 & 0.078059 & 0.104069 & 1.100 & 1.466\\
16 & 100 & 0.009926 & 0.009941 & 0.012469 & 1.002 & 1.256\\
64 & 0.01 & 0.049636 & 0.069429 & 0.134268 & 1.399 & 2.705\\
64 & 1 & 0.035319 & 0.038891 & 0.052035 & 1.101 & 1.473\\
64 & 100 & 0.004962 & 0.004970 & 0.006234 & 1.002 & 1.256\\
256 & 0.01 & 0.024818 & 0.034842 & 0.067134 & 1.404 & 2.705\\
256 & 1 & 0.017638 & 0.019428 & 0.026017 & 1.101 & 1.475\\
256 & 100 & 0.002481 & 0.002485 & 0.003117 & 1.002 & 1.256\\
\hline
\end{tabular}
\end{table}

\subsection{Exact entropy-proxy calibration in a binary Gaussian channel}
\label{app:binaryentropy}
A tractable distribution separates the covariance proxy from actual noisy-representation entropy. Let $X$ take $-1$ and $1$ with equal probability, let $N\sim N(0,1)$ be independent of $X$, and set
\[
 Y=aX+\sqrt\nu N,\qquad a\ge0,\quad\nu>0,\quad s=a/\sqrt\nu.
\]
All logarithms below are natural, so information is measured in nats. The covariance proxy and the actual noisy-input information are
\begin{align}\label{eq:binarymi}
 H_\nu&=\tfrac12\log(1+s^2),\\
 I(X;Y)&=\log2-
 \E_N\log\!\left(1+\exp[-2s(s+N)]\right).
\end{align}
This is the standard binary-input Gaussian channel, used here as a worked example rather than a novelty claim.

To derive the second expression, write $\varphi_\nu$ for the $N(0,\nu)$ density. Then
$p_Y(y)=\tfrac12\varphi_\nu(y-a)+\tfrac12\varphi_\nu(y+a)$, and Bayes' rule gives
\[
 \Prob(X=1\mid Y=y)=\frac{1}{1+\exp(-2ay/\nu)}.
\]
Conditional on $X=1$, one has $Y/\sqrt\nu=s+N$, so the negative logarithm of this posterior is $\log(1+\exp[-2s(s+N)])$. The corresponding conditional expectation for $X=-1$ is the same by symmetry. Thus the expectation in \eqref{eq:binarymi} equals $\mathrm H(X\mid Y)$, and $I(X;Y)=\mathrm H(X)-\mathrm H(X\mid Y)$ proves the formula.

The precise entropy objects matter. Since $\operatorname{Var}Y=a^2+\nu$ and the independent Gaussian noise has differential entropy $\tfrac12\log(2\pi e\nu)$,
\[
 I(X;Y)=\mathrm h_{\rm diff}(Y)-\tfrac12\log(2\pi e\nu).
\]
Consequently, the gap
\begin{align}\label{eq:binarygap}
 \Delta(s)&=H_\nu-I(X;Y)\\
 &=\tfrac12\log\!\left(2\pi e[a^2+\nu]\right)
   -\mathrm h_{\rm diff}(Y)\\
 &=D_{\rm KL}\!\left(P_Y\,\Vert\,N(0,a^2+\nu)\right)\ge0.
\end{align}
For the last equality, expand the relative entropy against the centered Gaussian with matching variance. Its cross-entropy is exactly $\tfrac12\log(2\pi e[a^2+\nu])$, because $\E Y^2=a^2+\nu$. Nonnegativity follows from the usual relative-entropy inequality. Thus this is also exactly the differential-entropy upper gap after restoring the constants omitted from the covariance proxy; it is not the differential entropy of the discrete variable $aX$.

Because $I(X;Y)\le\mathrm H(X)=\log2$,
\begin{equation}\label{eq:binarygaplower}
 \Delta(s)\ge\max\{0,\tfrac12\log(1+s^2)-\log2\},
 \qquad\Delta(s)\longrightarrow\infty\quad(s\longrightarrow\infty).
\end{equation}
The true information saturates at one bit while the covariance proxy keeps growing. More precisely, the decision $\widehat X=\operatorname{sign}Y$ has error probability $p_e=\Phi(-s)$, where $\Phi$ is the standard normal distribution function. Given $Y$, the error indicator $E=\mathbf1\{\widehat X\ne X\}$ determines $X$ and conversely. Therefore
\[
 \mathrm H(X\mid Y)=\mathrm H(E\mid Y)\le\mathrm H(E)
 =h_b(p_e),\qquad
 h_b(p)=-p\log p-(1-p)\log(1-p).
\]
It follows that $\log2-h_b(\Phi(-s))\le I(X;Y)\le\log2$, proving $I(X;Y)\to\log2$ and
$\Delta(s)=\tfrac12\log(1+s^2)-\log2+o(1)$.

The companion table evaluates the one-dimensional integral in \eqref{eq:binarymi} with \texttt{scipy.integrate.quad}, split at $N=-s$ and $N=0$. Each integral piece uses absolute tolerance $10^{-40}$ and relative tolerance $10^{-13}$. The reported error estimates are those of the quadrature routine, not interval-arithmetic certificates. At $s=10$, the computed conditional entropy is approximately $2.365714\times10^{-23}$ nats; $I(X;Y)$ consequently rounds to $\log2$ in double precision. The source records that conditional entropy separately. This example answers the surrogate-gap question for a specified law only; it does not estimate the gap for neural representations or test an algorithmic generalization bound.

\subsection{Bound tightness under an actual noisy sample code}
\label{app:samplecodegap}
The sample-code bound can be compared with an exact expected generalization gap in a deliberately simple model. This calculation is distinct from the invalid substitution of fresh-input activation covariance for sample-level information. It is a standard finite-model check, not evidence of bound tightness for the neural experiments.

Let $S=(Y_1,\ldots,Y_n)$ consist of independent labels satisfying $\Prob(Y_i=1)=\Prob(Y_i=-1)=1/2$. There are no features. Define the sample mean and its noisy release by
\[
 \overline Y=\frac1n\sum_{i=1}^nY_i,\qquad Q=\overline Y+\zeta,
 \qquad\zeta\sim N(0,\nu),\quad\nu>0,
\]
with noise independent of $S$, and release only the constant classifier $W=\operatorname{sign}Q\in\{-1,1\}$. The event $Q=0$ has probability zero. The channel $S\to Q\to W$ therefore holds exactly, and $\operatorname{Var}\overline Y=1/n$ is variance across possible training samples.

For zero--one loss $\ell(w,y)=(1-wy)/2$, the population risk against a fresh independent balanced label is $R(w)=1/2$, whereas
\[
 \widehat R_S(W)=\tfrac12(1-WM),\qquad
 R(W)-\widehat R_S(W)=\tfrac12WM.
\]
Write $m_k=(2k-n)/n$ and $p_k=2^{-n}\binom nk$ for $k=0,\ldots,n$, and let $\Phi$ denote the standard normal distribution function. Conditional on $\overline Y=m_k$, one has
$\Prob(W=1\mid \overline Y=m_k)=\Phi(m_k/\sqrt\nu)$, hence
\begin{equation}\label{eq:samplemeanexactgap}
 G_{n,\nu}:=\E[R(W)-\widehat R_S(W)]
 =\frac12\sum_{k=0}^np_km_k
 \left[2\Phi(m_k/\sqrt\nu)-1\right]>0.
\end{equation}
The strict inequality follows because the bracket has the sign of $m_k$, and $\overline Y\ne0$ with positive probability.

The conditional law of $W$ depends on $S$ only through $\overline Y$, so $\mathrm H(W\mid S)=\mathrm H(W\mid \overline Y)$. Symmetry makes $W$ marginally balanced. With $h_b(p)=-p\log p-(1-p)\log(1-p)$ and natural logarithms,
\begin{equation}\label{eq:samplemeanexactmi}
 I(S;W)=\log2-\sum_{k=0}^np_k
 h_b\!\left(\Phi(m_k/\sqrt\nu)\right).
\end{equation}
Thus both the expected gap and the algorithmic information are finite sums over $n+1$ values, without Monte Carlo or entropy estimation.

For any fixed classifier $w$, the centered zero--one loss takes the values $-1/2$ and $1/2$ equally often. Its moment generating function is $\cosh(t/2)\le\exp(t^2/8)$, so the sub-Gaussian variance proxy is $\sigma_\ell^2=1/4$. The usual information bound and the Gaussian sample-code inequality consequently give
\begin{equation}\label{eq:samplemeanboundchain}
 G_{n,\nu}
 \le \underbrace{\sqrt{\frac{I(S;W)}{2n}}}_{B_I}
 \le \underbrace{\frac12\sqrt{\frac{\log(1+1/(n\nu))}{n}}}_{B_G}.
\end{equation}
Indeed, data processing gives $I(S;W)\le I(S;Q)$, and Gaussian maximum entropy with $\operatorname{Var}Q=1/n+\nu$ yields
\[
 I(S;Q)=\mathrm h_{\rm diff}(Q)-\mathrm h_{\rm diff}(\zeta)
 \le\tfrac12\log(1+1/(n\nu)).
\]
Substitution proves the second bound. The covariance here is the covariance of the code signal $\overline Y$ over training samples; no within-minibatch activation covariance appears.

Table~\ref{tab:samplecodegap} evaluates \eqref{eq:samplemeanexactgap}--\eqref{eq:samplemeanboundchain} by enumerating every binomial outcome for $n\in\{16,64,256\}$ and $n\nu\in\{0.01,1,100\}$. In these cases $B_I/G_{n,\nu}$ ranges from approximately $1.002$ to $1.404$, and $B_G/G_{n,\nu}$ from $1.256$ to $2.735$. The additional data-processing and Gaussian covariance relaxations are therefore measurable even when the underlying sample channel is valid; this comparison does not isolate those two sources of slack from one another. These ratios concern the exact expected signed gap, which is positive here; they do not measure an expected absolute gap or a realized high-probability gap. Since the labels contain no predictive signal, the example tests the bound and its channel assumptions rather than the utility of learning.

\section{Complete primary and follow-up protocols}
\subsection{Primary benchmark protocol}\label{app:benchmarkprotocol}

The empirical question is whether a nonzero entropy intervention can remain admissible under separately specified task and group-risk requirements, and what predictive and computational costs this entails. We separate the effect of enforcing a risk contract from the additional effect of requesting entropy motion. The numerical contract instantiates the role assigned to external authority in HCLM: it declares protected quantities, allowable losses, an information request, and a control budget. Its numerical values were specified by the experimental protocol; they were not elicited from human participants.

\subsubsection{Data, learning tasks, and fixed specifications}
We use the offline scikit-learn Digits dataset \citep{sklearn} and Wisconsin Diagnostic Breast Cancer (WDBC) dataset \citep{wdbc1993}. Digits has 1,797 observations, 64 features and ten classes; WDBC has 569 observations, 30 features and two classes. Each dataset receives eight stratified 60/20/20 training/control/test splits, with paired initializations across methods. The resulting fitting/control/test sizes are 1,077/360/360 and 341/114/114. Digits inputs are divided by 16 and centered using training means. WDBC inputs are standardized using training means and standard deviations. The encoder is one tanh layer followed by a linear softmax readout, with latent dimension 16 for Digits and eight for WDBC, giving 1,210 and 266 fitted parameters. The implementation uses analytic full-batch gradients in NumPy.

Every paired run starts with 120 shared task-only updates. We then perform 180 continuation updates under two declared objectives. Uniform continuation keeps ordinary cross-entropy. In the objective-reweighting stress test, the empirical mean loss within group~1 receives weight 0.95 and that within group~0 receives weight 0.05. Digits groups are odd versus even class labels; WDBC groups are benign versus malignant labels, with group~0 the de-emphasized group. This is an imposed fine-tuning pressure on real labeled data, not an observed field distribution shift. The WDBC experiment studies classification mechanics and supplies no clinical safety or intervention evidence.

Let $R_j(\theta)$ be the mean cross-entropy in group $j$ on the separate control set, for $j\in\{0,1\}$. At the common warmstart $\theta_0$, the fixed caps are
\[
 c_j=1.02R_j(\theta_0)+0.005.
\]
They remain identical across methods within a paired run. Repeated use makes the control set part of the adaptive fitting procedure. Held-out predictions are evaluated only after the protocol is fixed and training is complete; control-set compliance is not a population guarantee.

For $n$ training representations $z_i\in\mathbb R^d$, with sample covariance $\widehat\Sigma_Z$, the monitored functional is
\[
 H(\theta)=\frac{1}{2d}\log\det\!\left(I_d+\widehat\Sigma_Z/0.1\right).
\]
The declared reference variance 0.1 is a design scale, rather than an estimated physical sensor-noise variance. The factor $1/d$ gives a per-coordinate normalization. This covariance functional remains a surrogate; the experiment does not estimate its gap to true representation entropy.

\subsubsection{Implemented controller and matched ablations}
For this subsection, write $g=\nabla L$, $h=\nabla H$, and $r_j=\nabla R_j$, with $g,h$ computed on training data and $r_j$ on control data. The code uses the descent-vector convention $\theta^+=\theta-\eta d$, so its parameter velocity is $v=-d$. The balanced proposal is $d_0=g+\beta h$, where
\[
 \beta=\min\{20,\,0.5\|g\|/\max(\|h\|,10^{-30})\}.
\]
Full HCLM projects this proposal in Euclidean distance onto the following intersection:
\begin{align}
 \|d\|&\leq2\|g\|, &
 g^\top d&\geq0.5\|g\|^2, \label{eq:benchmark-task-budget}\\
 r_j^\top d&\geq R_j-c_j+0.05\|r_j\|\|g\|,\quad j=0,1, &
 h^\top d&\geq\chi_v, \label{eq:benchmark-risk-entropy}
\end{align}
where $\chi_v=0.25\|g\|\|h\|$ while $H>0.8H(\theta_0)$ and $\chi_v=0$ otherwise. The norm budget is on the total descent vector $d$, not solely on an added information intervention. The $0.05\|r_j\|\|g\|$ term is a declared inward-rate reserve. It is not a statistical uncertainty radius. The entropy request is a local rate requirement; it is not an exponential tracking law or a guarantee of reaching $0.8H(\theta_0)$ within the update horizon.

The generic halfspace--ball certificate applies to this total-vector formulation by using decision variable $d$, normals $-g,-r_j,-h$, and right-hand sides equal to the negatives of the corresponding required lower rates. This is an origin-centered ball for the total vector, rather than the earlier ball for an intervention added to task descent.

The solver enumerates active faces after reducing to the span of the proposal and constraint gradients. With two group rows, this small problem has at most four halfspaces. If the entropy requirement is infeasible, the predetermined fallback removes that row while retaining the protected rows. If the remaining set is infeasible, the update is held. These branches are retained in the code and logs.

A common initial step 0.3 is halved at most 26 times. For full HCLM, a step is accepted only after actual training-loss descent, actual group losses below both fixed caps, and, when a positive entropy request is admitted, actual entropy decrease have been checked. The Armijo fraction is 0.1 and the numerical risk tolerance is $10^{-10}$. Thus the first-order conditions and realized finite-step checks remain distinct.

The original frozen comparison contains ERM, fixed logdet, balanced logdet, task-orthogonal entropy, a two-objective PCGrad adaptation, HCLM without risk rows, a risk-only filter, and full HCLM. Fixed logdet uses the initial norm-matched $\beta_0$ throughout, rather than a test-selected coefficient. The task-orthogonal method uses $g+\beta[h-(g^\top h)g/\|g\|^2]$. The PCGrad adaptation treats $L$ and $\beta H$ as two objectives and projects each conflicting gradient against the other before summing \citep{yu2020pcgrad}; it is not the original multi-task application setting. Fixed and balanced logdet use line search on their combined objective with the coefficient frozen for that step; the other methods check task descent.

Two ablations share the full controller's task and finite-step checks. The no-risk variant removes the protected rows and cap checks. The risk-only variant sets the information proposal and entropy request to zero while preserving the risk filter. A separately recorded protocol addendum adds protected balancing: the same balanced proposal, task/risk rows, total-vector norm budget and finite-step task/risk checks, with the entropy-rate row and entropy-decrease check omitted. This isolates the entropy request from generic filtered gradient balancing. Its 32 paired runs supplement the original 256, yielding 288 retained runs. All numerical settings were fixed without test-based selection. The original protocol, addendum, raw results, split indices, specifications, code hashes, and actual evaluation counts are supplied with the implementation.

\subsubsection{Pilot distinction between tangency and finite-step feasibility}
\label{app:benchmark-pilot}
A separate pilot used split/initialization seed 9001, excluded from the eight main seeds, and 100 continuation updates. No held-out test predictions were evaluated in the pilot. With no inward-rate reserve, a projection tangent to an active control cap drove the accepted step toward zero: 34\% of full-HCLM updates were held on Digits and 41\% on WDBC in the reweighting stress. The original raw pilot, its protocol and code hash are retained in \texttt{pilot\_tangent\_firstorder}. An executable reproduction uses the final logging code with reserve zero.

This behavior is consistent with the descent lemma. At $R_j(\theta)=c_j$, a first-order tangent direction with $r_j^\top d=0$ leaves a positive quadratic curvature contribution unconstrained. If $\nabla R_j$ is locally $\ell_j$-Lipschitz and $r_j^\top d\geq\mu_j>0$, then
\[
 R_j(\theta-\eta d)\leq R_j(\theta)-\eta\mu_j+\frac{\ell_j\eta^2}{2}\|d\|^2,
\]
so $0<\eta\leq2\mu_j/(\ell_j\|d\|^2)$ suffices to preserve an active cap. The positive reserve creates an inward direction when such a direction exists. Its cost is a stricter feasibility problem; a vanishing risk gradient or conflicting constraints can still force a hold. Actual finite-step checks remain necessary.

The final value 0.05 was fixed from this pilot before main runs. Both risk-only and full HCLM use the same reserve. The final 288 runs have no held updates, although the code retains rejection and hold branches. The pilot is an explanation of a finite-step failure mode, rather than evidence that the final controller can satisfy arbitrary risk or entropy requests.

\subsection{Observable and sample-channel bridge protocols}
\subsubsection{Observable-intervention protocol and factorization check}\label{app:observablebridge}
The observable restriction can be trained directly in a small intervention space while keeping the task-factorization assumption exact. We use binary Digits classes 0/1 and WDBC, eight paired stratified 80/20 train/test splits, and a four-unit tanh network. For an input dimension $d_x$ of 64 or 30, the random matrix $A=[A_1,\ldots,A_4]\in\R^{d_x\times4}$ is fixed before optimization. For $j=1,\ldots,4$, let $b_j$, $s_j$ and $\omega_j$ denote the hidden bias, latent log-scale and compensated readout coefficient, respectively, and let $\omega_0$ denote the output bias. The representation and compensated readout are
\[
 z_j(x)=e^{s_j}\tanh(A_j^\top x+b_j),\qquad
 f(x)=\sum_{j=1}^4(e^{-s_j}\omega_j)z_j(x)+\omega_0
     =\sum_{j=1}^4\omega_j\tanh(A_j^\top x+b_j)+\omega_0.
\]
The canonical intervention coordinates are the four hidden biases, four readout coefficients, output bias, and four latent log-scales. This gives a fixed 13-dimensional coordinate subspace inside ambient dimensions 269 and 133, respectively. The Euclidean metric is fixed in these declared coordinates. All gradients and Jacobians in this bridge are derivatives of the loss and predictor restricted to these 13 coordinates; the factorized theorem is applied to that restricted learner. Here $Z$ contains the four encoder outputs on the training inputs, and $H=H_{0.1}(Z)$ uses the normalized functional with noise variance $\nu=0.1$. The log-scales are exact prediction-null parameters but change this information functional. This deliberately redundant family instantiates the rescaling obstruction; it is not an architecture-performance benchmark.

Let $F$ contain all training logits and let $L$ be their mean binary cross-entropy. Then $L=\ell(F)$ holds exactly. At every update, an SVD of the analytic prediction Jacobian gives its orthogonal row-space projector $\Pi_{\rm obs}$. We compare task-only training, the unrestricted direction $-g-\beta h$, and the observable direction $-g-\beta\Pi_{\rm obs}h$. All methods share 100 task-only warmstart updates and 150 continuation updates. The coefficient cap is $\min\{10,0.5\|g\|/\max(\|h\|,10^{-150})\}$ for both information rules. This is the code's convention $\max(\|h\|^2,10^{-300})$ before taking a square root. The entropy-rate request is $0.05\|g\|\|h\|$, and the required task fraction is $1/2$. Each controller uses the exact feasible coefficient interval with $\|h\|^2$ or $\|\Pi_{\rm obs}h\|^2$, respectively. If the request is infeasible, the declared fallback retains task progress and reports the unmet entropy request. Backtracking starts at 0.3, uses Armijo fraction 0.1 and up to 30 halvings, and checks entropy descent when the full request is feasible. No test value selects a parameter.

Table~\ref{tab:observable-bridge} reports all 48 runs. The unrestricted requests are feasible throughout, whereas the observable requests are infeasible in every Digits update and in $96.75\%$ of WDBC updates on average. Approximately $1$--$2\%$ of squared information-force energy lies in the observable component. Substantial unrestricted covariance compression therefore occurs largely along a direction that leaves predictions unchanged. The nearly equal test losses and the distinct entropy ratios are consistent with that mechanism. Every run takes its task-preserving fallback when needed; no held update is omitted. This is an implemented local observability test, rather than evidence for the fresh-population certificate.

The factorization assumption is tested separately by replacing the full prediction map by just two training logits at the common warmstart. For this smaller map, define $b_{\rm obs}=g^\top\Pi_{\rm obs}h$, $b=g^\top h$, $a=\|g\|^2$ and $c_{\rm obs}=h^\top\Pi_{\rm obs}h$. The correct rates are
\[
 \dot L=-a-\beta b_{\rm obs},\qquad
 \dot H=-b-\beta c_{\rm obs}.
\]
There is no reason for $b_{\rm obs}=b$ when those probes omit the task observations. In the 16 warmstart checks, using $b$ in the first rate incurs absolute errors between $1.70\times10^{-7}$ and $2.51\times10^{-5}$; the general formula agrees with the directly evaluated directional derivative to below $9\times10^{-19}$. With the full training map, the factorization crossproduct error is below $1.4\times10^{-18}$. The analytic task, entropy and prediction derivatives are independently checked by central differences. This diagnostic prevents applying the factorized theorem to an unrelated minibatch map.

\subsubsection{Why the frozen-encoder experiment satisfies the sample channel}\label{app:channelbridge}
Fix the benchmark, its public/private partition, the public-only preprocessing, random encoder, public Gram matrix and numerical protocol. They are constants before drawing $S$, so the independent side information in Theorem~\ref{thm:gen} can be taken to be this fixed collection. Conditional on this setup, $P$ is a fixed finite uniform distribution on the nonpublic records. Private training pairs are sampled independently with replacement from $P$; increasing $n$ beyond the number of distinct population records is permitted under this sampling model. Population risks below refer to this finite law, not to a future patient or image distribution.

The embedding $\phi_{\rm pub}$ has one constant coordinate and six tanh coordinates, hence $\lVert\phi_{\rm pub}(x)\rVert\le\sqrt7$. If $N_{\rm pub}$ is the number of public records and $x^{\rm pub}_a$ is its $a$th input, the frozen Gram matrix is $K_{\rm pub}=N_{\rm pub}^{-1}\sum_{a=1}^{N_{\rm pub}}\phi_{\rm pub}(x^{\rm pub}_a)\phi_{\rm pub}(x^{\rm pub}_a)^\top$. Define the single-record vector $T(X,Y)=\phi_{\rm pub}(X)Y$, with $Y\in\{0,1\}$. Writing the population support as $(x_j,y_j)$, $j=1,\ldots,N_P$, its moments are exactly
\[
 \mu_T=\frac1{N_P}\sum_jT(x_j,y_j),\qquad
 C_T=\frac1{N_P}\sum_j[T(x_j,y_j)-\mu_T][T(x_j,y_j)-\mu_T]^\top.
\]
The denominator is $N_P$, because this is the covariance of the finite population law, not an unbiased covariance estimate from a random sample. Independence of the private records gives
$\operatorname{Cov}(n^{-1}\sum_i T(X_i,Y_i))=C_T/n=C_{\rm code}$.
Only the evaluator uses $C_T$ to compute the theoretical right side; the learner does not receive its label-dependent entries.

Fresh Gaussian noise, independent of $S$, produces $Q$ in \eqref{eq:frozencodebridge}. The public decoder is deterministic, and every private-sample-dependent argument of the prediction function is a function of $Q$. Thus the Markov condition $S\to Q\to\omega_Q$ is exact. There is no private-data-dependent encoder training, scaling, coefficient selection, early stopping or feedback path outside $Q$. In particular, the experiment does not replace $C_{\rm code}$ with a covariance across one batch's neural activations.

For any fixed readout vector $\omega\in\R^7$, define
\[
 \ell(\omega,(x,y))=\left[\operatorname{clip}(\phi_{\rm pub}(x)^\top\omega,0,1)-y\right]^2.
\]
Clipping makes the scalar prediction lie in $[0,1]$, so this loss also lies in $[0,1]$ for a binary label. Hoeffding's lemma gives sub-Gaussian variance proxy $1/4$ uniformly over readouts. Applying Theorem~\ref{thm:gen} yields \eqref{eq:frozencodebound}. The loss is squared prediction error; the supplementary classification accuracy at threshold $1/2$ is not the loss used in this bound. The public-feature decoder is a restricted linear fitting rule and is not asserted to be standard private-sample ridge regression or a differential-privacy guarantee.

For each fitted readout, the evaluator computes the finite-population and empirical losses exactly:
\[
 R_P(\omega_Q)=\frac1{N_P}\sum_j\ell(\omega_Q,(x_j,y_j)),\qquad
 \widehat R_S(\omega_Q)=\frac1n\sum_i\ell(\omega_Q,(X_i,Y_i)).
\] Their signed difference is averaged over 300 independent training-sample/noise draws at each setting. Reported $95\%$ Student-$t$ intervals quantify this Monte Carlo estimation variability. Comparing the estimated mean with the bound is an implementation check rather than a proof of an unknown expectation; the theorem establishes the expected-gap inequality. All settings and draws, including unfavorable risks and loose bounds, are retained.

\begin{table}[t]
\centering
\small\setlength{\tabcolsep}{4pt}
\caption{Frozen-public-encoder sample channel. The gap is the Monte Carlo mean of population minus empirical bounded squared loss; $\pm$ is its $95\%$ Monte Carlo interval half-width over 300 independent sample/noise draws. $B_{\rm code}$ uses the exactly known finite-population code covariance, not activation covariance. Population risk is evaluated exactly per fitted model. Accuracy is supplementary and is not the bounded loss used in the theorem.}
\label{tab:frozencodebridgefull}
\begin{tabular}{lrrrrrr}
\toprule
Dataset & $n$ & $\nu$ & Pop. risk & Expected gap (MC) & $B_{\rm code}$ & Acc. (\%)\\
\midrule
Digits 0/1 & 64 & $10^{-5}$ & 0.0702 & $0.0112\pm0.0013$ & 0.3115 & 95.04\\
Digits 0/1 & 64 & $10^{-4}$ & 0.0720 & $0.0112\pm0.0015$ & 0.2101 & 94.37\\
Digits 0/1 & 64 & $10^{-3}$ & 0.0852 & $0.0101\pm0.0017$ & 0.1160 & 90.94\\
Digits 0/1 & 256 & $10^{-5}$ & 0.0584 & $0.0027\pm0.0005$ & 0.1256 & 97.38\\
Digits 0/1 & 256 & $10^{-4}$ & 0.0619 & $0.0024\pm0.0005$ & 0.0755 & 97.04\\
Digits 0/1 & 256 & $10^{-3}$ & 0.0803 & $0.0030\pm0.0008$ & 0.0358 & 92.63\\
WDBC & 64 & $10^{-5}$ & 0.1057 & $0.0165\pm0.0025$ & 0.3603 & 86.12\\
WDBC & 64 & $10^{-4}$ & 0.1070 & $0.0171\pm0.0026$ & 0.2658 & 85.95\\
WDBC & 64 & $10^{-3}$ & 0.1181 & $0.0179\pm0.0027$ & 0.1609 & 84.32\\
WDBC & 256 & $10^{-5}$ & 0.0920 & $0.0035\pm0.0012$ & 0.1527 & 87.97\\
WDBC & 256 & $10^{-4}$ & 0.0929 & $0.0044\pm0.0012$ & 0.1013 & 87.83\\
WDBC & 256 & $10^{-3}$ & 0.1057 & $0.0060\pm0.0013$ & 0.0516 & 85.96\\
\bottomrule
\end{tabular}
\end{table}

\subsection{Fixed-coefficient sensitivity protocol and full grid}
\begin{table}[p]
\centering
\caption{Prespecified fixed-coefficient sensitivity analysis. Each nonzero multiplier scales the initial coefficient of the original fixed-logdet method. Multiplier0 reuses ERM. Entries are means $\pm$ 95\% repeated-split variability interval half-widths over eight paired runs; cap violations are percentages of continuation updates on the reused control set. All coefficients are reported; held-out metrics do not select a coefficient.}
\label{tab:fixed-grid}
\footnotesize\setlength{\tabcolsep}{3pt}
\begin{tabular}{rrrrrr}
\toprule
Mult. & Control task CE & Test task CE & Test group0 CE & $H_T/H_0$ & Cap viol.(\%)\\
\midrule
\multicolumn{6}{l}{\textit{Digits / uniform}}\\
0 & $0.182\pm0.017$ & $0.173\pm0.028$ & $0.162\pm0.031$ & $1.108\pm0.006$ & 0.0\\
0.01 & $0.182\pm0.017$ & $0.173\pm0.028$ & $0.162\pm0.031$ & $1.106\pm0.006$ & 0.0\\
0.03 & $0.182\pm0.017$ & $0.174\pm0.028$ & $0.162\pm0.031$ & $1.102\pm0.006$ & 0.0\\
0.1 & $0.183\pm0.017$ & $0.174\pm0.028$ & $0.163\pm0.031$ & $1.088\pm0.005$ & 0.0\\
0.3 & $0.187\pm0.016$ & $0.177\pm0.027$ & $0.165\pm0.031$ & $1.049\pm0.004$ & 0.0\\
1 & $0.206\pm0.016$ & $0.194\pm0.027$ & $0.181\pm0.032$ & $0.909\pm0.006$ & 0.0\\
3 & $0.324\pm0.016$ & $0.312\pm0.027$ & $0.293\pm0.032$ & $0.585\pm0.005$ & 17.5\\
\midrule
\multicolumn{6}{l}{\textit{Digits /95/5}}\\
0 & $0.127\pm0.020$ & $0.128\pm0.022$ & $0.625\pm0.077$ & $1.072\pm0.007$ & 100.0\\
0.01 & $0.128\pm0.020$ & $0.128\pm0.022$ & $0.629\pm0.077$ & $1.066\pm0.007$ & 100.0\\
0.03 & $0.129\pm0.020$ & $0.128\pm0.022$ & $0.636\pm0.077$ & $1.053\pm0.007$ & 100.0\\
0.1 & $0.132\pm0.019$ & $0.131\pm0.021$ & $0.667\pm0.079$ & $1.008\pm0.006$ & 100.0\\
0.3 & $0.145\pm0.018$ & $0.141\pm0.020$ & $0.810\pm0.084$ & $0.873\pm0.009$ & 100.0\\
1 & $0.257\pm0.018$ & $0.252\pm0.015$ & $2.004\pm0.088$ & $0.474\pm0.015$ & 100.0\\
3 & $0.556\pm0.019$ & $0.551\pm0.025$ & $3.812\pm0.046$ & $0.129\pm0.004$ & 100.0\\
\midrule
\multicolumn{6}{l}{\textit{WDBC / uniform}}\\
0 & $0.095\pm0.025$ & $0.084\pm0.044$ & $0.149\pm0.090$ & $1.034\pm0.025$ & 57.4\\
0.01 & $0.095\pm0.025$ & $0.084\pm0.044$ & $0.149\pm0.090$ & $1.031\pm0.025$ & 57.2\\
0.03 & $0.096\pm0.025$ & $0.084\pm0.044$ & $0.149\pm0.090$ & $1.027\pm0.026$ & 57.2\\
0.1 & $0.096\pm0.025$ & $0.084\pm0.044$ & $0.149\pm0.090$ & $1.011\pm0.027$ & 56.7\\
0.3 & $0.096\pm0.025$ & $0.085\pm0.044$ & $0.149\pm0.090$ & $0.968\pm0.030$ & 55.8\\
1 & $0.097\pm0.022$ & $0.085\pm0.044$ & $0.150\pm0.089$ & $0.840\pm0.033$ & 51.2\\
3 & $0.101\pm0.020$ & $0.087\pm0.044$ & $0.150\pm0.086$ & $0.644\pm0.028$ & 59.9\\
\midrule
\multicolumn{6}{l}{\textit{WDBC /95/5}}\\
0 & $0.024\pm0.006$ & $0.021\pm0.009$ & $0.221\pm0.101$ & $0.970\pm0.004$ & 96.7\\
0.01 & $0.024\pm0.006$ & $0.021\pm0.009$ & $0.221\pm0.101$ & $0.961\pm0.004$ & 96.7\\
0.03 & $0.024\pm0.007$ & $0.021\pm0.009$ & $0.221\pm0.101$ & $0.944\pm0.005$ & 96.8\\
0.1 & $0.024\pm0.007$ & $0.021\pm0.009$ & $0.221\pm0.101$ & $0.887\pm0.010$ & 96.7\\
0.3 & $0.025\pm0.007$ & $0.022\pm0.010$ & $0.222\pm0.101$ & $0.755\pm0.026$ & 96.5\\
1 & $0.027\pm0.009$ & $0.024\pm0.012$ & $0.221\pm0.099$ & $0.586\pm0.029$ & 94.9\\
3 & $0.027\pm0.007$ & $0.023\pm0.010$ & $0.229\pm0.097$ & $0.523\pm0.027$ & 91.7\\
\bottomrule
\end{tabular}
\end{table}

\subsection{Fashion-MNIST nuisance protocol}\subsubsection{Fashion-MNIST paired-nuisance protocol}
The four official Fashion-MNIST gzip files were downloaded from the Keras-hosted dataset mirror, with original dimensions, record counts and SHA-256 hashes retained. For each seed $0,\ldots,7$, a stratified 10,000-image subset of the 60,000 official training images is split into 8,000 fitting images and a 2,000-image auxiliary pool. Of the latter, 256 form the fixed, reused risk-control set and 1,744 are reserved for terminal validation reporting. A stratified 2,000-image subset of the separate 10,000-image official test set supplies all terminal test scenarios. Fixed $2\times2$ mean pooling gives 196 stable features, scaled by 255 and centered using fitting means. The generated nuisance adds four coordinates; no synthetic images or duplicated training images are counted as additional real observations.

The learner is a $200\to16\to8\to10$ tanh/tanh/linear-softmax network with 3,442 parameters. Each paired run shares 100 task-only minibatch warmstart updates and then 400 continuation updates, all with batch size 128. A nuisance vector has coordinates $n_j=1.5s+\epsilon_j$, $j=1,\ldots,4$, where the $\epsilon_j$ are independent uniform draws on $[-0.5,0.5]$ and $s\in\{-1,1\}$ is a shared sign. This sign agrees with class-label parity with probability $0.9$ during fitting. Risk-control nuisance signs are independent of labels. Terminal test scenarios are the same $0.9$ correlation, the reversed $0.9$ correlation, and zero nuisance. Labels define functional parity groups, not demographic protected attributes.

For the information calculation, each fitting minibatch supplies two independently drawn nuisance vectors at every stable image. If the rows of $\Delta Z\in\R^{m\times p}$ collect the differences $e_\theta(x_i,n_{i1})-e_\theta(x_i,n_{i2})$ of final hidden outputs for the $m=128$ images in the current minibatch, let $\widehat C_{\Delta Z}=\Delta Z^\top P_m\Delta Z/(m-1)$ be its centered sample covariance, with the centering matrix $P_m$ defined in Section~\ref{sec:geometry}. The functional is
\[
 H_\Delta(\theta)=\frac1{2p}\log\det(I_p+\widehat C_{\Delta Z}/0.1),\qquad p=8.
\]
Its gradient backpropagates both branches with opposite signs. This is the declared nuisance-response object, rather than the entropy of one unmodified latent representation. The reference variance $0.1$ is a design scale. Finite differences over all 55 coordinates of an independently generated small two-hidden-layer network check task, paired-response and both group-risk gradients, with maximum absolute errors below $2.5\times10^{-10}$.

For parity group $j\in\{0,1\}$, $R_j$ is its mean CE on the fixed risk-control set, and the cap is $\tau_j=1.05R_j(\theta_0)+0.01$ at the common warmstart $\theta_0$. Let $g=\nabla_\theta L$ be the current task-minibatch gradient and $h=\nabla_\theta H_\Delta$ the gradient of the paired-nuisance minibatch functional. The Euclidean metric is used. Protected velocities satisfy the primary benchmark's task/risk halfspace--ball construction, with total-velocity budget $B_{\rm tot}=3\|g\|$, task fraction $0.5$ and inward reserve $0.01\|\nabla R_j\|\|g\|$. Protected balancing starts from $w_{\rm prop}=-g-\beta h$, where $\beta=\min\{20,0.5\|g\|/\max(\|h\|,10^{-30})\}$. Full HCLM adds $h^\top w\le-0.2\|g\|\|h\|$ while the current minibatch functional exceeds $0.5H_{\Delta,0}$, and requests $h^\top w\le0$ otherwise. Here $H_{\Delta,0}$ is evaluated once at $\theta_0$ on a fixed pair of nuisance-perturbed risk-control inputs, and the reported terminal entropy ratio uses the same reference inputs. This compares a changing minibatch estimate with a fixed reference; it is a local empirical contract, not the deterministic exponential-tracking law.

Actual candidate checks use the current task minibatch, both fixed risk-control groups, and the admitted positive information-rate request. The common initial step is $0.15$, the Armijo fraction is $0.1$, and at most 22 halvings are attempted. All 32 final runs have no observed cap violation, held update or rejected entropy request. The largest recorded task/risk projection residual is $3.34\times10^{-16}$; the largest control-cap excess is $-0.0351$. The reused control set does not satisfy the fresh-probe hypotheses. The reserved validation and official test sets are evaluated only after fitting, and are never used for configuration selection.

The paired IID/shifted/zero outcomes and source, protocol, dataset and split hashes accompany the code. Reported intervals describe repeated subsets and initializations of this fixed benchmark. Timings include common information diagnostics even for ERM, and are not production-overhead estimates. Both the robust gains and the small difference from generic protected balancing are retained.

\section{Earlier fixed-split diagnostic}
\subsection{Earlier fixed-split learning diagnostic}\label{app:legacy}
We use the 1,797-sample handwritten-digits dataset distributed with scikit-learn \citep{sklearn}, with ten classes and 64 input features. A fixed stratified split contains 1,077 training, 360 validation, and 360 test examples. Pixel values are divided by 16 and centered using training means. A network with 16 tanh hidden units and a linear ten-class readout is trained for 300 full-batch steps from each of 20 common random initializations. The task loss is mean cross-entropy; $H=H_{0.1}(Z)$ uses training activations. We do not inject noise into network predictions; $\nu=0.1$ specifies the covariance reference scale.

The configurations were fixed before this execution. ERM uses task descent. Weight decay adds $\lambda\norm{\theta}^2/2$, with $\lambda\in\{10^{-4},10^{-3},10^{-2}\}$. Fixed entropy uses $L+\beta H$, with $\beta\in\{0.01,0.1,1\}$. Each family selects its coefficient by final validation cross-entropy separately for each initialization; test performance is never used for selection. The norm-matched rule chooses $\beta=\min\{10,0.2\sqrt{a/c}\}$ without accounting for the gradient angle. The projected rule uses $q$ from \eqref{eq:effective-gradient} and chooses $\beta=\min\{10,0.2\sqrt{a/c_{\rm eff}}\}$. Both use zero control when the corresponding task norm vanishes, with numerical denominator safeguards stated in the code.

The scalar controller uses $\kappa=1/2$, $k=0.1$, $H_\star=0.8H(\theta_0)$, and $\overline\beta=\min\{10,\sqrt{a/c}\}$. It chooses $\beta_{\rm lo}$ when feasible and $\beta_{\rm hi}$ under the task-priority fallback. This deliberately tests a fixed compression demand; it is not a validation-selected entropy target. Comparing it with a tuned regularizer therefore diagnoses a control request, rather than establishing an optimally tuned algorithm ranking.

All methods start with step size $0.2$ and use Armijo backtracking with constant $0.1$. ERM, the projected rule, and the scalar controller check task loss; weight decay and unprojected entropy methods check their current penalized objective, holding that step's coefficient fixed. The experiment comprises 200 runs and 60,000 parameter updates. Every run includes gradient diagnostics, even for ERM, so recorded runtimes are not estimates of the production overhead of these methods. Appendix~\ref{app:protocol} gives further implementation details.

\begin{table}[t]
\centering\small
\caption{Digits: final results over 20 initializations. Intervals are conditional on the fixed split. $H_T/H_0$ uses training covariance. Infeasibility is evaluated only for scalar feedback; a dash means that no entropy-rate request was imposed.}\label{tab:digits}
\begin{tabular}{lrrrr}
\toprule
Method & Test CE & Accuracy (\%) & $H_T/H_0$ & Infeasible (\%)\\
\midrule
ERM & $0.2274\pm0.0032$ & $94.33\pm0.25$ & 2.584 & --\\
Weight decay & $0.2284\pm0.0032$ & $94.33\pm0.25$ & 2.580 & --\\
Fixed entropy & $0.2277\pm0.0032$ & $94.32\pm0.27$ & 2.576 & --\\
Norm-matched & $0.2394\pm0.0034$ & $94.24\pm0.26$ & 2.325 & --\\
Projected entropy & $0.2232\pm0.0032$ & $94.57\pm0.26$ & 2.344 & --\\
Scalar feedback & $0.3092\pm0.0057$ & $93.83\pm0.31$ & 1.727 & $93.05\pm0.53$\\
\bottomrule
\end{tabular}
\end{table}

Table~\ref{tab:digits} reports means and 95\% Student-$t$ intervals over initializations, conditional on the fixed split. The intervals describe optimization variability; they do not measure uncertainty over datasets or constitute a multiple-comparison-adjusted benchmark claim. The full coefficient grid is retained in the supplementary files.

The projected rule has a modestly lower mean test cross-entropy than ERM, whereas norm matching without projection has higher cross-entropy. The scalar controller gives substantially higher test cross-entropy even though it satisfies the task-descent certificate at every recorded update. Its requested entropy rate is infeasible in $93.05\%$ of steps on average, and its final entropy is $1.727$ times its initial value, well above the target ratio $0.8$. A smaller increase than ERM is not target attainment. This is the distinction between respecting a priority constraint and accomplishing both objectives.

There were no realized task-loss increases for the scalar or projected controllers in these full-batch runs. This checks the implemented task safeguard in this setting; it does not verify stochastic, delayed-feedback, or population-risk safety. Nor does the smaller test loss of one projected rule establish that entropy caused better generalization: changing the update direction also changes optimization speed and the solution reached after 300 steps. The projected rule's generalization gap is slightly larger than ERM's despite its lower test loss, illustrating why the gap should not be interpreted in isolation.

The experiment provides a concrete failure mode for overaggressive entropy requests and a controlled comparison of angle-aware and norm-only intervention. It is too small to establish performance across architectures. It includes neither transformer training nor a human reward signal, and it does not test the sample-code assumption of Theorem~\ref{thm:gen}. Wider empirical claims would require the studies described next.

\subsection{Earlier experiment protocol and complete comparison}\label{app:protocol}
The data are loaded by scikit-learn's \texttt{load\_digits}. The split uses stratified \texttt{train\_test\_split} with random state 42: first 20\% for testing, then 25\% of the remainder for validation. The training mean is subtracted from every split after dividing pixels by 16. For each seed from 0 to 19, encoder weights have independent $N(0,1/64)$ entries, readout weights have independent $N(0,1/16)$ entries, and biases start at zero. All methods share the same initialization for that seed. Gradients are implemented analytically in NumPy float64 and checked by central differences with increment $10^{-6}$.

The loss is evaluated using a stable log-sum-exp. The covariance uses denominator $m-1$ and the entropy gradient includes the division by latent dimension. There is no minibatch sampling, early stopping, augmentation, momentum, or adaptive optimizer. Backtracking halves the step up to 30 trials. Gradient-ratio denominators use a floor of $10^{-30}$; inequality checks use the tolerances recorded in the source. The experiment ran with NumPy 2.3.5, SciPy 1.17.0, and scikit-learn 1.8.0. The environment file records dataset and protocol hashes.

The primary table contains validation-selected weight-decay and fixed-entropy configurations. In this run, every seed selected $\lambda=10^{-4}$ and $\beta=0.01$, respectively. All other configurations appear in Table~\ref{tab:grid}. No run was excluded. Initialization intervals are $\bar x\pm t_{0.975,19}s/\sqrt{20}$. Paired comparisons subtract ERM from the corresponding method within each seed; their intervals are computed from those differences, not from independent error bars.

\begin{table}[htbp]
\centering\small
\caption{Complete coefficient grid and adaptive configurations. Values are means over the same 20 initializations. No configuration was removed.}\label{tab:grid}
\begin{tabular}{llrrrr}
\toprule
Method & Coefficient & Train CE & Validation CE & Test CE & Accuracy (\%)\\
\midrule
ERM & -- & 0.1899 & 0.2309 & 0.2274 & 94.33\\
Weight decay & 0.0001 & 0.1910 & 0.2319 & 0.2284 & 94.33\\
Weight decay & 0.001 & 0.2010 & 0.2417 & 0.2376 & 94.25\\
Weight decay & 0.01 & 0.3152 & 0.3554 & 0.3465 & 93.44\\
Fixed entropy & 0.01 & 0.1901 & 0.2312 & 0.2277 & 94.32\\
Fixed entropy & 0.1 & 0.1929 & 0.2345 & 0.2308 & 94.39\\
Fixed entropy & 1 & 0.2511 & 0.2998 & 0.2916 & 93.92\\
Norm-matched entropy & adaptive & 0.2008 & 0.2437 & 0.2394 & 94.24\\
Projected entropy & adaptive & 0.1823 & 0.2252 & 0.2232 & 94.57\\
Scalar feedback & adaptive & 0.2690 & 0.3182 & 0.3092 & 93.83\\
\bottomrule
\end{tabular}
\end{table}

Supplementary CSV files retain every final result and trajectory diagnostics. Diagnostics are recorded every ten updates and immediately before the final update; the final-result files report the state after all 300 updates. Separate files provide the exact split indices, configuration, dependency versions, data hashes, and numerical check outputs. The accompanying README maps each file to its purpose. The code regenerates the tables and two figures from these records. Runtime values include diagnostic computations and should not be used to claim a hardware-speed advantage.

\section*{Funding}
This research received no funding.

\section*{Declaration of competing interest}
The author declares no competing interests.

\section*{Data and code availability}
Digits and WDBC are accessed through the scikit-learn dataset loaders. Fashion-MNIST is obtained from its official distribution using the supplied downloader; source URLs, file hashes and the dataset license are retained. The supplementary material contains the experiment code, declared protocols, split indices, raw numerical results, and scripts for regenerating the tables and figures.

\section*{Human-participant research}
No new human participants were recruited for this study. The experiments use public benchmark datasets and do not constitute a clinical or human-feedback study.

\nocite{*}
\bibliographystyle{plainnat}
\bibliography{references}
\end{document}